\pdfoutput=1
\documentclass{article} %
\usepackage{iclr2027_conference,times}
\iclrfinalcopy
\makeatletter
\fancypagestyle{preprintfirst}{\fancyhf{}\cfoot{\thepage}\lfoot{\footnotesize Preprint.}\renewcommand{\headrulewidth}{0pt}}
\let\iclr@oldmaketitle\maketitle
\renewcommand{\maketitle}{\iclr@oldmaketitle
  \fancyhf{}\cfoot{\thepage}\renewcommand{\headrulewidth}{0pt}%
  \thispagestyle{preprintfirst}}%
\makeatother
\usepackage{microtype}
\usepackage{graphicx}
\usepackage{amssymb}
\usepackage{longtable}
\usepackage{booktabs}
\usepackage{amsmath}
\usepackage{amsthm}
\usepackage{enumitem}
\usepackage{algorithm}
\usepackage{algorithmicx}
\usepackage{adjustbox}
\usepackage{algpseudocode}
\usepackage{caption}

\newcommand{\Rbb}{{\mathbb{R}}}

\newcommand{\Ab}{{\bf A}}

\newcommand{\Mb}{{\bf M}}

\newcommand{\Hb}{{\bf H}}
\newcommand{\Cb}{{\bf C}}

\newcommand{\gb}{{\mathbf{g}}}

\newcommand{\Qb}{{\bf Q}}

\newcommand{\xb}{{\bf x}}

\newcommand{\yb}{{\bf y}}
\newcommand{\zb}{{\bf z}}
\newcommand{\Zb}{{\bf Z}}

\newcommand{\EX}{{\mathbb{E}}}

\newcommand{\zetab}{{\bf \zeta}}
\newcommand{\Bb}{{\bf B}}
\newcommand{\Gb}{{\bf G}}
\newcommand{\Xb}{{\bf X}}

\newcommand{\Wb}{{\bf W}}

\newcommand{\nf}[1]{\|#1\|_F} %
\newcommand{\nt}[1]{\|#1\|_2}

\newcommand{\Bc}{{\mathcal{B}}}

\newcommand{\Psib}{{\mbox{\boldmath $\Psi$}}}
\newcommand{\Sigmab}{{\mbox{\boldmath $\Sigma$}}}

\newcommand{\thetab}{{\mbox{\boldmath $\theta$}}}

\newcommand{\argmin}{\mathop{{\rm argmin}}}

\newcommand{\Tr}{\mbox{{Tr}}}

\newcommand{\Ib}{{\bf I}}

\newcommand{\Deltab}{{\mbox{\boldmath $\Delta$}}}

\newcommand{\vb}{{\mathbf v}}

\newcommand{\Dc}{{\mathcal{D}}}

\newcommand{\Nc}{{\mathcal{N}}}

\newcommand{\Pc}{{\mathcal{P}}}

\newcommand{\Hc}{{\cal H}}
\newcommand{\Lc}{{\cal L}}
\newsavebox\mybox

\usepackage{xcolor}
\newif\ifrevisionmarkup
\revisionmarkuptrue
\ifrevisionmarkup
  
\else
  
\fi
\definecolor{RevRed}{rgb}{0.78,0.05,0.09}

\definecolor{SeqBlue}{HTML}{2E6DA4}
\definecolor{SeqGreen}{HTML}{4A7C34}
\definecolor{SeqAmber}{HTML}{D9A441}
\definecolor{SeqGray}{HTML}{6B7280}
\usepackage{multirow}
\usepackage{subcaption}

\newcommand{\CA}{Continual Alternating}

\usepackage{hyperref}
\newtheorem{theorem}{Theorem}

\newcommand{\refappendix}[1]{\hyperref[#1]{Appendix~\ref*{#1}}}
\newtheorem{lemma}[theorem]{Lemma}
\newtheorem{definition}{Definition}
\newtheorem{assumption}{Assumption}

\newtheorem{proposition}{Proposition}
\newtheorem{remark}{Remark}
\newtheorem{example}{Example}
\newif\ifshowrev

\showrevtrue
\usepackage{pgfplots}
\pgfplotsset{compat=1.18}
\usepgfplotslibrary{groupplots,fillbetween}
\usetikzlibrary{fillbetween,plotmarks}
\usepackage{xcolor}
\usepackage{pgfplotstable}

\definecolor{coral}{rgb}{0.8, 0.5, 0.31}
\definecolor{BenignGray}{gray}{0.35}
\definecolor{AttackRed}{rgb}{0.78,0.00,0.08}

\definecolor{CClipMagenta}{rgb}{0.80,0.00,0.50}
\definecolor{CClipBuck}{rgb}{0.55,0.00,0.55}
\definecolor{BulyanBuck}{rgb}{0.54, 0.17, 0.89}
\definecolor{RfaGreen}{rgb}{0.10,0.50,0.10}
\definecolor{RfaBuck}{rgb}{0.05,0.35,0.25}
\definecolor{CwMedBlue}{rgb}{0.00,0.30,0.80}
\definecolor{HuberBrown}{RGB}{150,134,54}

\colorlet{FedLAWColor}{orange!80}
\colorlet{BulyanColor}{cyan!98!blue}
\colorlet{KrumColor}{teal!80}
\colorlet{TrimmedColor}{red!80}
\colorlet{NoDefenceColor}{black}

\colorlet{CClipColor}{CClipMagenta}
\colorlet{CClipBuckColor}{CClipBuck}
\colorlet{BulyanBuckColor}{BulyanBuck}
\colorlet{RFAColor}{RfaGreen}
\colorlet{RFABuckColor}{RfaBuck}
\colorlet{CwMedColor}{black!50} %
\colorlet{HuberColor}{HuberBrown}

\pgfplotsset{
  /tikz/fedlaw/.style     ={FedLAWColor,     very thick, mark=*,        mark options={scale=1, fill=FedLAWColor}},
  /tikz/bulyan/.style     ={BulyanColor,     semithick, mark=o,         mark options={scale=1}},
  /tikz/krum/.style       ={KrumColor,       thin,      mark=square*,   mark options={scale=0.7}},
  /tikz/trimmed/.style    ={TrimmedColor,    thin,      mark=triangle*, mark options={scale=1}},
  /tikz/nodefence/.style  ={NoDefenceColor,  thin,      mark=diamond*,  mark options={scale=1}},
  /tikz/cclip/.style      ={CClipColor,      thin,      mark=star,      mark options={scale=1}},
  /tikz/cclipbuck/.style  ={CClipBuckColor,  thin,      mark=asterisk,  mark options={scale=1}},
  /tikz/bulyanbuck/.style ={BulyanBuckColor, thin,      mark=pentagon*, mark options={scale=1}},
  /tikz/rfa/.style        ={RFAColor,        thin,      mark=x,         mark options={scale=1.15}},
  /tikz/rfabuck/.style    ={RFABuckColor,    thin,      mark=+,         mark options={scale=1.1}},
  /tikz/cwmed/.style      ={CwMedColor,      thin,      mark=diamond*,  mark options={scale=1}},
  /tikz/huber/.style      ={HuberColor,      semithick, mark=triangle,  mark options={scale=1}},
}

\pgfplotsset{
  neurips style/.style={
    width=0.9\linewidth,
    height=5.3cm,
    xmin=0, xmax=100,
    line width=0.45pt,
    grid=major,
    grid style={solid, gray!30},
    legend cell align=left,
    legend style={
      font=\footnotesize,
      at={(0.97,0.5)}, anchor=north east,
      draw=none,
      fill=white, fill opacity=0.9,
      /tikz/every even column/.style={column sep=5pt},
    },
    legend columns=2,
  }
}

\pgfkeys{/series/.is family, /series,
  fedlaw/.code={
    \addplot[/tikz/fedlaw]
      table[x=fraction_malicious, y=sparse_capped_mean, col sep=space]{\SeriesFile};},
  bulyan/.code={
    \addplot[/tikz/bulyan]
      table[x=fraction_malicious, y=bulyan_mean, col sep=space]{\SeriesFile};},
  krum/.code={
    \addplot[/tikz/krum]
      table[x=fraction_malicious, y=krum_mean, col sep=space]{\SeriesFile};},
  trimmed/.code={
    \addplot[/tikz/trimmed]
      table[x=fraction_malicious, y=trimmed_mean, col sep=space]{\SeriesFile};},
  nodefence/.code={
    \addplot[/tikz/nodefence]
      table[x=fraction_malicious, y=nodefence_mean, col sep=space]{\SeriesFile};},
  cclip/.code={
    \addplot[/tikz/cclip]
      table[x=fraction_malicious, y=cclip_mean, col sep=space]{\SeriesFile};},
  cclipbuck/.code={
    \addplot[/tikz/cclipbuck]
      table[x=fraction_malicious, y=cclip_bucketing_mean, col sep=space]{\SeriesFile};},
  bulyanbuck/.code={
    \addplot[/tikz/bulyanbuck]
      table[x=fraction_malicious, y=bulyan_bucketing_mean, col sep=space]{\SeriesFile};},
  rfa/.code={
    \addplot[/tikz/rfa]
      table[x=fraction_malicious, y=rfa_mean, col sep=space]{\SeriesFile};},
  rfabuck/.code={
    \addplot[/tikz/rfabuck]
      table[x=fraction_malicious, y=rfa_bucketing_mean, col sep=space]{\SeriesFile};},
  cwmed/.code={
    \addplot[/tikz/cwmed]
      table[x=fraction_malicious, y=cwmed_mean, col sep=space]{\SeriesFile};},
  huber/.code={
    \addplot[/tikz/huber]
      table[x=fraction_malicious, y=huber_mean, col sep=space]{\SeriesFile};},
}

\pgfkeys{/serieslegend/.is family, /serieslegend,
  fedlaw/.code={\pgfkeys{/series/fedlaw}\addlegendentry{\textbf{FedLAW (Ours)}}},
  bulyan/.code={\pgfkeys{/series/bulyan}\addlegendentry{Bulyan}},
  krum/.code={\pgfkeys{/series/krum}\addlegendentry{Krum}},
  trimmed/.code={\pgfkeys{/series/trimmed}\addlegendentry{Trimmed Mean}},
  nodefence/.code={\pgfkeys{/series/nodefence}\addlegendentry{No Defence}},
  cclip/.code={\pgfkeys{/series/cclip}\addlegendentry{CClip}},
  cclipbuck/.code={\pgfkeys{/series/cclipbuck}\addlegendentry{CClip + Bucketing}},
  bulyanbuck/.code={\pgfkeys{/series/bulyanbuck}\addlegendentry{Bulyan + Bucketing}},
  rfa/.code={\pgfkeys{/series/rfa}\addlegendentry{RFA}},
  rfabuck/.code={\pgfkeys{/series/rfabuck}\addlegendentry{RFA + Bucketing}},
  cwmed/.code={\pgfkeys{/series/cwmed}\addlegendentry{Coordinate-wise Median}},
  huber/.code={\pgfkeys{/series/huber}\addlegendentry{Huber Aggregator}},
}

\title{SeqLoRA: Bilevel Orthogonal Adaptation for Continual Multi-Concept Generation}
\author{Javad Parsa\thanks{Equal contribution. Correspondence to: \texttt{jparsa@ethz.ch}}\\
ETH Zurich
\And
Enis Simsar\footnotemark[1]\\
ETH Zurich
\And
Amir Joudaki\\
ETH Zurich
\And
Thomas Hofmann\\
ETH Zurich
\AND
Andr\'e M. H. Teixeira\\
Uppsala University}
\begin{document}
\maketitle
\vspace{-0.8cm}
\begin{abstract}
   Parameter-efficient fine-tuning enables fast personalization of text-to-image diffusion models {to user-provided concepts (objects, people, or styles)}, but composing multiple {such} concepts remains challenging due to representation interference. Existing modular methods{, usually built on low-rank adaptation (LoRA),} either rely on expensive post-hoc fusion or freeze {the LoRA} adaptation subspaces, which limit expressiveness and concept fidelity. To address this trade-off, we propose Sequential regularized LoRA (SeqLoRA), a constrained continual learning framework that jointly optimizes both LoRA factors via bilevel optimization {while keeping each new basis orthogonal to all previously learned ones}. Theoretically, we establish {monotone descent and convergence to a critical point of the constrained problem}, and model the residual layer activations as a matrix sub-Gaussian process to derive a high-probability bound on catastrophic forgetting {in multi-layer nonlinear networks. This bound depends on the basis only through a residual interference energy, and within the feasible subspace we prove that a data-adapted basis minimizes it, whereas a random frozen basis is suboptimal in expectation. Experiments on Stable Diffusion with up to 101 concepts show that SeqLoRA improves identity preservation over fusion-based methods, attains the lowest cross-concept leakage in multi-concept compositions, requires no fusion step, and scales to concept counts at which fusion runs out of memory.}
\end{abstract}
\vspace{-0.6cm}
\section{Introduction}
\vspace{-0.25cm}
{Text-to-image diffusion models generate high-quality, diverse images from natural-language prompts~\citep{ho2020denoising,rombach2022high,saharia2022photorealistic,ramesh2022hierarchical} and now support a wide range of creative and scientific applications~\citep{po2024state,luo2025taming}. Parameter-efficient fine-tuning, especially Low-Rank Adaptation (LoRA)~\citep{hu2022lora}, has made it feasible for individual users to adapt these foundation models to new concepts, such as a specific object, person, or style, from a small set of reference images~\citep{ruiz2023dreambooth,gal2023image}.}

While single-concept customization has become highly mature~\citep{ruiz2023dreambooth,gal2023image,hu2022lora}, a key challenge arises when users want to generate images that include multiple personalized concepts at once, for instance, two specific pets interacting in the same scene, or a particular artistic style combined with a custom object~\citep{kumari2023multi,gu2023mix}. The central challenge is \emph{representation interference}: because concepts are usually learned separately, their corresponding adaptations may occupy overlapping regions in the model’s parameter space~\citep{po2024orthogonal,liang2024inflora}. When these adaptations are combined at inference time, the shared directions can lead to \emph{attribute entanglement}, where identity-related features from one concept bleed into another, resulting in mixed attributes, weakened identities, or even incoherent generations~\citep{po2024orthogonal,gu2023mix}. This problem is especially pronounced for semantically similar concepts, such as different human faces, whose learned weight residuals tend to be highly correlated and lie in nearly the same subspaces~\citep{po2024orthogonal}. As the number of concepts increases, the combinatorial growth in possible concept subsets makes it infeasible to train a separate model for every combination, motivating modular methods in which independently fine-tuned components can be composed at test time without retraining~\citep{gu2023mix,po2024orthogonal}.
\begin{figure}[t] %
        \centering
        \includegraphics[width=0.82\textwidth]{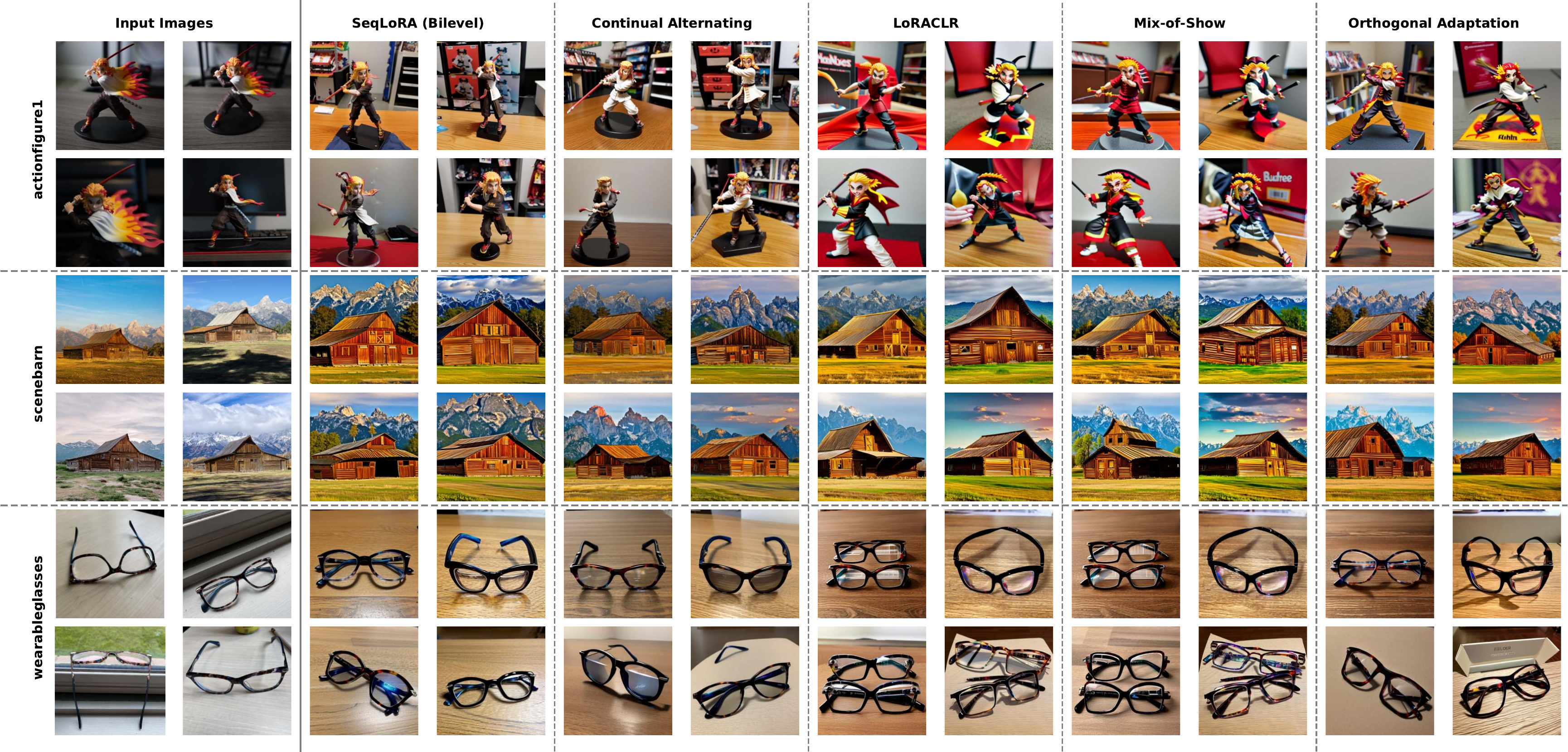}
        \caption{Qualitative comparison after learning 32 concepts into one model. We show the original concept images (Input) and 4 generated samples (seeds) for all methods.}
        \label{fig:qualitative}
        \vspace{-0.6cm}
\end{figure}

Existing remedies for concept interference come with distinct trade-offs.  \emph{Joint training} methods such as Custom Diffusion~\citep{kumari2023multi} and Break-A-Scene~\citep{avrahami2023break} optimize all concepts simultaneously within a shared model, which can reduce interference but requires concurrent access to all concept data, a requirement that is unscalable to large concept libraries and raises data privacy concerns in multi-user settings.  Gradient fusion approaches, exemplified by Mix-of-Show~\citep{gu2023mix}, allow independent fine-tuning of each concept and then merge the resulting models through an optimization-based fusion step to preserve identity; however, the fusion must be re-executed for every new combination of concepts, preventing truly modular composition where individual adapters can be added or removed on the fly. Orthogonal Adaptation~\citep{po2024orthogonal} takes a different approach by sampling each concept's LoRA basis $\Bb_i$ from a shared orthogonal matrix and freezing it during training, optimizing only the coefficient matrix $\Ab_i$.  This guarantees approximate orthogonality between concepts and enables instant merging via simple summation, but freezing $\Bb_i$ restricts the feasible space of the adaptation and limits per-concept fidelity.  In the continual learning literature, InfLoRA~\citep{liang2024inflora} designs task-specific LoRA subspaces for image classification that are orthogonal to the gradient subspaces of previously learned tasks, achieving strong stability--plasticity trade-offs for sequential classification; however, it similarly freezes the basis matrix for each task, inheriting the same expressiveness limitation.

{These approaches share a tension between orthogonality, needed to reduce interference, and expressiveness, needed for high-fidelity adaptation: no existing method optimizes both LoRA factors freely while enforcing the subspace orthogonality required for low-interference composition. We address this tension by formulating multi-concept generation as a constrained continual learning problem. Figure~\ref{fig:qualitative} previews the effect after 32 concepts: on challenging concepts such as \texttt{wearableglasses} the baselines entangle attributes, whereas SeqLoRA preserves the identity.}
\vspace{-0.15cm}
\paragraph{Contributions.}
\vspace{-0.15cm}
\begin{itemize}[leftmargin=1.2em, itemsep=0pt, topsep=0pt]
    \item \textbf{Method.}  SeqLoRA jointly optimizes both LoRA factors under exact subspace
    orthogonality via a bilevel update with a cross-Hessian correction and a closed-form
    projection (\S\ref{sec:method}){, yielding a single model that composes any subset of concepts by summation, with no fusion step and no stored data of past concepts}.
    \item \textbf{Theory.} Monotone descent and convergence to a critical point of the
    constrained problem (Theorem~\ref{thm:smoothcon}); a high-probability forgetting bound
    for the \emph{nonlinear} network (Theorem~\ref{thm:hanson_wright_catas}) identifying
    residual activation energy as the driver of interference, {with an interpretable rate
     and a dominance result, within the feasible subspace, for covariance-adapted over random frozen bases.}
    \item \textbf{Experiments.} On Stable Diffusion v1.5 with 8 to 101 concepts, SeqLoRA improves identity preservation over the baselines and scales past the point at which they run out of memory. {In multi-concept compositions it attains the lowest cross-concept leakage of all methods; we also quantify robustness to concept order, rank and concept similarity, compute cost, and report a preliminary transfer to SDXL \citep{podell2024sdxl}.}
\end{itemize}
\vspace{-0.2cm}
\vspace{-0.3cm}
\section{Problem Formulation}\label{sec:formulation}
\vspace{-0.2cm}
\subsection{Continual Learning over concepts}
\vspace{-0.2cm}
We consider a sequential concept stream $\tau_1, \dots, \tau_T$, where each task $\tau_i$ provides a small reference dataset $\mathcal{D}_i = \{(\mathbf{x}_{i,m}, c_i)\}_{m=1}^{M_i}$ pairing $M_i$ reference images with a concept-specific text token $c_i$. Only $\mathcal{D}_i$ is accessible when task $\tau_i$ arrives; prior datasets are discarded. The goal is a compact, composable representation per concept such that any subset of the $T$ concepts can be jointly rendered at inference time without retraining, balancing \emph{plasticity} (adapting to new concepts) and \emph{stability} (preserving previous ones).
\vspace{-0.2cm}
\subsection{Single-Concept Fine-Tuning with LoRA}
\vspace{-0.2cm}
{Rather than updating the full weight matrix $\Wb_i \in \Rbb^{n \times m}$ of an attention layer, LoRA~\citep{hu2022lora} learns a low-rank residual for each concept $i$:}
\vspace{-0.1cm}
\begin{equation}\label{eq:lora}
    \Wb_i = \Wb_0 + \Ab_i\Bb_i^\top,
\end{equation}
where $\Wb_0$ contains the pretrained parameters for that layer, $\Ab_i \in \Rbb^{n \times r}$ and $\Bb_i \in \Rbb^{m \times r}$ are the trainable low-rank factors with rank $r \ll \min\{m,n\}$.  For a single concept $i$, the standard training objective minimizes the denoising loss
\vspace{-0.3cm}
\begin{equation}\label{eq:single-loss}
    \Lc_i(\Ab_i, \Bb_i) = \EX_{\xb_i, \boldsymbol{\epsilon}, t}\left[\left\|\boldsymbol{\epsilon} - \boldsymbol{\epsilon}_{\boldsymbol{\theta}(\Wb_i)}(\xb_i^{(t)}, t, c_i)\right\|_2^2\right],
\end{equation}
where $\boldsymbol{\epsilon} \sim \Nc(\mathbf{0}, \Ib)$ is the noise target, $t$ is a uniformly sampled diffusion timestep, $\xb_i^{(t)}$ is the noised version of the reference image, $\boldsymbol{\epsilon}_{\boldsymbol{\theta}}$ is the noise-prediction network, and $\thetab(\Wb_i)$ denotes the full network parameters with the adapted layer weight set to $\Wb_i$. {Adaptation is applied to $L$ layers; we omit the layer index whenever the argument is the same for every layer.}
\vspace{-0.2cm}
\subsection{Multi-Concept Aggregation and the Interference Problem}
\vspace{-0.2cm}
In a practical deployment scenario, typically each concept is fine-tuned independently, potentially by different users in parallel, producing a separate LoRA module $(\Ab_i, \Bb_i)$ per concept \citep{gu2023mix, po2024orthogonal}.  At inference time, the goal is to compose an arbitrary subset of these modules into a single model for multi-concept generation.  The simplest aggregation strategy is to sum the individual LoRA residuals per layer:
\vspace{-0.3cm}
\begin{equation}
    \Wb_T = \Wb_0 + \sum_{i=1}^{T} \Ab_{i}\Bb_{i}^\top.
    \label{eq:aggregate}
\end{equation}
To explain why naive aggregation fails, consider a linear layer adapted for concept $i$ with input $\Xb_i$.  The single-concept output is $\mathbf{O}_i(\Xb_i) = (\Wb_0 + \Ab_i\Bb_i^\top)\Xb_i$, while the merged output becomes $\hat{\mathbf{O}}_i(\Xb_i) = (\Wb_0 + \Ab_i\Bb_i^\top + \Ab_j\Bb_j^\top)\Xb_i$.  The discrepancy $\hat{\mathbf{O}}_i - \mathbf{O}_i = \Ab_j\Bb_j^\top\Xb_i$ is the crosstalk term~\citep{po2024orthogonal}.  Concept identity is preserved when this term vanishes, but since $\Xb_i$ typically has full column rank, exact cancellation is impossible.  A natural relaxation~\citep{po2024orthogonal} projects $\Xb_i$ onto the orthogonal complement of $\operatorname{col}(\Bb_j)$.  Let $\bar{\Bb}_j \in \Rbb^{m \times (m-r)}$ span this complement; then $\Ab_j\Bb_j^\top \bar{\Bb}_j\bar{\Bb}_j^\top\Xb_i = \mathbf{0}$ since $\Bb_j^\top\bar{\Bb}_j = \mathbf{0}$.  Because $r \ll \min\{m,n\}$, the complement covers most of $\Rbb^m$, so $\bar{\Bb}_j\bar{\Bb}_j^\top\Xb_i \approx \Xb_i$.  At the same time, for concept $i$'s residual to remain effective, the columns of $\Bb_i$ must not lie in $\operatorname{col}(\Bb_j)$, which requires $\Bb_i^\top\Bb_j = \mathbf{0}$.  This establishes pairwise orthogonality of the LoRA bases as the key condition for low interference aggregation.  %
\vspace{-0.2cm}
\subsection{Desiderata}\label{sec:desiderata}
\vspace{-0.2cm}
The preceding analysis motivates three requirements: (i)~pairwise
orthogonality $\Bb_i^\top \Bb_j = \mathbf{0}$ for $i \neq j$ to reduce
inference-time interference; (ii)~joint optimization of $\Ab_i$ and $\Bb_i$
to preserve expressiveness; and (iii)~continual learning, accommodating new concepts sequentially without
access to past data or joint retraining. Prior works~\citep{po2024orthogonal, liang2024inflora} sacrifice~(ii) by
freezing $\Bb_i$. The following warm-up motivates why learning the basis is preferable to freezing it.
\begin{example}[Warm-Up: why the basis matters]\label{exa:linearwarmup}
    Consider a single-layer linear map with input $\xb\in\Rbb^m$ and output $\yb\in\Rbb^n$. For concept $j$ with distribution $\Dc_j$, define the population risk $\Lc_j(\Wb)\triangleq\EX_{(\xb,\yb)\sim \Dc_j}\{\|\yb-\Wb \xb\|_2^2\}$ and the input covariance $\Sigmab_j\triangleq \EX_{\xb\sim \Dc_j}\{\xb\xb^\top\}$. Under orthogonality constraints, the interference for concept $j$ after learning $T$ concepts is proportional to its uncaptured residual energy: $\Tr(\Pc_{B_j}^{\perp}\Sigmab_j)$, where $\Pc_{B_j}^{\perp}$ is the orthogonal projector onto the complement of $\mathrm{col}(\Bb_j)$ (Theorem~\ref{thm:linear_residual_forgetting} in Appendix~\ref{sec:population_theory}). To minimize this interference, the rank-$r$ basis $\Bb_j$ must align with the top-$r$ principal components of $\Sigmab_j$. Consequently, randomly frozen bases (e.g., Orthogonal Adaptation \citep{po2024orthogonal}) {cannot exploit this spectral geometry and leave more residual interference in
expectation (Proposition~\ref{prop:random_frozen_basis}).}
\end{example}
{The example does not model the nonlinear denoising network, and the covariance-optimal basis \emph{is not necessarily} optimal for the end-to-end loss in \eqref{eq:single-loss}. Rather, its role is to isolate, within a stylized linear model, the quantity that a learned basis can reduce but a frozen basis cannot.}
\vspace{-0.25cm}
\section{Proposed Method}\label{sec:method}
\vspace{-0.25cm}
{SeqLoRA treats the previously learned bases $\{\Bb_j\}_{j<i}$ as a compact memory of past concepts and constrains each new basis $\Bb_i$ to the orthogonal complement of their column span; the full procedure is summarized in Algorithm~\ref{alg:seqlora}.}
\vspace{-0.2cm}
\subsection{Constrained Optimization Problem}
\vspace{-0.2cm}
When concept $i$ arrives, we seek LoRA factors $(\Ab_i, \Bb_i)$ that minimize the denoising loss while keeping $\Bb_i$ orthogonal to all previously learned bases:
\begin{equation}\label{eq:constrained}
    \min_{\Ab_i,\,\Bb_i\in\Bc_i}\; \Lc_i(\Ab_i, \Bb_i),
\end{equation}
where
\begin{equation}
    \Bc_i = \{\Bb_i\in\Rbb^{m\times r}\mid  \Bb_j^\top \Bb_i = \mathbf{0}, \quad \forall\, j < i, \mathbf{0}\in\Rbb^{r\times r}\}.
\end{equation}
 Problem~\eqref{eq:constrained} is jointly non-convex in $(\Ab_i, \Bb_i)$, but it has favorable structure: $\Lc_i$ is differentiable with respect to each factor individually, and the constraints are linear in $\Bb_i$.
\vspace{-0.2cm}
\subsection{Bilevel Optimization Algorithm}\label{subsec:bilevel}
\vspace{-0.2cm}
{Block-coordinate schemes such as BSUM~\citep{razaviyayn2013bsum} or prox-linear methods~\citep{drusvyatskiy2019efficiency} solve \eqref{eq:constrained} by alternating between $\Ab_i$ for fixed $\Bb_i$ and $\Bb_i$ for fixed $\Ab_i$, ignoring the coupling between the two blocks. We instead rewrite \eqref{eq:constrained} as a bilevel problem~\citep{dempe2002foundations}:}
\vspace{-0.06cm}
\begin{align}
    \nonumber&\min_{\Bb_i\in\Bc_i} \Lc_i(\Ab_i^*(\Bb_i), \Bb_i)\\&
    \text{subject to}\quad \Ab_i^*(\Bb_i) = \argmin\limits_{\Ab_i}\Lc_i(\Ab_i, \Bb_i).
    \label{eq:jointthetawnest}
\end{align}
To solve the lower-level optimization problem, we use a quadratic approximation $\hat{\Lc}_i(\Ab_i,\Bb_i)$ of $\Lc_i(\Ab_i, \Bb_i)$
\vspace{-0.25cm}
\begin{equation}\label{eq:quad-A}
    \hat{\Lc}_i(\Ab_i,\Bb_i;\Ab_i^{(k)}) = \Lc_i(\Ab_i^{(k)}, \Bb_i) + \bigl\langle \nabla_{\Ab_i}\Lc_i(\Ab_i^{(k)}, \Bb_i),\; \Ab_i - \Ab_i^{(k)} \bigr\rangle + \frac{1}{2\alpha}\bigl\|\Ab_i - \Ab_i^{(k)}\bigr\|_F^2.
\end{equation}
Substituting this quadratic approximation into the lower-level optimization problem in~\eqref{eq:jointthetawnest} leads to:
\vspace{-0.1cm}
\begin{equation}\label{eq:inner}
    \Ab_i^{(k+1)}(\Bb_i) = \arg\min_{\Ab_i}\; \hat{\Lc}_i(\Ab_i,\Bb_i;\Ab_i^{(k)}) = \Ab_i^{(k)} - \alpha\,\nabla_{\Ab_i}\Lc_i(\Ab_i^{(k)}, \Bb_i),
\end{equation}
where $\alpha$ is a step size. Note that the model update depends on the $\Bb_i$. This dependence is accounted for in the upper-level optimization in \eqref{eq:jointthetawnest}, whose cost function becomes
\vspace{-0.15cm}
\begin{equation}\label{eq:reduced}
    \Phi_i(\Bb_i) \triangleq \Lc_i\!\bigl(\Ab_i^{(k+1)}(\Bb_i),\, \Bb_i\bigr) = \Lc_i\!\Bigl(\Ab_i^{(k)} - \alpha\,\nabla_{\Ab_i}\Lc_i(\Ab_i^{(k)}, \Bb_i),\;\Bb_i\Bigr),
\end{equation}
which is now a function of $\Bb_i$ alone. The constrained upper-level problem is
\vspace{-0.1cm}
\begin{equation}\label{eq:B-constrained}
    \min_{\Bb_i\in \Bc_i}\; \Phi_i(\Bb_i).
\end{equation}
{By \eqref{eq:reduced}, $\Bb_i$ enters $\Phi_i$ twice: directly through $\nabla_{\Bb_i}\Lc_i$, and indirectly through the dependence of the updated factor $\Ab_i^{(k+1)}(\Bb_i)$ on $\Bb_i$, which is captured by the mixed derivative $\nabla^2_{\Bb_i\Ab_i}\Lc_i$. Solving \eqref{eq:B-constrained} therefore seeks a basis that both lowers the loss at the current iterate and steers the lower-level update of $\Ab_i$ toward directions that lower it further, while the projection onto $\Bc_i$ maintains orthogonality across concepts.}

Since $\Phi_i$ is non-convex and $\Bc_i$ is a constraint set, we again use a quadratic model $\hat{\Phi}_i(\Bb_i)$ around $\Bb_i^{(k)}$:
\vspace{-0.2cm}
\begin{equation}\label{eq:quad-B}
    \hat{\Phi}_i(\Bb_i;\Bb_i^{(k)}) \;= \; \Phi_i(\Bb_i^{(k)}) + \bigl\langle \nabla_{\Bb_i}\Phi_i(\Bb_i^{(k)}),\; \Bb_i - \Bb_i^{(k)} \bigr\rangle + \frac{1}{2\beta}\bigl\|\Bb_i - \Bb_i^{(k)}\bigr\|_F^2,
\end{equation}
where $\beta>0$ is the upper-level step-size. Note that evaluating \eqref{eq:inner} at $\Bb_i=\Bb_i^{(k)}$ yields $\Lc_i(\Ab_i^{(k+1)}(\Bb_i^{(k)}),\Bb_i)$, which recovers the standard alternating-minimization pattern: (i) fix $\Bb_i=\Bb_i^{(k)}$ and update $\Ab_i$, then (ii) fix $\Ab_i=\Ab_i^{(k+1)}$ and update $\Bb_i$.
However, unlike standard alternating minimization, which descends only the partial objective $\Lc_i(\Ab_i^{(k+1)}, \cdot)$ with frozen $\Ab_i^{(k+1)}$, the bilevel gradient $\nabla_{\Bb_i}\Phi_i$ additionally captures the coupling between factors through a cross-Hessian correction $\nabla^2_{\Bb_i\Ab_i}\Lc_i$, yielding a richer descent direction on the true reduced objective. We include this standard alternating minimization variant (without the cross-Hessian correction) as a baseline in our experiments, denoted \texttt{Continual Alternating} in Section~\ref{sec:experiments}.

Minimizing this quadratic model \eqref{eq:quad-B} subject to the orthogonality constraints replaces \eqref{eq:B-constrained} by
\begin{align}
     \Bb^{(k+1)}_i &= \argmin_{\Bb_i\in\Bc_i} \hat{\Phi}_i(\Bb_i;\Bb_i^{(k)}).
    \label{eq:findweighquadratic}
\end{align}
By taking the gradient of the main cost function and dropping constant terms, the above update is equivalent to
\vspace{-0.2cm}
\begin{equation}\label{eq:nearest}
    \min_{\Bb_i\in\Bc_i}\; \bigl\|\Bb_i - \widetilde{\Bb}_i\bigr\|_F^2,
\end{equation}
where
\begin{equation}\label{eq:B-tilde}
    \widetilde{\Bb}_i = \Bb_i^{(k)} - \beta\,\nabla_{\Bb_i}\Phi_i(\Bb_i^{(k)}).
\end{equation}
We use the chain rule to compute the gradient $\nabla_{\Bb_i}\Phi_i(\Bb_i^{(k)})$. So, it is straightforward to show that
\begin{equation}\label{eq:grad-Phi-exact-vec}
    \nabla_{\Bb_i}\Phi_i(\Bb_i^{(k)}) = \nabla_{\Bb_i}\Lc_i(\Ab_i^{(k+1)}(\Bb_i^{(k)}), \Bb_i)|_{\Bb_i = \Bb_i^{(k)}} -\alpha\,\Hc_{\Ab\Bb}^{(k)}\!\left[\nabla_{\Ab_i}\Lc_i\!\bigl({\Ab}_i^{(k+1)}(\Bb_i^{(k)}), \Bb_i^{(k)}\bigr)\right].
\end{equation}
The term $\Hc_{\Ab\Bb}^{(k)}[\cdot]: \Rbb^{n \times r} \to \Rbb^{m \times r}$ is the cross-Hessian contraction operator. Directly instantiating this 4th-order tensor might be computationally prohibitive. Instead, we compute the required Hessian-matrix product efficiently using a nested automatic differentiation strategy (a double grad or Hessian-vector product trick). For a given inner residual $\mathbf{g_A} = \nabla_{\Ab_i}\Lc_i$, this operation is mathematically equivalent to:
\vspace{-0.2cm}
\begin{equation}\label{eq:hvp}
    \Hc_{\Ab\Bb}^{(k)}[\mathbf{g_A}] = \nabla_{\Bb_i} \Tr\left(\mathbf{g_A}^\top \nabla_{\Ab_i}\Lc_i(\Ab_i^{(k)}, \Bb_i)\right) \Bigg|_{\Bb_i = \Bb_i^{(k)}}.
\end{equation}
The optimization problem~\eqref{eq:nearest} is a convex quadratic program with linear equality constraints. Let $\Bb_{\mathrm{int}} =[\Bb_1, \dots, \Bb_{i-1}] \in \Rbb^{m \times (i-1)r}$ denote the column-wise concatenation of all previously learned bases, allowing us to express the orthogonality constraint compactly as $\Bb_{\mathrm{int}}^\top\Bb_i = \mathbf{0}$.

The minimizer of this problem is simply the exact orthogonal projection of the unconstrained step $\widetilde{\Bb}_i$ onto the null space of $\Bb_{\mathrm{int}}^\top$. By solving the KKT conditions, this yields the closed-form geometric update:
\vspace{-0.1cm}
\begin{equation}\label{eq:update-B}
    \Bb_i^{(k+1)} = \underbrace{\left[\Ib - \Bb_{\mathrm{int}}\!\left(\Bb_{\mathrm{int}}^\top \Bb_{\mathrm{int}}\right)^{-1}\!\Bb_{\mathrm{int}}^\top\right]}_{\triangleq\, \Pc^{\perp}_{B_i}} \widetilde{\Bb}_i = \operatorname{Proj}_{\operatorname{span}\{\Bb_j\}_{j<i}^{\perp}}\!\left(\widetilde{\Bb}_i\right).
\end{equation}

The projector $\Pc^\perp_{B_i}$ strips gradient components overlapping with
past concepts. Given $\Bb^{(k+1)}_i$, the lower-level factor is updated by~\eqref{eq:update-B}:
\begin{equation}
\Ab^{(k+1)}_i = \Ab^{(k)}_i - \alpha \nabla_{\Ab_i} \Lc_i\!\left(\Ab^{(k)}_i, \Bb^{(k+1)}_i\right).
\end{equation}
After $K$ bilevel iterations we freeze $(\Ab_i, \Bb_i)$, append $\Bb_i$ to $\Bb_{\text{int}}$,
and update $\Wb_i = \Wb_0 + \sum_{j \le i} \Ab_j \Bb_j^\top$. The same procedure is applied to all adapted layers, and processing the concept stream $i = 1, \ldots, T$ in sequence yields the full method (Algorithm~\ref{alg:seqlora}{; in practice each bilevel iteration uses $S_A$ and $S_B$ local gradient steps for the two factors}). {The cross-Hessian--vector product costs one additional backward pass per $\Bb$-step; Section~\ref{sec:efficiency} quantifies the resulting overhead.}
\vspace{-0.2cm}
\subsection{Theoretical Results}
\vspace{-0.2cm}
SeqLoRA uses quadratic surrogate models $\hat{\Lc}_i$ and $\hat{\Phi}_i$ to simplify bilevel optimization. A key question is whether minimizing these local approximations guarantees convergence to a critical point of the main loss function in \eqref{eq:constrained}, which is addressed in the following theorem.
\begin{theorem}\label{thm:smoothcon}
Under Assumption \ref{assump:smoothness} (see Appendix~\ref{sec:mathfound}), let $ \Lc_i(\Ab, \Bb) + \delta_{\Bc_i}(\Bb)$ be the composite objective, $\delta_{\Bc_i}(\Bb)$ is an indicator function, and let $\{(\Ab_i^{(k)}, \Bb_i^{(k)})\}_{k=0}^K$ be generated by Algorithm \ref{alg:seqlora}. If the step sizes satisfy $0 < \alpha < \frac{1}{L}$ and $0 < \beta < \frac{1}{L_\Phi}$, where $L_\Phi \le L(1 + \alpha L)^2 + \alpha \rho M$, then:
\vspace{-0.3cm}
\begin{enumerate}
    \item \textbf{Monotonic Descent:} The composite objective decreases monotonically at each iteration $k$. In particular, the step sizes $\alpha^* = \frac{1}{2L}$ and $\beta^* = \frac{1}{2L_\Phi}$ maximize the descent coefficients, yielding the sharpest guaranteed decrease.
    \item \textbf{Asymptotic Stationarity:} The sequence of iterates $\{(\Ab_i^{(k)}, \Bb_i^{(k)})\}$ asymptotically converges to a critical point of the constrained optimization problem \eqref{eq:constrained}.
\end{enumerate}
\end{theorem}
\vspace{-0.4cm}
\begin{proof}
    See Appendix \ref{sec:proofthm}.
    \vspace{-0.2cm}
\end{proof}
{We next bound the \emph{forgetting} of concept $j$, defined as the increase $\Lc_j(\theta_T)-\Lc_j(\theta_j)$ of its loss between the model $\theta_j$ obtained right after learning it and the final model $\theta_T$.} Building on Example~\ref{exa:linearwarmup}, we model the part of each layer's activations \emph{not} captured by the learned basis as a matrix sub-Gaussian process and propagate the resulting per-layer crosstalk through the network.
{\begin{theorem}[Hanson-Wright High-Probability Bound on Catastrophic Forgetting] \label{thm:hanson_wright_catas}
Under Assumption \ref{assum:multilayer_subgaussian} (see Appendix~\ref{sec:mathfound}), let $\Gamma_\ell \triangleq \prod\limits_{\ell'=\ell+1}^{L} \gamma_{\ell'} \nt{\Wb_T^{(\ell')}}$ be the downstream amplification factor at layer $\ell$, let $\Cb_j^{(\ell)} = \sum\limits_{k=j+1}^T\Ab_k^{(\ell)}(\Bb_k^{(\ell)})^\top$ be the crosstalk operator for concept $j$ with $\kappa_\ell \triangleq \nt{\Cb_j^{(\ell)}}$, and let $\epsilon_\ell \triangleq \Tr(\Psib_j^{(\ell)}) \Tr(\Sigmab_j^{(\ell),\perp})$ be the residual
activation energy at layer $\ell$. For any $\xi \in (0,1)$, with probability at least $1-\xi$, the end-to-end catastrophic forgetting satisfies
\begin{equation}\label{eq:e2e_boundfin}
    \Lc_j(\theta_T) - \Lc_j(\theta_j) = \mathcal{O}\!\left(\textstyle\sum\limits_{\ell=1}^{L} \Gamma_\ell\kappa_\ell \sqrt{\epsilon_\ell}\right),
\end{equation}
where $\theta_T=\{\Wb_T^{(\ell)}\}_{\ell=1}^L$ and $\theta_j=\{\Wb_j^{(\ell)}\}_{\ell=1}^L$ are the full model's parameters after learning concept $T$ and $j$, respectively.
 Additionally, minimizing the expected residual energy $\Tr(\Psib_j^{(\ell)}) \Tr(\Sigmab_j^{(\ell),\perp})$ requires the rank-$r$ basis $\Bb_j^{(\ell)}$ to span the top-$r$ principal components of $\boldsymbol{\Sigma}_j^{(\ell)}$, the activation covariance of concept $j$ at layer $\ell$, within the available orthogonal complement. Thus, randomly frozen bases are generically suboptimal, yielding higher residual energy unless the remaining feature space is isotropic.
\end{theorem}}
\vspace{-0.4cm}
\begin{proof}
    See Appendix \ref{sec:prooftheorem2}.
    \vspace{-0.2cm}
\end{proof}
The explicit form of the bound, with all constants, is \eqref{eq:e2e_boundfin1} in Appendix~\ref{sec:prooftheorem2}. Theorem~\ref{thm:hanson_wright_catas} shows that forgetting at each layer is governed by the residual covariance
$\Sigmab_j^{(\ell),\perp}$, which measures the activation energy that \emph{leaks} outside concept $j$'s learned
subspace $\mathrm{col}(\Bb_j^{(\ell)})$. Orthogonality ensures this leaked energy is the only source of
interference. The downstream factors $\Gamma_\ell$ is typical of norm-based bounds for deep networks~\citep{bartlett2017spectrally,neyshabur2018pac}.
{The bound is minimized when each $\Bb_j^{(\ell)}$ aligns with the top principal components of its local activations. SeqLoRA does not necessarily compute this eigenspace; it learns $\Bb_j^{(\ell)}$ by minimizing the end-to-end denoising loss (Remark~\ref{rem:relseqlora}). The theorem therefore explains \emph{why} a learned basis can reduce interference where a frozen one cannot. Appendix~\ref{subsec:interpre} interprets the constants, Appendix~\ref{sec:theory_numerics} provides a numerical verification in a small nonlinear network, while Section~\ref{sec:experiments} empirically evaluates the resulting prediction that learned orthogonal bases exhibit less forgetting and interference than randomly frozen bases.}
\vspace{-0.3cm}
\section{Experiments}\label{sec:experiments}
\vspace{-0.2cm}
\paragraph{Baselines and data.} We evaluate SeqLoRA against state-of-the-art baselines: Mix-of-Show~\citep{gu2023mix}, Orthogonal Adaptation~\citep{po2024orthogonal}, LoRACLR~\citep{simsar2025loraclr}, and \texttt{Continual Alternating} (see Subsection~\ref{subsec:bilevel}) to isolate the contribution of the cross-Hessian coupling. To assess scalability,
we vary the number of concepts from 8 to 101. All methods are evaluated using prompts from the
CustomConcept101 dataset~\citep{kumari2023multi}, generating 16 images per prompt for metric computation.

\paragraph{Metrics and protocols.} We measure identity preservation with DINO/DINOv2/DINOv3, holistic similarity with DreamSim, image and text alignment with CLIP-I/CLIP-T, and visual quality with HPSv2/HPSv3 (Appendix~\ref{app:metrics}). {Three protocols are used. \emph{Single-concept retention} queries the final $T$-concept model for each concept in isolation and tests whether learning $T$ concepts into one model degrades any of them. \emph{Multi-concept composition} renders three or four concepts in one image under the regional sketch/keypose protocol of \citet{gu2023mix} and measures, from each concept's region, its DINOv2 similarity to its own references (\emph{identity}) and to the co-composed concepts' references (\emph{leakage}); it tests interference directly. \emph{Forgetting} compares each concept's score right after it is learned with its score in the final model.} Tables report mean $\pm$ standard error across concepts; figures report means (Appendix~\ref{sec:appendix_eval_methodology}). SeqLoRA uses $K=3$ bilevel iterations, $S_A=S_B=2$ local steps, $r=4$, and $\epsilon = 10^{-8}$ unless stated otherwise; all methods use the same rank and data per concept, on $2\times$ NVIDIA A100 GPUs.
\subsection{Main results}
\begin{table}[t]
    \centering
    \caption{\textbf{Single-concept retention after learning 32 concepts.} Mean $\pm$ standard error over the 32 concepts (16 images each); best in bold, second best underlined.}
    \vspace{-0.1cm}
    \resizebox{\textwidth}{!}{%
    \begin{tabular}{@{}lcccccccc@{}}
    \toprule
     & \multicolumn{5}{c}{Reference similarity} & \multicolumn{3}{c}{Prompt following \& aesthetics} \\
    \cmidrule(lr){2-6} \cmidrule(lr){7-9}
    Method & DINO$\uparrow$ & DINOv2$\uparrow$ & DINOv3$\uparrow$ & DreamSim$\uparrow$ & CLIP-I$\uparrow$ & CLIP-T$\uparrow$ & HPSv2$\uparrow$ & HPSv3$\uparrow$ \\
    \midrule
    Mix-of-Show  & 0.436 $\pm$ 0.023  & 0.377 $\pm$ 0.020  & 0.378 $\pm$ 0.017  & 0.448 $\pm$ 0.017  & 0.677 $\pm$ 0.010  & 0.276 $\pm$ 0.003  & \underline{0.271} $\pm$ 0.001  & 8.039 $\pm$ 0.457 \\
    Orthogonal Adaptation  & 0.428 $\pm$ 0.023  & 0.361 $\pm$ 0.020  & 0.359 $\pm$ 0.017  & 0.442 $\pm$ 0.018  & 0.671 $\pm$ 0.009  & 0.277 $\pm$ 0.003  & 0.271 $\pm$ 0.001  & 8.336 $\pm$ 0.432 \\
    LoRACLR  & 0.434 $\pm$ 0.023  & 0.372 $\pm$ 0.020  & 0.371 $\pm$ 0.017  & 0.445 $\pm$ 0.018  & 0.674 $\pm$ 0.010  & 0.277 $\pm$ 0.003  & \textbf{0.271} $\pm$ 0.001  & 8.382 $\pm$ 0.431 \\
    \CA{}  & \underline{0.463} $\pm$ 0.020  & \underline{0.412} $\pm$ 0.017  & \underline{0.406} $\pm$ 0.015  & \underline{0.470} $\pm$ 0.015  & \underline{0.688} $\pm$ 0.008  & \textbf{0.287} $\pm$ 0.004  & 0.271 $\pm$ 0.001  & \textbf{9.242} $\pm$ 0.338 \\
    \textbf{SeqLoRA (ours)}  & \textbf{0.468} $\pm$ 0.021  & \textbf{0.418} $\pm$ 0.019  & \textbf{0.418} $\pm$ 0.016  & \textbf{0.477} $\pm$ 0.016  & \textbf{0.693} $\pm$ 0.009  & \underline{0.286} $\pm$ 0.003  & 0.270 $\pm$ 0.001  & \underline{9.043} $\pm$ 0.399 \\
    \bottomrule
    \end{tabular}}
    \label{tab:overall_comparison_32}
    \vspace{-0.1cm}
\end{table}

\paragraph{Single-concept retention.} Table~\ref{tab:overall_comparison_32} reports the 32-concept experiment. {The two sequential methods, SeqLoRA and \CA{}, outperform all three fusion-based baselines on every identity and similarity metric (DINO, DINOv2, DINOv3, DreamSim, CLIP-I) by a wide margin, and SeqLoRA is the better of the two on all five; CLIP-T and HPSv2/HPSv3 are within one standard error across the sequential methods, with \CA{} slightly ahead on HPSv3. The dominant gain thus comes from the constrained continual formulation, and the bilevel update adds a consistent further improvement in identity preservation.} Figure~\ref{fig:qualitative} and Appendix~\ref{sec:supp_qual} show the images for all 32 concepts.

\begin{table}[t]
    \centering
     
    \caption{{\textbf{Quantitative multi-concept regional composition at 32 concepts} (DINOv2 features, 5 seeds per composition). Identity is similarity of each concept's region to its own references ($\uparrow$); leakage is similarity of each concept's region to \emph{co-composed} concepts' references ($\downarrow$); CLIP-T measures text alignment ($\uparrow$). SeqLoRA attains the lowest cross-concept leakage; the reduction over Continual Alternating is statistically significant (paired Wilcoxon $p=0.014$); margin is identity minus leakage}.}
    \vspace{-0.1cm}
    \small
    \begin{tabular}{@{}lcccc@{}}
    \toprule
    Method & Identity$\uparrow$ & Leakage$\downarrow$ & Margin$\uparrow$ & CLIP-T$\uparrow$ \\
    \midrule
    Mix-of-Show           & \textbf{0.374} & 0.234 & 0.140 & \textbf{0.262} \\
    Orthogonal Adaptation & 0.344 & 0.231 & 0.113 & 0.259 \\
    LoRACLR               & 0.349 & 0.230 & 0.119 & 0.258 \\
    \CA{}                 & \underline{0.366} & \underline{0.223} & \underline{0.143} & \underline{0.259} \\
    \textbf{SeqLoRA (ours)} & 0.347 & \textbf{0.203} & \textbf{0.144} & 0.257 \\
    \bottomrule
    \end{tabular}
    \label{tab:composition}
     
    \vspace{-0.3cm}
\end{table}

\paragraph{Multi-concept composition.} {Table~\ref{tab:composition} quantifies interference on multi-concept regional composition at 32 concepts. SeqLoRA attains the lowest cross-concept leakage of all methods, and its improvement over \CA{}, the closest method, is statistically significant (paired Wilcoxon test over compositions, $p=0.014$). Its in-region identity (0.347) lies within the range spanned by the baselines: Mix-of-Show and \CA{} obtain higher absolute identity but at higher leakage, so that the identity--leakage margin of SeqLoRA (0.144) is on par with \CA{} (0.143) and above the fusion-based methods, while text alignment is essentially unchanged across methods. Mix-of-Show, whose fusion step is optimized for exactly this composition, remains the strongest on absolute in-region identity. Leakage is the quantity the orthogonality constraint targets, and it is where learned and frozen orthogonal bases differ most: Orthogonal Adaptation, which also composes by summation, has leakage 0.231 against 0.203. Although SeqLoRA is not the best on the identity score, the lower leakage is visible in Figure~\ref{fig:regional_comparison} (Appendix~\ref{sec:supp_regional}): at the same seed and control inputs, SeqLoRA keeps each concept’s appearance inside
its own region, whereas the baselines mix attributes across regions.}

\subsection{Scalability and forgetting}\label{sec:scalability}
\begin{figure}[t]
    \centering
    \begin{subfigure}{\textwidth}
        \centering
        \includegraphics[width=1.01\textwidth]{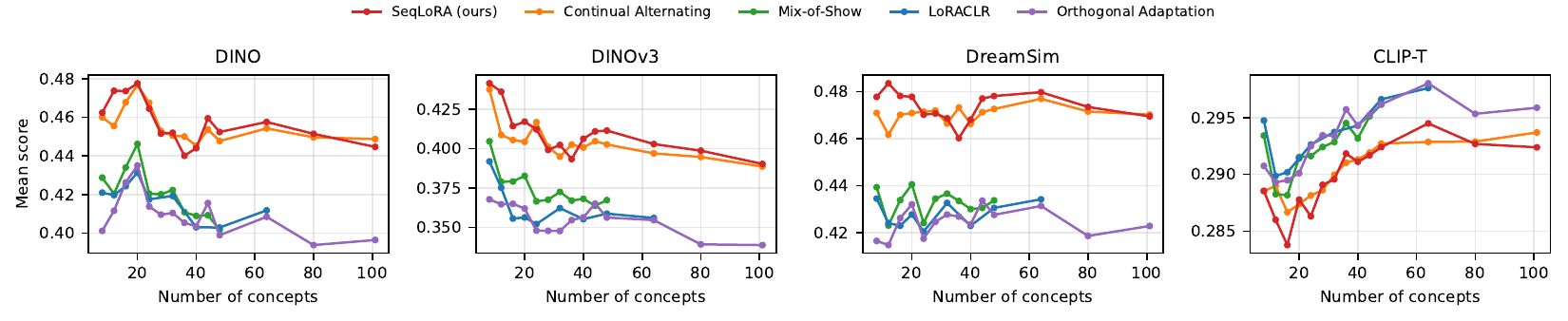}
        \vspace{-0.45cm}
        \caption{Final-model performance versus number of concepts ($T=8$ to $101$), averaged over all $T$ concepts. Mix-of-Show and LoRACLR run out of memory beyond 48 and 64 concepts.}
        \label{fig:overall_comparison}
    \end{subfigure}
    \begin{subfigure}{\textwidth}
        \centering
        \includegraphics[width=\textwidth]{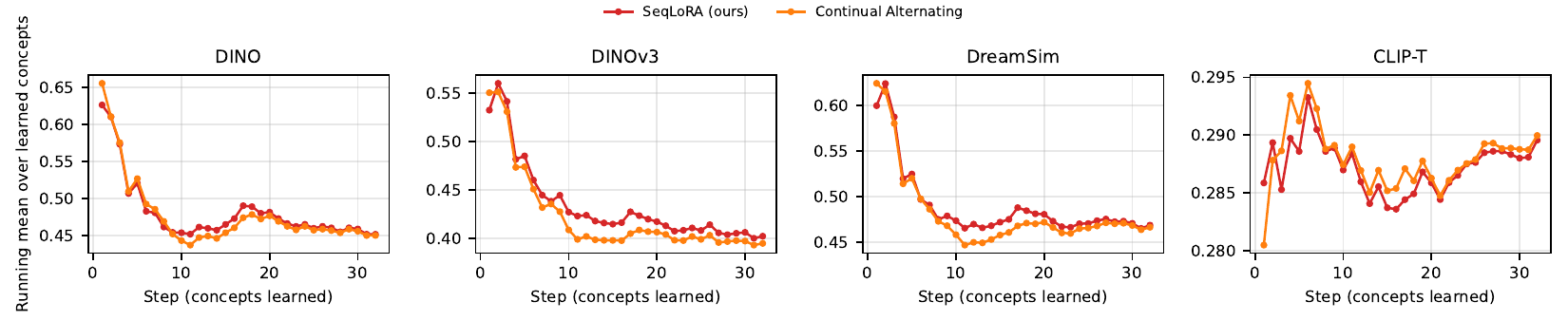}
        \vspace{-0.45cm}
        \caption{Running mean over all concepts learned so far, after each step of the 32-concept stream. {The evaluated pool grows with the step, so the curve mixes forgetting with the difficulty of newly added concepts; Table~\ref{tab:summary_forgetting_new} isolates forgetting per concept.}}
        \label{fig:forgetting}
    \end{subfigure}
    \vspace{-0.2cm}
    \caption{Scalability (a) and stability during sequential training (b) on four representative metrics; all eight metrics are in Figures~\ref{fig:overall_comparison_full} and \ref{fig:forgetting_full}.}
\end{figure}
\paragraph{Scalability.} Figure~\ref{fig:overall_comparison} reports final-model performance for $T\in\{8,\dots,101\}$ (all metrics in Table~\ref{tab:supp_overall}). {The identity metrics of the two sequential methods stay well above those of the fusion-based methods at every $T$, and the gap does not close as $T$ grows. Mix-of-Show and LoRACLR fail with out-of-memory errors at 64 and 80 concepts, because their fusion step must hold all per-concept adapters (and, for optimization-based fusion, the optimizer state) in memory at once; SeqLoRA and Orthogonal Adaptation, which compose by summation, run to 101 concepts (Section~\ref{sec:efficiency}).}
\begin{table}[t]
    \centering
    \caption{\textbf{Forgetting for SeqLoRA across stream lengths.} Each concept's score right after it is learned minus its score in the final $T$-concept model, mean $\pm$ standard error over concepts; initial and final scores in Table~\ref{tab:summary_forgetting_full}.}
    \vspace{-0.1cm}
    \resizebox{\textwidth}{!}{%
    \begin{tabular}{@{}lcccccccc@{}}
    \toprule
    $T$ & CLIP-I & CLIP-T & DINO & DINOv2 & DINOv3 & DreamSim & HPSv2 & HPSv3 \\
    \midrule
    8  & 0.014 $\pm$ 0.006 & $-$0.002 $\pm$ 0.001 & 0.022 $\pm$ 0.009 & 0.025 $\pm$ 0.009 & 0.026 $\pm$ 0.010 & 0.017 $\pm$ 0.008 & $-$0.001 $\pm$ 0.000 & $-$0.025 $\pm$ 0.020 \\
    16 & 0.009 $\pm$ 0.004 & $-$0.002 $\pm$ 0.001 & 0.018 $\pm$ 0.007 & 0.020 $\pm$ 0.007 & 0.022 $\pm$ 0.008 & 0.014 $\pm$ 0.005 & $-$0.000 $\pm$ 0.000 & 0.019 $\pm$ 0.029 \\
    32 & 0.008 $\pm$ 0.003 & $-$0.001 $\pm$ 0.000 & 0.013 $\pm$ 0.004 & 0.016 $\pm$ 0.005 & 0.017 $\pm$ 0.005 & 0.012 $\pm$ 0.004 & $-$0.000 $\pm$ 0.000 & 0.001 $\pm$ 0.025 \\
    \bottomrule
    \end{tabular}}
    \label{tab:summary_forgetting_new}
    \vspace{-0.3cm}
\end{table}

\paragraph{Forgetting.} Figure~\ref{fig:forgetting} tracks the running mean over all concepts learned so far along the 32-concept stream: SeqLoRA is at or above \CA{} on the visual-fidelity metrics at almost every step. {Because the evaluated pool grows, this curve measures system-level stability; Table~\ref{tab:summary_forgetting_new} therefore compares each concept's score right after it is learned with its score in the final model. The drop is small (e.g., $0.013$ DINO and $0.008$ CLIP-I at 32 concepts) and does not grow with the stream length, while CLIP-T and HPSv2 are unchanged: SeqLoRA mitigates forgetting substantially but does not eliminate it, consistent with Theorem~\ref{thm:hanson_wright_catas}.} Appendix~\ref{sec:supp_forgetting} gives per-step numbers and images of every concept across training.

\subsection{Efficiency}\label{sec:efficiency}
\begin{table}[t]
    \centering
    
    \caption{\textbf{End-to-end cost at 32 concepts on the same hardware.} Training time per concept, and fusion time and peak memory for one composition. The fusion step is repeated for every new subset of concepts; SeqLoRA and \CA{} compose by summation (memory versus $T$ in Table~\ref{tab:memory}).}
    \vspace{-0.1cm}
    \small
    \begin{tabular}{@{}lccc@{}}
    \toprule
    Method & Train / concept & Fusion time & Peak fusion mem. \\
    \midrule
    Orthogonal Adaptation & ${\sim}78$\,s & 15\,s & --- \\
    LoRACLR               & ${\sim}78$\,s & 8\,min & 44.7\,GB \\
    Mix-of-Show           & ${\sim}78$\,s & 33\,min & 18.4\,GB \\
    \CA{}                 & ${\sim}454$\,s & none & --- \\
    \textbf{SeqLoRA (ours)} & ${\sim}660$\,s & none & --- \\
    \bottomrule
    \end{tabular}
    \label{tab:efficiency}
    
    \vspace{-0.3cm}
\end{table}

Table~\ref{tab:efficiency} compares end-to-end cost at 32 concepts on the same hardware. Parallel methods (LoRACLR, Mix-of-Show, Orthogonal Adaptation) train per-concept adapters independently (${\sim}78$s each) and fuse them post-hoc, while SeqLoRA and Continual Alternating learn sequentially in a single model with no fusion step. The accumulating orthogonality constraints raise the per-concept cost to ${\sim}454$s for Continual Alternating and ${\sim}660$s for SeqLoRA (the latter due to cross-Hessian evaluations). However, while parallel methods can leverage distributed computing to reduce wall-clock training time (at the cost of higher total compute), they require a complex and often time-consuming optimization step to fuse the independent adapters into a single multi-concept model. At $32$ concepts, post-hoc fusion itself takes ${\sim}33$ min for Mix-of-Show and ${\sim}8$ min for LoRACLR (Orthogonal Adaptation fuses in $15$s via simple summation), an overhead SeqLoRA avoids entirely while also supporting truly sequential lifelong learning without data accumulation.

\subsection{Robustness and ablations}\label{sec:robustness}
  
Full tables are in Appendix~\ref{sec:supp_experiments}.

\emph{Concept order.} Over four random orderings of a 16-concept stream (five seeds each), composition identity varies by $\pm 0.007$ and leakage by $\pm 0.006$ (coefficient of variation ${\le}3\%$; Table~\ref{tab:order}), so the sequential projection does not make later concepts systematically worse.

\emph{LoRA rank.} Raising the rank from $4$ to $8$ or $16$ increases leakage (0.200 to 0.221/0.219) without improving the identity--leakage margin (Table~\ref{tab:rank}); since the constraint admits at most $\lfloor m/r\rfloor$ concepts per layer, a small rank is preferable for capacity and interference alike.

\emph{Concept similarity.} The projection matters most where interference is largest: for visually similar concepts it reduces composition leakage from $0.434$ to $0.386$, whereas for a diverse group leakage is already low and unchanged ($0.070$ vs.\ $0.078$; Table~\ref{tab:similarity}).

\emph{Hyperparameters.} Varying local steps, bilevel iterations and $\epsilon$ changes identity metrics by less than one standard error (Table~\ref{tab:ablation_32}).

\emph{Backbone.} The projection is backbone-agnostic; a preliminary port to SDXL~\citep{podell2024sdxl} (four concepts) trains unchanged and reaches a mean DINO identity of $0.67$ with a $+0.47$ margin over cross-concept similarity (Table~\ref{tab:sdxl}).
  
\section{Limitations and Discussion}\label{sec:limitations}
  
\emph{Capacity.} The orthogonality constraint admits at most $\lfloor m/r \rfloor$ concepts per layer (192 for $m{=}768$, $r{=}4$), which covers the scales we evaluate; later concepts are learned in a shrinking complement, and raising $r$ increases leakage (Table~\ref{tab:rank}).

\emph{Compute.} The bilevel formulation adds a cross-Hessian–vector product per iteration, increasing per-concept training cost (Table~\ref{tab:efficiency}), but avoids the costly post-hoc fusion required by parallel methods (e.g., ${\sim}33$ minutes for Mix-of-Show at 32 concepts).

\emph{Theory and scope.} Our theoretical results rely on Assumptions~\ref{assump:smoothness} and~\ref{assum:multilayer_subgaussian}, which are mild standard conditions discussed in Appendix~\ref{sec:mathfound}. Current experiments use Stable Diffusion v1.5 for consistency with prior work; a preliminary extension to SDXL is provided in Table~\ref{tab:sdxl}.
  
\section{Conclusion}\label{sec:conclusion}
\vspace{-0.3cm}
{SeqLoRA casts multi-concept customization as constrained continual learning: both LoRA factors of each new concept are learned, and the basis is kept orthogonal to all previous bases by a closed-form projection inside a bilevel update. We prove monotone descent and convergence to critical points of the constrained problem, derive a high-probability forgetting bound for nonlinear multi-layer networks whose basis-dependent term is the residual activation energy, and show that a data-adapted basis reduces this term where a random frozen basis cannot. On Stable Diffusion with up to 101 concepts, SeqLoRA improves identity preservation over fusion-based methods, attains the lowest cross-concept leakage in compositions, needs no fusion step, and scales past the point at which fusion runs out of memory; the bilevel correction gives a consistent further gain over alternating minimization.}

\section*{Ethics Statement}
SeqLoRA advances multi-concept personalization of text-to-image diffusion models, which has broad positive applications in creative design, media production, and accessibility. As with all generative modeling research, improvements in image fidelity could potentially be misused for generating misleading visual content. However, SeqLoRA does not introduce new generative capabilities beyond what existing diffusion models already provide; it improves the composability of personalized adaptations within established frameworks. We encourage the responsible use of such technology in accordance with applicable guidelines and regulations.

\section*{AI Use Statement}
LLMs were used only to edit the text for readability and to review the manuscript for consistency of notation. They were not used to develop the method, to formulate or prove the mathematical claims, or to design experiments. All AI-assisted edits were reviewed and verified by the authors.
  
\bibliography{myref}

\begin{thebibliography}{36}
\providecommand{\natexlab}[1]{#1}
\providecommand{\url}[1]{\texttt{#1}}
\expandafter\ifx\csname urlstyle\endcsname\relax
  \providecommand{\doi}[1]{doi: #1}\else
  \providecommand{\doi}{doi: \begingroup \urlstyle{rm}\Url}\fi

\bibitem[Attouch et~al.(2011)Attouch, Bolte, and Svaiter]{articlesubgradient}
Hedy Attouch, Jérôme Bolte, and Benar Svaiter.
\newblock Convergence of descent methods for semi-algebraic and tame problems:
  Proximal algorithms, forward-backward splitting, and regularized gauss-seidel
  methods.
\newblock \emph{Mathematical Programming}, 137, 01 2011.
\newblock \doi{10.1007/s10107-011-0484-9}.

\bibitem[Avrahami et~al.(2023)Avrahami, Aberman, Fried, Cohen-Or, and
  Lischinski]{avrahami2023break}
Omri Avrahami, Kfir Aberman, Ohad Fried, Daniel Cohen-Or, and Dani Lischinski.
\newblock Break-a-scene: Extracting multiple concepts from a single image.
\newblock In \emph{SIGGRAPH Asia 2023 Conference Papers}, pages 1--12, 2023.

\bibitem[Bartlett et~al.(2017)Bartlett, Foster, and
  Telgarsky]{bartlett2017spectrally}
Peter~L Bartlett, Dylan~J Foster, and Matus~J Telgarsky.
\newblock Spectrally-normalized margin bounds for neural networks.
\newblock In \emph{Advances in Neural Information Processing Systems
  (NeurIPS)}, pages 6240--6249, 2017.

\bibitem[Bolte et~al.(2014)Bolte, Sabach, and Teboulle]{bolte:hal-00916090}
Jerome Bolte, Shoham Sabach, and Marc Teboulle.
\newblock {Proximal alternating linearized minimization for nonconvex and
  nonsmooth problems}.
\newblock \emph{{Mathematical Programming}}, 146\penalty0 (1-2):\penalty0
  459--494, 2014.
\newblock \doi{10.1007/s10107-013-0701-9}.
\newblock URL \url{https://inria.hal.science/hal-00916090}.

\bibitem[Caron et~al.(2021)Caron, Touvron, Misra, J{\'e}gou, Mairal,
  Bojanowski, and Joulin]{caron2021emerging}
Mathilde Caron, Hugo Touvron, Ishan Misra, Herv{\'e} J{\'e}gou, Julien Mairal,
  Piotr Bojanowski, and Armand Joulin.
\newblock Emerging properties in self-supervised vision transformers.
\newblock In \emph{Proceedings of the IEEE/CVF international conference on
  computer vision}, pages 9650--9660, 2021.

\bibitem[Dawid(1981)]{dawid1981some}
A~Philip Dawid.
\newblock Some matrix-variate distribution theory: notational considerations
  and a bayesian application.
\newblock \emph{Biometrika}, 68\penalty0 (1):\penalty0 265--274, 1981.

\bibitem[Dempe(2002)]{dempe2002foundations}
Stephan Dempe.
\newblock \emph{Foundations of Bilevel Programming}, volume~61 of
  \emph{Nonconvex Optimization and Its Applications}.
\newblock Springer Science \& Business Media, 2002.
\newblock \doi{10.1007/978-1-4615-1503-8}.

\bibitem[Drusvyatskiy et~al.(2019)Drusvyatskiy, Lewis, and
  Paquette]{drusvyatskiy2019efficiency}
Dmitriy Drusvyatskiy, Adrian~S Lewis, and Catherine~A Paquette.
\newblock Efficiency of the prox-linear algorithm for composite optimization.
\newblock \emph{Mathematical Programming}, 178:\penalty0 503--558, 2019.

\bibitem[Fallah et~al.(2020)Fallah, Mokhtari, and
  Ozdaglar]{fallah2020convergence}
Alireza Fallah, Aryan Mokhtari, and Asuman Ozdaglar.
\newblock On the convergence theory of gradient-based model-agnostic
  meta-learning algorithms.
\newblock In \emph{International conference on artificial intelligence and
  statistics}, pages 1082--1092. PMLR, 2020.

\bibitem[Fu et~al.(2023)Fu, Tamir, Sundaram, Chai, Zhang, Dekel, and
  Isola]{fu2023dreamsim}
Stephanie Fu, Netanel Tamir, Shobhita Sundaram, Lucy Chai, Richard Zhang, Tali
  Dekel, and Phillip Isola.
\newblock Dreamsim: Learning new dimensions of human visual similarity using
  synthetic data.
\newblock \emph{arXiv preprint arXiv:2306.09344}, 2023.

\bibitem[Gal et~al.(2023)Gal, Alaluf, Atzmon, Patashnik, Bermano, Chechik, and
  Cohen-Or]{gal2023image}
Rinon Gal, Yuval Alaluf, Yuval Atzmon, Or~Patashnik, Amit~Haim Bermano, Gal
  Chechik, and Daniel Cohen-Or.
\newblock An image is worth one word: Personalizing text-to-image generation
  using textual inversion.
\newblock In \emph{The Eleventh International Conference on Learning
  Representations}, 2023.

\bibitem[Gu et~al.(2023)Gu, Wang, Wu, Shi, Chen, Fan, Xiao, Zhao, Chang, Wu,
  et~al.]{gu2023mix}
Yuchao Gu, Xintao Wang, Jay~Zhangjie Wu, Yujun Shi, Yunpeng Chen, Zihan Fan,
  Wuyou Xiao, Rui Zhao, Shuning Chang, Weijia Wu, et~al.
\newblock Mix-of-show: Decentralized low-rank adaptation for multi-concept
  customization of diffusion models.
\newblock \emph{Advances in Neural Information Processing Systems},
  36:\penalty0 15890--15902, 2023.

\bibitem[Gupta and Nagar(2000)]{gupta2018matrix}
A.~K. Gupta and D.~K. Nagar.
\newblock \emph{Matrix Variate Distributions}.
\newblock Chapman \& Hall/CRC, 2000.

\bibitem[Ho et~al.(2020)Ho, Jain, and Abbeel]{ho2020denoising}
Jonathan Ho, Ajay Jain, and Pieter Abbeel.
\newblock Denoising diffusion probabilistic models.
\newblock \emph{Advances in neural information processing systems},
  33:\penalty0 6840--6851, 2020.

\bibitem[Horn and Johnson(1991)]{horn1991topics}
Roger~A. Horn and Charles~R. Johnson.
\newblock \emph{Topics in Matrix Analysis}.
\newblock Cambridge University Press, Cambridge, 1991.

\bibitem[Horn and Johnson(2012)]{horn2012matrix}
Roger~A. Horn and Charles~R. Johnson.
\newblock \emph{Matrix Analysis}.
\newblock Cambridge University Press, 2nd edition, 2012.

\bibitem[Hu et~al.(2022)Hu, Shen, Wallis, Allen-Zhu, Li, Wang, Wang, and
  Chen]{hu2022lora}
Edward~J. Hu, Yelong Shen, Phillip Wallis, Zeyuan Allen-Zhu, Yuanzhi Li, Shean
  Wang, Lu~Wang, and Weizhu Chen.
\newblock {LoRA}: Low-rank adaptation of large language models.
\newblock In \emph{International Conference on Learning Representations
  (ICLR)}, 2022.

\bibitem[Kumari et~al.(2023)Kumari, Zhang, Zhang, Shechtman, and
  Zhu]{kumari2023multi}
Nupur Kumari, Bingliang Zhang, Richard Zhang, Eli Shechtman, and Jun-Yan Zhu.
\newblock Multi-concept customization of text-to-image diffusion.
\newblock In \emph{Proceedings of the IEEE/CVF conference on computer vision
  and pattern recognition}, pages 1931--1941, 2023.

\bibitem[Liang and Li(2024)]{liang2024inflora}
Yan-Shuo Liang and Wu-Jun Li.
\newblock Inflora: Interference-free low-rank adaptation for continual
  learning.
\newblock In \emph{Proceedings of the IEEE/CVF Conference on Computer Vision
  and Pattern Recognition}, pages 23638--23647, 2024.

\bibitem[Luo et~al.(2025)Luo, Gustafsson, Zhao, Sj{\"o}lund, and
  Sch{\"o}n]{luo2025taming}
Ziwei Luo, Fredrik Gustafsson, Zheng Zhao, Jens Sj{\"o}lund, and Thomas
  Sch{\"o}n.
\newblock Taming diffusion models for image restoration: a review.
\newblock \emph{Philosophical Transactions of the Royal Society A:
  Mathematical, Physical and Engineering Sciences}, 383\penalty0 (2299), 2025.

\bibitem[Ma et~al.(2025)Ma, Wu, Sun, and Li]{ma2025hpsv3}
Yuhang Ma, Xiaoshi Wu, Keqiang Sun, and Hongsheng Li.
\newblock Hpsv3: Towards wide-spectrum human preference score.
\newblock In \emph{Proceedings of the IEEE/CVF International Conference on
  Computer Vision}, pages 15086--15095, 2025.

\bibitem[Neyshabur et~al.(2018)Neyshabur, Bhojanapalli, and
  Srebro]{neyshabur2018pac}
Behnam Neyshabur, Srinadh Bhojanapalli, and Nathan Srebro.
\newblock A {PAC}-{B}ayesian approach to spectrally-normalized margin bounds
  for neural networks.
\newblock \emph{International Conference on Learning Representations (ICLR)},
  2018.

\bibitem[Oquab et~al.(2023)Oquab, Darcet, Moutakanni, Vo, Szafraniec, Khalidov,
  Fernandez, Haziza, Massa, El-Nouby, et~al.]{oquab2023dinov2}
Maxime Oquab, Timoth{\'e}e Darcet, Th{\'e}o Moutakanni, Huy Vo, Marc
  Szafraniec, Vasil Khalidov, Pierre Fernandez, Daniel Haziza, Francisco Massa,
  Alaaeldin El-Nouby, et~al.
\newblock Dinov2: Learning robust visual features without supervision.
\newblock \emph{arXiv preprint arXiv:2304.07193}, 2023.

\bibitem[Po et~al.(2024{\natexlab{a}})Po, Yang, Aberman, and
  Wetzstein]{po2024orthogonal}
Ryan Po, Guandao Yang, Kfir Aberman, and Gordon Wetzstein.
\newblock Orthogonal adaptation for modular customization of diffusion models.
\newblock In \emph{Proceedings of the IEEE/CVF Conference on Computer Vision
  and Pattern Recognition (CVPR)}, 2024{\natexlab{a}}.

\bibitem[Po et~al.(2024{\natexlab{b}})Po, Yifan, Golyanik, Aberman, Barron,
  Bermano, Chan, Dekel, Holynski, Kanazawa, et~al.]{po2024state}
Ryan Po, Wang Yifan, Vladislav Golyanik, Kfir Aberman, Jonathan~T Barron, Amit
  Bermano, Eric Chan, Tali Dekel, Aleksander Holynski, Angjoo Kanazawa, et~al.
\newblock State of the art on diffusion models for visual computing.
\newblock In \emph{Computer graphics forum}, volume~43, page e15063. Wiley
  Online Library, 2024{\natexlab{b}}.

\bibitem[Podell et~al.(2024)Podell, English, Lacey, Blattmann, Dockhorn,
  M{\"u}ller, Penna, and Rombach]{podell2024sdxl}
Dustin Podell, Zion English, Kyle Lacey, Andreas Blattmann, Tim Dockhorn, Jonas
  M{\"u}ller, Joe Penna, and Robin Rombach.
\newblock Sdxl: Improving latent diffusion models for high-resolution image
  synthesis.
\newblock In \emph{International Conference on Learning Representations},
  volume 2024, pages 1862--1874, 2024.

\bibitem[Radford et~al.(2021)Radford, Kim, Hallacy, Ramesh, Goh, Agarwal,
  Sastry, Askell, Mishkin, Clark, et~al.]{radford2021learning}
Alec Radford, Jong~Wook Kim, Chris Hallacy, Aditya Ramesh, Gabriel Goh,
  Sandhini Agarwal, Girish Sastry, Amanda Askell, Pamela Mishkin, Jack Clark,
  et~al.
\newblock Learning transferable visual models from natural language
  supervision.
\newblock In \emph{International conference on machine learning}, pages
  8748--8763. PmLR, 2021.

\bibitem[Ramesh et~al.(2022)Ramesh, Dhariwal, Nichol, Chu, and
  Chen]{ramesh2022hierarchical}
Aditya Ramesh, Prafulla Dhariwal, Alex Nichol, Casey Chu, and Mark Chen.
\newblock Hierarchical text-conditional image generation with {CLIP} latents.
\newblock \emph{arXiv preprint arXiv:2204.06125}, 2022.

\bibitem[Razaviyayn et~al.(2013)Razaviyayn, Hong, and Luo]{razaviyayn2013bsum}
Meisam Razaviyayn, Mingyi Hong, and Zhi-Quan Luo.
\newblock A unified convergence analysis of block successive minimization
  methods for nonsmooth optimization.
\newblock \emph{SIAM Journal on Optimization}, 23\penalty0 (2):\penalty0
  1126--1153, 2013.

\bibitem[Rombach et~al.(2022)Rombach, Blattmann, Lorenz, Esser, and
  Ommer]{rombach2022high}
Robin Rombach, Andreas Blattmann, Dominik Lorenz, Patrick Esser, and Bj{\"o}rn
  Ommer.
\newblock High-resolution image synthesis with latent diffusion models.
\newblock In \emph{Proceedings of the IEEE/CVF conference on computer vision
  and pattern recognition}, pages 10684--10695, 2022.

\bibitem[Ruiz et~al.(2023)Ruiz, Li, Jampani, Pritch, Rubinstein, and
  Aberman]{ruiz2023dreambooth}
Nataniel Ruiz, Yuanzhen Li, Varun Jampani, Yael Pritch, Michael Rubinstein, and
  Kfir Aberman.
\newblock Dreambooth: Fine tuning text-to-image diffusion models for
  subject-driven generation.
\newblock In \emph{Proceedings of the IEEE/CVF conference on computer vision
  and pattern recognition}, pages 22500--22510, 2023.

\bibitem[Saharia et~al.(2022)Saharia, Chan, Saxena, Li, Whang, Denton,
  Ghasemipour, Gontijo~Lopes, Karagol~Ayan, Salimans,
  et~al.]{saharia2022photorealistic}
Chitwan Saharia, William Chan, Saurabh Saxena, Lala Li, Jay Whang, Emily~L
  Denton, Kamyar Ghasemipour, Raphael Gontijo~Lopes, Burcu Karagol~Ayan, Tim
  Salimans, et~al.
\newblock Photorealistic text-to-image diffusion models with deep language
  understanding.
\newblock \emph{Advances in neural information processing systems},
  35:\penalty0 36479--36494, 2022.

\bibitem[Sim{\'e}oni et~al.(2025)Sim{\'e}oni, Vo, Seitzer, Baldassarre, Oquab,
  Jose, Khalidov, Szafraniec, Yi, Ramamonjisoa, et~al.]{simeoni2025dinov3}
Oriane Sim{\'e}oni, Huy~V Vo, Maximilian Seitzer, Federico Baldassarre, Maxime
  Oquab, Cijo Jose, Vasil Khalidov, Marc Szafraniec, Seungeun Yi, Micha{\"e}l
  Ramamonjisoa, et~al.
\newblock Dinov3.
\newblock \emph{arXiv preprint arXiv:2508.10104}, 2025.

\bibitem[Simsar et~al.(2025)Simsar, Hofmann, Tombari, and
  Yanardag]{simsar2025loraclr}
Enis Simsar, Thomas Hofmann, Federico Tombari, and Pinar Yanardag.
\newblock Loraclr: Contrastive adaptation for customization of diffusion
  models.
\newblock In \emph{Proceedings of the Computer Vision and Pattern Recognition
  Conference}, pages 13189--13198, 2025.

\bibitem[Vershynin(2026)]{Vershynin_2026}
Roman Vershynin.
\newblock \emph{High-Dimensional Probability: An Introduction with Applications
  in Data Science}.
\newblock Cambridge Series in Statistical and Probabilistic Mathematics.
  Cambridge University Press, Cambridge, 2 edition, 2026.

\bibitem[Wu et~al.(2023)Wu, Hao, Sun, Chen, Zhu, Zhao, and Li]{wu2023human}
Xiaoshi Wu, Yiming Hao, Keqiang Sun, Yixiong Chen, Feng Zhu, Rui Zhao, and
  Hongsheng Li.
\newblock Human preference score v2: A solid benchmark for evaluating human
  preferences of text-to-image synthesis.
\newblock \emph{arXiv preprint arXiv:2306.09341}, 2023.

\end{thebibliography}
\bibliographystyle{plainnat}
\newpage
\tableofcontents

\newpage
\appendix
\begin{algorithm}[t]
\caption{Sequential Regularized LoRA (SeqLoRA)}
\label{alg:seqlora}
\begin{algorithmic}[1]
\Require Pretrained weights $\{\Wb_0^{(\ell)}\}_{\ell=1}^L$; concept datasets $\Dc_1, \dots, \Dc_T$; step sizes $\alpha, \beta$; bilevel iterations $K$; local iterations $S_B, S_A'$; regularization $\epsilon$
\Ensure LoRA factors $\{(\Ab_i^{(\ell)}, \Bb_i^{(\ell)})\}_{i,\ell}$
\For{concept $i = 1$ to $T$}
    \State Initialize $\{\Ab_i^{(\ell,0)}, \Bb_i^{(\ell,0)}\}_{\ell=1}^L$ randomly
    \State $\Bb_{\mathrm{int}}^{(\ell)} \leftarrow [\Bb_1^{(\ell)}, \dots, \Bb_{i-1}^{(\ell)}]$ \textbf{ for all } $\ell \in \{1,\dots,L\}$ \Comment{Previous bases at each layer}
    \For{$k = 0$ to $K-1$} \Comment{Bilevel iterations}
        \State
        \State \textcolor{gray}{\textit{// Step 1: Tentative lower-level update of $\{\Ab_i^{(\ell)}\}_{\ell=1}^L$}}
        \State $\widetilde{\Ab}^{(\ell)} \leftarrow \Ab_i^{(\ell,k)} - \alpha\,\nabla_{\Ab^{(\ell)}}\Lc_i\!\left(\{\Ab_i^{(\ell,k)}, \Bb_i^{(\ell,k)}\}_{\ell=1}^L\right)$ \textbf{for all} $\ell$
        \State
        \State \textcolor{gray}{\textit{// Steps 2--3: Reduced gradient + projected $\{\Bb_i^{(\ell)}\}_{\ell=1}^L$ update}}
        \State $\Bb^{(\ell)} \leftarrow \Bb_i^{(\ell,k)}$ \textbf{for all} $\ell$
        \For{$s = 1$ to $S_B$} \Comment{Local $B$-iterations}
            \State $\gb_\Bb^{(\ell)} \leftarrow \nabla_{\Bb_i^{(\ell)}}\Lc_i(\{\widetilde{\Ab}^{(\ell)}, \Bb^{(\ell)}\}_{\ell=1}^L)$ \textbf{for all} $\ell$ \Comment{Direct partial}
            \State $\gb_\Ab^{(\ell)} \leftarrow \nabla_{\Ab_i^{(\ell)}}\Lc_i(\{\widetilde{\Ab}^{(\ell)}, \Bb^{(\ell)}\}_{\ell=1}^L)$ \textbf{for all} $\ell$ \Comment{Inner residual}
            \State \textcolor{gray}{\textit{// Cross-Hessian via autodiff at current iterate}}
\State $\Hc_{\Ab\Bb}^{(k,\ell)}[\gb_\Ab^{(\ell)}] \leftarrow
\nabla_{\Bb_i^{(\ell)}} \Tr\!\left((\gb_\Ab^{(\ell)})^\top
\nabla_{\Ab_i^{(\ell)}}\Lc_i(\Ab_i^{(\ell,k)}, \Bb_i^{(\ell)})\right)
\Big|_{\Bb_i^{(\ell)} = \Bb_i^{(\ell,k)}}$\textbf{for all} $\ell$
            \State $\nabla_{\Bb^{(\ell)}}\Phi_i \leftarrow \gb_\Bb^{(\ell)} - \alpha\,\Hc^{(k,\ell)}_{\Ab\Bb}[\gb_\Ab^{(\ell)}]$ \textbf{for all} $\ell$ \Comment{Reduced gradient}
            \State $\widetilde{\Bb}^{(\ell)} \leftarrow \Bb^{(\ell)} - \beta\,\nabla_{\Bb^{(\ell)}}\Phi_i$ \textbf{for all} $\ell$
            \State $\Bb^{(\ell)} \leftarrow \left[\Ib - \Bb_{\mathrm{int}}^{(\ell)}\!\left((\Bb_{\mathrm{int}}^{(\ell)})^\top \Bb_{\mathrm{int}}^{(\ell)} + \epsilon\Ib\right)^{-1}\!(\Bb_{\mathrm{int}}^{(\ell)})^\top\right]\widetilde{\Bb}^{(\ell)}$ \textbf{for all} $\ell$ \Comment{Regularized projection}
        \EndFor\\
        \State $\Bb_i^{(\ell,k+1)} \leftarrow \Bb^{(\ell)}$ \textbf{for all} $\ell$
        \State
        \State \textcolor{gray}{\textit{// Step 4: Final update of $\{\Ab_i^{(\ell)}\}_{\ell=1}^L$ using $\{\Bb_i^{(\ell,k+1)}\}_{\ell=1}^L$}}
        \State $\Ab^{(\ell)} \leftarrow \Ab_i^{(\ell,k)}$ \textbf{for all} $\ell$
        \For{$s = 1$ to $S_A'$} \Comment{Local $A$-iterations at new $\{\Bb^{(\ell)}\}$}
            \State $\Ab^{(\ell)} \leftarrow \Ab^{(\ell)} - \alpha\,\nabla_{\Ab^{(\ell)}}\Lc_i\!\left(\{\Ab^{(\ell)}, \Bb_i^{(\ell,k+1)}\}_{\ell=1}^L\right)$ \textbf{for all} $\ell$
        \EndFor\\
        \State $\Ab_i^{(\ell,k+1)} \leftarrow \Ab^{(\ell)}$ \textbf{for all} $\ell$\\
    \EndFor
    \State Store $(\Ab_i^{(\ell)}, \Bb_i^{(\ell)}) \leftarrow (\Ab_i^{(\ell,K)}, \Bb_i^{(\ell,K)})$ \textbf{for all} $\ell$
\EndFor
\end{algorithmic}
\end{algorithm}

\section{Mathematical Preliminaries}\label{sec:mathfound}
This section gathers the notation, definitions, and assumptions used in the paper.
\subsection{Notations}
We use bold lowercase letters for vectors and bold uppercase letters for matrices. For a matrix $\Ab$, $\Ab^\top$, $\Tr(\Ab)$, and $\operatorname{vec}(\Ab)$ denote its transpose, trace, and vectorization, respectively. The Frobenius norm and the operator (spectral) norm are denoted by $\nf{\Ab}$ and $\nt{\Ab}$. The Kronecker product is denoted by $\Ab \otimes \Bb$. For a matrix $\Bb$, $\mathrm{col}(\Bb)$ denotes its column space. The sub-Gaussian norm of a random variable $Z$ is denoted by $\|Z\|_{\psi_2}$. We use superscripts in parentheses (e.g., $\Wb^{(\ell)}$) to index network layers; subscripts denote either the LoRA factors associated with concept~$i$
(e.g., $\Ab_i$, $\Bb_i$) or the full parameter state after learning concept~$i$
(e.g., $\Wb_i$), with the meaning clear from context.

\paragraph{Lower-level iterates.} In the bilevel formulation the lower-level solution depends on the upper-level variable, and we write $\Ab_i^{(k+1)}(\Bb) = \Ab_i^{(k)} - \alpha\nabla_{\Ab}\Lc_i(\Ab_i^{(k)},\Bb)$ to make this dependence explicit (Eq.~\eqref{eq:inner}). Two instances recur: the \emph{tentative} step at the current basis, $\widetilde{\Ab}_i^{(k)} \triangleq \Ab_i^{(k+1)}(\Bb_i^{(k)})$, and the \emph{accepted} iterate at the projected basis, $\Ab_i^{(k+1)} \triangleq \Ab_i^{(k+1)}(\Bb_i^{(k+1)})$. With this convention the reduced objective of Eq.~\eqref{eq:reduced} satisfies $\Phi_i(\Bb_i^{(k)}) = \Lc_i(\widetilde{\Ab}_i^{(k)},\Bb_i^{(k)})$ and $\Phi_i(\Bb_i^{(k+1)}) = \Lc_i(\Ab_i^{(k+1)},\Bb_i^{(k+1)})$, identities that are used in Appendix~\ref{sec:proofthm}. When concept and iteration indices are irrelevant we drop them and write $\Ab^{(k+1)}(\Bb)$ with base point $\Ab_0 = \Ab_i^{(k)}$.
  
\subsection{Definitions}
\begin{definition}\citep{articlesubgradient}
The subdifferential of a PLSC function $g$ at $\xb \in \Rbb^n$ is defined as
\[
\partial g(\xb) \stackrel{\triangle}{=} \left\{ \zetab \in \Rbb^n \mid \exists \xb_k \to \xb, \, g(\xb_k) \to g(\xb), \, \zetab_k \to \zetab \in \partial \hat{g}(\xb_k) \right\}
\]
where $\partial \hat{g}(\xb_k)$ is the Fréchet subdifferential of $g$ at $\xb_k \in \Rbb^n$, defined as
\begin{align}
\partial \hat{g}(\xb_k) = \left\{ \zeta \in \Rbb^n \mid \liminf_{\vb \neq \xb, \vb \to \xb} \frac{1}{\nt{\vb-\xb}^2} \left[g(\vb) - g(\xb) - \langle \vb - \xb, \zeta \rangle \right] \geq 0 \right\}.
\end{align}
\end{definition}

\begin{definition}\citep{articlesubgradient}
    A point $\xb^*$ is called a critical point of a PLSC function $f(\xb)$ if $0 \in \partial f(\xb^*)$.
\end{definition}

\begin{lemma}[Descent lemma \citep{bolte:hal-00916090}]
Let $f: \Rbb^n \to \Rbb$ be a continuously differentiable function whose gradient $\nabla f$ is L-Lipschitz continuous. Then, for all $\xb, \yb \in \Rbb^n$:
\begin{equation}
  f(\yb) \leq f(\xb) + \langle \nabla f(\xb), \yb - \xb \rangle + \frac{L}{2} \nt{\yb-\xb}^2.
\end{equation}
\label{lem:lsmooth}
\end{lemma}

\begin{definition}[Operator norm of the cross-Hessian \citep{horn2012matrix}]
For a linear map $\Hc: \mathbb{R}^{n \times r} \to \mathbb{R}^{m \times r}$,
the operator norm is defined as:
\begin{equation}
    \nt{\Hc} = \sup_{\mathbf{V} \neq \mathbf{0}}
    \frac{\|\Hc[\mathbf{V}]\|_F}{\|\mathbf{V}\|_F}.
\end{equation}
\end{definition}
\subsection{Assumptions}
To establish the convergence of the SeqLoRA algorithm (Theorem \ref{thm:smoothcon}), we require the following standard regularity conditions on the optimization landscape:
\begin{assumption} \label{assump:smoothness}
The loss function $\Lc_i(\Ab, \Bb)$ satisfies the following properties:
\begin{itemize}
    \item $\Lc_i$ is jointly $L$-smooth with respect to $\Ab$ and $\Bb$ in the Frobenius norm. This inherently bounds the operator norm of the cross-Hessian: $\nt{\nabla^2_{\Ab\Bb}\Lc_i(\Ab, \Bb)} \le L$.
     \item The cross-Hessian is $\rho$-Lipschitz continuous with respect to $\Bb$:
\begin{equation}
    \nt{\nabla^2_{\Ab\Bb} \Lc_i(\Ab_0, \Bb_1) - \nabla^2_{\Ab\Bb} \Lc_i(\Ab_0, \Bb_2)} \le \rho \|\Bb_1 - \Bb_2\|_F.
\end{equation}
\item  \label{assump:bounded_grad}
The partial gradient of the loss with respect to $\Ab$ is bounded by a constant $M > 0$:
\begin{equation}
    \|\nabla_{\Ab} \Lc_i(\Ab, \Bb)\|_F \le M.
\end{equation}
\end{itemize}
\end{assumption}
\paragraph{Discussion of Assumption~\ref{assump:smoothness}.} These are standard regularity conditions in the analysis of non-convex and bilevel optimization \citep{fallah2020convergence}. Joint smoothness holds locally in any bounded region where training is stable. The bounded gradient assumption for the inner variable is a mild condition, further justified by the common deep learning practice of {gradient clipping}, which explicitly enforces such a bound to prevent training instability. The Lipschitz continuity of the cross-Hessian is a second-order smoothness condition necessary for analyzing algorithms that leverage parameter coupling. Using the operator norm is the rigorous choice, as the cross-Hessian is a linear operator (a 4th-order tensor) mapping one matrix space to another.

To bound the catastrophic forgetting across the continuous learning stream (Theorem \ref{thm:hanson_wright_catas}), we employ the following structural assumptions on the network mappings and the statistical distribution of the residual activations:
\begin{assumption}[Multi-Layer Structure and Sub-Gaussian Residuals]
\label{assum:multilayer_subgaussian}
The network applies LoRA adaptation at $L$ layers. For a concept $j$ at layer $\ell$, let $\Xb_j^{(\ell)} \in \Rbb^{m_\ell \times p}$ be the input activations and let $\Pc_{B_j}^{(\ell),\perp}$ denote the orthogonal projector onto the orthogonal complement of the column space of $\Bb_j^{(\ell)}$. We assume:
\begin{enumerate}
    \item \label{assum_item:lipschitz} \textbf{Lipschitz Network Mappings:} Between consecutive LoRA-adapted layers $\ell$ and $\ell+1$, the intermediate network operations compose a mapping $\phi_{\ell+1}: \Rbb^{n_\ell \times p} \to \Rbb^{m_{\ell+1} \times p}$ that is $\gamma_\ell$-Lipschitz continuous with respect to the Frobenius norm. Furthermore, the task loss applied to the final layer output is $L_o$-Lipschitz.

    \item \label{assum_item:subgaussian} \textbf{Sub-Gaussian Residuals:} At each LoRA-adapted layer $\ell$, the residual input activation $\Xb_j^{(\ell),\perp} = \Pc_{B_j}^{(\ell),\perp} \Xb_j^{(\ell)}$ follows a matrix sub-Gaussian process:
    \begin{equation}\label{eq:separablecov}
        \mathbf{X}_j^{(\ell),\perp} = (\Sigmab_j^{(\ell),\perp})^{1/2} \Zb^{(\ell)} (\Psib_j^{(\ell)})^{1/2},
    \end{equation}
    where $\Sigmab_j^{(\ell),\perp} \in \Rbb^{m_\ell \times m_\ell}$ captures the feature-dimension covariance, $\Psib_j^{(\ell)}\in \Rbb^{p \times p}$ captures the token-dimension correlation (sequence dependencies), and $\Zb^{(\ell)} \in \Rbb^{m_\ell \times p}$ consists of independent, mean-zero, unit-variance sub-Gaussian entries bounded by $\|Z_{a,b}^{(\ell)}\|_{\psi_2} \le K$.
\end{enumerate}
\end{assumption}
\paragraph{Discussion of Assumption~\ref{assum:multilayer_subgaussian}.}  These assumptions are based on the standard architecture of modern diffusion models. Item 1 (Lipschitz Mappings) is well-justified because deep networks are compositions of Lipschitz-continuous operations, a property required for stable training. Crucially, Item 2 (Sub-Gaussian Residuals) is an assumption applied only to the \textit{residual} activations. The learned LoRA basis $\Bb_j^{(\ell)}$ is optimized to capture the dominant, highly-structured components of the concept's features. The residual, therefore, represents the uncaptured ``noise floor,'' which can be modeled as a sub-Gaussian process. This is a mild assumption, as the sub-Gaussian class is broad and includes not only perfect Gaussians but also any bounded distribution (such as those produced by LayerNorm or Tanh) and mixtures thereof. The separable Kronecker covariance structure in~\eqref{eq:separablecov} is a standard modeling choice in high-dimensional matrix-variate statistics~\citep{dawid1981some,gupta2018matrix}, adopted here for analytical tractability.

\newpage

\section{Proof of Theorem \ref{thm:smoothcon}}\label{sec:proofthm}
We analyze the convergence of the SeqLoRA algorithm for the concept $i$.

\begin{proof}
To prove the theorem, we first establish the following lemma.
\begin{lemma}[Gradient Lipschitz Continuity of $\Phi_i$] \label{lemma:lip_phi}
Under Assumptions \ref{assump:smoothness}, the reduced objective function $\Phi_i(\Bb) \triangleq \Lc_i(\Ab^*(\Bb), \Bb)$ where $\Ab^*(\Bb) = \Ab_0 - \alpha \nabla_{\Ab} \Lc_i(\Ab_0, \Bb)$ and it has Lipschitz continuous gradients (i.e., is $L_\Phi$-smooth) with the Lipschitz constant:
\begin{equation}
    L_\Phi \leq L(1 + \alpha L)^2 + \alpha \rho M.
\end{equation}
That is, for any $\Bb_1, \Bb_2 \in \mathbb{R}^{m \times r}$, we have:
\begin{equation}
    \|\nabla \Phi_i(\Bb_1) - \nabla \Phi_i(\Bb_2)\|_F \le L_\Phi \|\Bb_1 - \Bb_2\|_F.
\end{equation}
\end{lemma}

\begin{proof}
Let $\Bb_1$ and $\Bb_2$ be two arbitrary matrices.

Using the definition of the lower-level update \eqref{eq:inner} for $\Ab_0 = \Ab_i^{(k)}$ and the $L$-smoothness from Assumption \ref{assump:smoothness}:
\begin{align}
    \|\Ab^*(\Bb_1) - \Ab^*(\Bb_2)\|_F &= \|(\Ab_0 - \alpha \nabla_{\Ab} \Lc_i(\Ab_0, \Bb_1)) - (\Ab_0 - \alpha \nabla_{\Ab} \Lc_i(\Ab_0, \Bb_2))\|_F \nonumber \\
    &= \alpha \|\nabla_{\Ab} \Lc_i(\Ab_0, \Bb_1) - \nabla_{\Ab} \Lc_i(\Ab_0, \Bb_2)\|_F
    \le \alpha L \|\Bb_1 - \Bb_2\|_F. \label{eq:A_lip}
\end{align}

Using \eqref{eq:grad-Phi-exact-vec}, the total derivative of $\Phi_i$ with respect to $\Bb$ is given by:
\begin{equation}\label{eq:sumtwoterm}
    \nabla \Phi_i(\Bb) = \underbrace{\nabla_{\Bb} \Lc_i(\Ab^*(\Bb), \Bb)}_{T_1} - \alpha \underbrace{\nabla^2_{\Ab\Bb} \Lc_i(\Ab_0, \Bb) \big[ \nabla_{\Ab} \Lc_i(\Ab^*(\Bb), \Bb) \big]}_{T_2}.
\end{equation}
To bound the $\nf{\nabla \Phi_i(\Bb_1) - \nabla \Phi_i(\Bb_2)}$, we use the triangle inequality. We first bound the term $T_1$.
 Using the Assumption \ref{assump:smoothness} and substituting the bound from \eqref{eq:A_lip}:
\begin{align}
    & \|\nabla_{\Bb} \Lc_i(\Ab^*(\Bb_1), \Bb_1) - \nabla_{\Bb} \Lc_i(\Ab^*(\Bb_2), \Bb_2)\|_F
    \le L \Big( \|\Ab^*(\Bb_1) - \Ab^*(\Bb_2)\|_F + \|\Bb_1 - \Bb_2\|_F \Big) \nonumber \\
    &\le L \Big( \alpha L \|\Bb_1 - \Bb_2\|_F + \|\Bb_1 - \Bb_2\|_F \Big) = L(1 + \alpha L) \|\Bb_1 - \Bb_2\|_F. \label{eq:bound_T1}
\end{align}

To bound the term $T_2$ in \eqref{eq:sumtwoterm}, let $\Hc_j = \nabla^2_{\Ab\Bb} \Lc_i(\Ab_0, \Bb_j)$ and $G_j = \nabla_{\Ab} \Lc_i(\Ab^*(\Bb_j), \Bb_j)$ for $j \in \{1, 2\}$. We need to bound $\|\Hc_1[G_1] - \Hc_2[G_2]\|_F$.
Adding and subtracting the cross term $\Hc_1[G_2]$, and applying the triangle inequality and sub-multiplicativity of operator norms:
\begin{align}
    &\nf{\Hc_1[G_1] - \Hc_1[G_2] + \Hc_1[G_2] - \Hc_2[G_2]} \leq \|\Hc_1[G_1 - G_2]\|_F + \|(\Hc_1 - \Hc_2)[G_2]\|_F \nonumber \\
    &\leq \nt{\Hc_1} \nf{G_1 - G_2} + \nt{\Hc_1 - \Hc_2} \nf{G_2}. \label{eq:T2_split}
\end{align}
We now bound each of the four components in the above inequality:
by Assumption \ref{assump:smoothness}, we have $\nt{\Hc_1} \le L$ and  $\nf{G_2} \le M$. Using the $\rho$-Lipschitz property in Assumption \ref{assump:smoothness} yields
\begin{equation*}
 \nt{\Hc_1 - \Hc_2} \le \rho \|\Bb_1 - \Bb_2\|_F
\end{equation*}
Now, we use joint smoothness and \eqref{eq:A_lip}, resulting in
    \begin{align}
        \nonumber&\nf{G_1 - G_2} = \|\nabla_{\Ab} \Lc_i(\Ab^*(\Bb_1), \Bb_1) - \nabla_{\Ab} \Lc_i(\Ab^*(\Bb_2), \Bb_2)\|_F  \leq\\& L \Big( \|\Ab^*(\Bb_1) - \Ab^*(\Bb_2)\|_F + \|\Bb_1 - \Bb_2\|_F \Big)
        \le L(1 + \alpha L) \|\Bb_1 - \Bb_2\|_F. \label{eq:bound_G}
    \end{align}
Substituting these four bounds back into \eqref{eq:T2_split} yields:
\begin{align}
    &\nonumber\nf{\Hc_1[G_1] - \Hc_1[G_2] + \Hc_1[G_2] - \Hc_2[G_2]} \le L \Big[ L(1 + \alpha L) \|\Bb_1 - \Bb_2\|_F \Big] + \Big[ \rho \|\Bb_1 - \Bb_2\|_F \Big] M \\&
    = \Big( L^2(1 + \alpha L) + \rho M \Big) \nf{\Bb_1 - \Bb_2}. \label{eq:bound_T2}
\end{align}

Combining the bounds  \eqref{eq:bound_T1} and \eqref{eq:bound_T2} gives us:
\begin{align}
    &\|\nabla \Phi_i(\Bb_1) - \nabla \Phi_i(\Bb_2)\|_F \le \Big[ L(1 + \alpha L) + \alpha \big( L^2(1 + \alpha L) + \rho M \big) \Big] \|\Bb_1 - \Bb_2\|_F \nonumber \\
    &= \Big[ L(1 + \alpha L) + \alpha L^2(1 + \alpha L) + \alpha \rho M \Big] \|\Bb_1 - \Bb_2\|_F = \Big[ L(1 + \alpha L)(1 + \alpha L) + \alpha \rho M \Big] \|\Bb_1 - \Bb_2\|_F \nonumber \\
    &= \Big( L(1 + \alpha L)^2 + \alpha \rho M \Big) \|\Bb_1 - \Bb_2\|_F.
\end{align}
Thus, $\nabla \Phi_i(\Bb)$ is Lipschitz continuous with constant $L_\Phi \leq L(1 + \alpha L)^2 + \alpha \rho M$. This completes the proof of Lemma~\ref{lemma:lip_phi}.
\end{proof}
Next, we show the rest of the results of Theorem \ref{thm:smoothcon}.

The tentative lower-level update is $\tilde{\Ab}_i^{(k)} = \Ab_i^{(k)} - \alpha \nabla_{\Ab} \Lc_i(\Ab_i^{(k)}, \Bb_i^{(k)})$.
Applying the Descent Lemma on the $L$-smooth loss $\Lc_i$ with respect to $\Ab_i$:
\begin{equation} \label{eq:A_descent}
    \Lc_i(\tilde{\Ab}_i^{(k)}, \Bb_i^{(k)}) \leq \Lc_i(\Ab_i^{(k)}, \Bb_i^{(k)}) +\left\langle\nabla_{\Ab_i} \Lc_i(\Ab_i^{(k)},\Bb_i^{(k)}),\tilde{\Ab}_i^{(k)} - \Ab_i^{(k)}\right\rangle+ \frac{L}{2} \|\tilde{\Ab}_i^{(k)} - \Ab_i^{(k)}\|_F^2.
\end{equation}
Since $\widetilde{\Ab}_i^{(k)}$ is the minimizer of \eqref{eq:inner} when $\Bb_i = \Bb_i^{(k)}$, we have
\begin{equation}\label{eq:quad-A1}
    \Lc_i(\Ab_i^{(k)}, \Bb_i^{(k)}) + \bigl\langle \nabla_{\Ab_i}\Lc_i(\Ab_i^{(k)}, \Bb_i^{(k)}),\; \tilde{\Ab}_i^{(k)} - \Ab_i^{(k)} \bigr\rangle + \frac{1}{2\alpha}\bigl\|\tilde{\Ab}_i^{(k)} - \Ab_i^{(k)}\bigr\|_F^2\leq \Lc_i(\Ab_i^{(k)}, \Bb_i^{(k)}),
\end{equation}
Adding \eqref{eq:A_descent} into the above inequality gives us
\begin{equation}\label{eq:firsparteq}
    \Lc_i(\tilde{\Ab}_i^{(k)}, \Bb_i^{(k)}) \leq \Lc_i(\Ab_i^{(k)}, \Bb_i^{(k)}) + (\frac{L}{2}-\frac{1}{2\alpha}) \|\tilde{\Ab}_i^{(k)} - \Ab_i^{(k)}\|_F^2.
\end{equation}

Algorithm \ref{alg:seqlora} projects the gradient step of $\Phi_i$ onto $\Bc_i$. This is exactly the minimizer of the proximal quadratic model in which $\delta_{\Bc_i}(\Bb)$ is an indicator function
\begin{equation}\label{eq:projbk1}
    \Bb_i^{(k+1)} = \arg\min_{\Bb} \left\{ \langle \nabla_{\Bb_i} \Phi_i(\Bb_i^{(k)}), \Bb - \Bb_i^{(k)} \rangle + \frac{1}{2\beta} \|\Bb - \Bb_i^{(k)}\|_F^2 + \delta_{\Bc_i}(\Bb) \right\}.
\end{equation}
Since $\Bb_i^{(k+1)}$ minimizes this model compared to the feasible point $\Bb_i^{(k)}$, we have:
\begin{equation} \label{eq:inner_prod_ind}
    \langle \nabla_{\Bb} \Phi_i(\Bb_i^{(k)}; \Ab_i^{(k)}), \Bb_i^{(k+1)} - \Bb_i^{(k)} \rangle \le \delta_{\Bc_i}(\Bb_i^{(k)}) - \delta_{\Bc_i}(\Bb_i^{(k+1)}) - \frac{1}{2\beta} \|\Bb_i^{(k+1)} - \Bb_i^{(k)}\|_F^2.
\end{equation}

Applying the descent lemma to $\Phi_i$ (which is $L_\Phi$-smooth by Lemma \ref{lemma:lip_phi}) between iterates $k$ and $k+1$ results in
\begin{equation}
    \Phi_i(\Bb_i^{(k+1)})\leq \Phi_i(\Bb_i^{(k)})+\langle\nabla_{\Bb_i}\Phi_i(\Bb_i^{(k)}),\Bb_i^{(k+1)}-\Bb_i^{(k)}\rangle+\frac{L_{\Phi}}{2}\nf{\Bb_i^{(k+1)}-\Bb_i^{(k)}}^2
\end{equation}
Adding the above inequality to \eqref{eq:inner_prod_ind} yields:
\begin{equation} \label{eq:Phi_descent}
    \Phi_i(\Bb_i^{(k+1)}) + \delta_{\Bc_i}(\Bb_i^{(k+1)}) \le \Phi_i(\Bb_i^{(k)}) + \delta_{\Bc_i}(\Bb_i^{(k)}) - \left( \frac{1}{2\beta} - \frac{L_\Phi}{2} \right) \|\Bb_i^{(k+1)} - \Bb_i^{(k)}\|_F^2.
\end{equation}

By definition \eqref{eq:reduced}, $\Phi_i(\Bb_i^{(k+1)}) = \Lc_i(\Ab_i^{(k+1)}, \Bb_i^{(k+1)})$, and $\Phi_i(\Bb_i^{(k)}) = \Lc_i(\tilde{\Ab}_i^{(k)}, \Bb_i^{(k)})$.
Substituting these and adding \eqref{eq:firsparteq} into \eqref{eq:Phi_descent} yields:
\begin{align} \label{eq:master_descent}
    F_i(\Ab_i^{(k+1)}, \Bb_i^{(k+1)}) \le F_i(\Ab_i^{(k)}, \Bb_i^{(k)}) &- \left( \frac{1}{2\alpha} - \frac{L}{2} \right) \|\tilde{\Ab}_i^{(k)} - \Ab_i^{(k)}\|_F^2 \nonumber \\
    &- \left( \frac{1}{2\beta} - \frac{L_\Phi}{2} \right) \|\Bb_i^{(k+1)} - \Bb_i^{(k)}\|_F^2.
\end{align}
where $F_i(\Ab_i, \Bb_i) = \Lc_i(\Ab_i, \Bb_i)+\delta_{\Bc_i}(\Bb_i)$.
Because $\alpha < \frac{1}{L}$ and $\beta < \frac{1}{L_\Phi}$, both coefficients are non-negative, resulting in a decreasing loss function. Substituting the tentative lower-level update $\tilde{\Ab}_i^{(k)} = \Ab_i^{(k)} - \alpha \nabla_{\Ab} \Lc_i(\Ab_i^{(k)}, \Bb_i^{(k)})$, using \eqref{eq:B-tilde} and \eqref{eq:update-B} upper-level update $\Bb_i^{(k+1)} = \Pc_{B_i}^{\perp}(\Bb_i^{(k)}-\beta\nabla_{\Bb_i}\Phi_i(\Bb_i^{(k)}))$ in the above inequality gives us
\begin{align} \label{eq:master_descent1}
    F_i(\Ab_i^{(k+1)}, \Bb_i^{(k+1)}) \le F_i(\Ab_i^{(k)}, \Bb_i^{(k)}) &- \left( \frac{\alpha}{2} - \frac{\alpha^2 L}{2} \right) \nf{\nabla_{\Ab_i} \Lc_i(\Ab_i^{(k)}, \Bb_i^{(k)})}^2 \nonumber \\
    &- \left( \frac{\beta}{2} - \frac{L_\Phi\beta^2}{2} \right) \nf{\Pc_{B_i}^{\perp}\nabla_{\Bb_i}\Phi_i(\Bb_i^{(k)})}^2.
\end{align}
Maximizing $\left( \frac{\alpha}{2} - \frac{\alpha^2 L}{2} \right)$ and $\left( \frac{\beta}{2} - \frac{L_\Phi\beta^2}{2} \right)$ over $\alpha\in(0, \frac{1}{L})$ and $\beta\in(0, \frac{1}{L_{\Phi}})$, respectively, gives $\alpha^* = \frac{1}{2L}$ and $\beta^* = \frac{1}{2L_{\Phi}} \geq \frac{1}{2L(1 + \alpha^* L)^2 + 2\alpha^* \rho M} = \frac{1}{\frac{9}{2}L+\frac{\rho M}{L}}$.

Summing \eqref{eq:master_descent} over $k=0$ to $\infty$ reveals that the sum of the squared differences is bounded by the initial suboptimality $F_i^{(0)} - F_i^{\infty}$. Therefore, we obtain the limits:
\begin{equation} \label{eq:limits}
    \lim_{k \to \infty} \|\tilde{\Ab}_i^{(k)} - \Ab_i^{(k)}\|_F = 0 \quad \text{and} \quad \lim_{k \to \infty} \|\Bb_i^{(k+1)} - \Bb_i^{(k)}\|_F = 0.
\end{equation}

We now prove that the convergence of the algorithm implies the stationary conditions.
First, for $\Ab_i$, recall that $\Ab_i^{(k)} - \tilde{\Ab}_i^{(k)} = \alpha \nabla_{\Ab} \Lc_i(\Ab_i^{(k)}, \Bb_i^{(k)})$. From \eqref{eq:limits}, it trivially follows that:
\begin{equation}
    \lim_{k \to \infty} \|\nabla_{\Ab} \Lc_i(\Ab_i^{(k)}, \Bb_i^{(k)})\|_F = 0,
\end{equation}
satisfying the critical point condition.

Second, for $\Bb_i$, using \eqref{eq:projbk1} the projected gradient mapping of the reduced objective vanishes, meaning:
\begin{equation} \label{eq:proj_phi}
    -\nabla_{\Bb} \Phi_i(\Bb_i^{(k)}; \Ab_i^{(k)}) \in \partial \delta_{\Bc_i}(\Bb^{(k)}_i) \quad \text{as } k \to \infty.
\end{equation}
where $\partial \delta_{\Bc_i}(\Bb^{(k)}_i)$ is the sub-gradient of the indicator $\delta_{\Bc_i}(\Bb)$ at $\Bb^{(k)}_i$.
We want to show that $\nabla_{\Bb} \Phi_i$ converges to the partial gradient of the true loss $\nabla_{\Bb} \Lc_i$. Using the chain rule from Lemma~\ref{lemma:lip_phi}:
\begin{equation}
    \nabla_{\Bb} \Phi_i(\Bb_i^{(k)}) = \nabla_{\Bb} \Lc_i(\tilde{\Ab}_i^{(k)}, \Bb_i^{(k)}) - \alpha \nabla^2_{\Ab\Bb} \Lc_i(\Ab_i^{(k)}, \Bb_i^{(k)}) \big[ \nabla_{\Ab} \Lc_i(\tilde{\Ab}_i^{(k)}, \Bb_i^{(k)}) \big].
\end{equation}
We analyze the discrepancy $\mathbf{D}^{(k)} = \nabla_{\Bb} \Phi_i(\Bb_i^{(k)}; \Ab_i^{(k)}) - \nabla_{\Bb} \Lc_i(\Ab_i^{(k)}, \Bb_i^{(k)})$:
\begin{align}
    \|\mathbf{D}^{(k)}\|_F &\le \|\nabla_{\Bb} \Lc_i(\tilde{\Ab}_i^{(k)}, \Bb_i^{(k)}) - \nabla_{\Bb} \Lc_i(\Ab_i^{(k)}, \Bb_i^{(k)})\|_F \nonumber \\
    &\quad + \alpha \nt{\nabla^2_{\Ab\Bb} \Lc_i} \|\nabla_{\Ab} \Lc_i(\tilde{\Ab}_i^{(k)}, \Bb_i^{(k)})\|_F.
\end{align}
By the joint $L$-smoothness of $\Lc_i$, the first term is bounded by $L \|\tilde{\Ab}_i^{(k)} - \Ab_i^{(k)}\|_F$. The cross-Hessian norm is bounded by $L$. Furthermore, the gradient evaluates to:
\begin{equation}
    \|\nabla_{\Ab} \Lc_i(\tilde{\Ab}_i^{(k)}, \Bb_i^{(k)})\|_F \le \|\nabla_{\Ab} \Lc_i(\Ab_i^{(k)}, \Bb_i^{(k)})\|_F + L \|\tilde{\Ab}_i^{(k)} - \Ab_i^{(k)}\|_F.
\end{equation}
Since $\alpha \nabla_{\Ab} \Lc_i(\Ab_i^{(k)}, \Bb_i^{(k)}) = \Ab_i^{(k)} - \tilde{\Ab}_i^{(k)}$, every component of the discrepancy is bounded by a constant multiple of $\|\tilde{\Ab}_i^{(k)} - \Ab_i^{(k)}\|_F$.
Because $\|\tilde{\Ab}_i^{(k)} - \Ab_i^{(k)}\|_F \to 0$ as established in \eqref{eq:limits}, we conclude that:
\begin{equation}
    \lim_{k \to \infty} \|\mathbf{D}^{(k)}\|_F = 0 \implies \lim_{k \to \infty} \nabla_{\Bb} \Phi_i(\Bb_i^{(k)}; \Ab_i^{(k)}) = \lim_{k \to \infty} \nabla_{\Bb} \Lc_i(\Ab_i^{(k)}, \Bb_i^{(k)}).
\end{equation}
Therefore, substituting this equivalence into \eqref{eq:proj_phi}, the algorithm perfectly satisfies the original composite stationarity condition:
\begin{equation}
    -\nabla_{\Bb} \Lc_i(\Ab_i^{(k)}, \Bb_i^{(k)}) \in \partial \delta_{\Bc_i}(\Bb^{(k)}_i) \quad \text{as } k \to \infty,
\end{equation}
Concluding the proof that the sequence converges to a joint critical point of the main composite loss function $F_i$.
\end{proof}

\section{Proof of Theorem \ref{thm:hanson_wright_catas}}\label{sec:prooftheorem2}
The proof of Theorem~\ref{thm:hanson_wright_catas} has two parts. \emph{Part 1 (Subsection~\ref{subsec:proof})} establishes the high-probability bound~\eqref{eq:e2e_boundfin}: Lemma~\ref{lem:hanson_wright} bounds the per-layer crosstalk $\nf{\Deltab_j^{(\ell)}}$ via Hanson--Wright, after Proposition~\ref{prop:crosstalk_residual} reduces it to the residual subspace using orthogonality; Lemma~\ref{lem:e2e_forgetting} propagates these local bounds through the $L$ Lipschitz layers via a linear recursion; a union bound over $\ell \in [L]$ yields the network-wide statement. \emph{Part 2 (Subsection~\ref{sec:optbas})} establishes basis optimality: Theorem~\ref{thm:optimal_feasible_basis} shows the top-$r$ eigenspace of the compressed covariance $\widetilde{\Sigmab}_j^{(\ell)}$ minimizes residual energy among feasible rank-$r$ projectors, and Proposition~\ref{prop:random_frozen_basis} shows random frozen bases are strictly suboptimal except in the isotropic case.
\subsection{Analysis of Catastrophic Forgetting}\label{subsec:proof}
To establish the end-to-end bound on catastrophic forgetting, we systematically decouple the local statistical interference from its global deterministic propagation through the deep network. We structure the proof into two complementary lemmas. First, we leverage the Hanson-Wright inequality to provide a high-probability statistical bound on the isolated activation crosstalk at any single layer. Second, the network-wide error propagation is analyzed, unrolling the Lipschitz-continuous inter-layer mappings to bound the total catastrophic forgetting as a weighted sum of these local perturbations. Finally, we synthesize these two results using a union bound across all $L$ layers to complete the theorem.
\begin{lemma}[Hanson-Wright High-Probability Bound on Crosstalk] \label{lem:hanson_wright}
Under Assumption \ref{assum:multilayer_subgaussian} item \ref{assum_item:subgaussian}, let $\Cb_j^{(\ell)} = \sum\limits_{k=j+1}^T\Ab_k^{(\ell)}(\Bb_k^{(\ell)})^\top$ be the crosstalk operator for concept $j$ and let the crosstalk be $\Deltab_j^{(\ell)} = \Cb_j^{(\ell)}\Xb_j^{(\ell)}$.
There exists an absolute universal constant $c > 0$ such that for any $\xi \in (0, 1)$, the crosstalk with probability at least $1-\xi$ is bounded by
\begin{equation}\label{eq:crosstalkupp}
   \nf{\Deltab_j^{(\ell)}} \le \sqrt{\Tr(\Psib_j^{(\ell)})~\nf{\Qb_j^{(\ell)}}^2 + t^{(\ell)}},
\end{equation}
where
\begin{equation}\label{eq:tell}
     t^{(\ell)} = \max \left( \sqrt{C_1} K^2 \nf{\Psib_j^{(\ell)}}\nf{(\Qb_j^{(\ell)})^\top\Qb_j^{(\ell)}} \sqrt{\log \frac{2}{\xi}}, \,\, C_1 K^2 \nt{\Psib_j^{(\ell)}}\nt{\Qb_j^{(\ell)}}^2 \log \frac{2}{\xi} \right),
\end{equation}
 $C_1 = 1/c$, and $\Qb_j^{(\ell)} \triangleq \Cb_j^{(\ell)} (\Sigmab_j^{(\ell),\perp})^{1/2} \in \Rbb^{n_\ell \times m_\ell}$.
\end{lemma}

\begin{proof}
Consider a model with $L$ layers adapted for $T$ concepts via SeqLoRA. At layer $\ell$, the pre-trained weight is $\Wb_0^{(\ell)} \in \Rbb^{n_\ell \times m_\ell}$ and the LoRA factors for concept $k$ are $\Ab_k^{(\ell)} \in \Rbb^{n_\ell \times r}$ and $\Bb_k^{(\ell)} \in \Rbb^{m_\ell \times r}$.

In the sequential continual learning setting, the accumulated weight matrix after learning concept $j$ at layer $\ell$ is:
\begin{equation}
    \Wb_j^{(\ell)} = \Wb_0^{(\ell)} + \sum_{k=1}^j \Ab_k^{(\ell)} (\Bb_k^{(\ell)})^\top.
\end{equation}
The composed output for concept $j$ is:
\begin{equation}\label{eq:single_concept}
    O_j(\mathbf{X}_j^{(\ell)}) = \Wb_j^{(\ell)} \mathbf{X}_j^{(\ell)}
\end{equation}
while the composed output (with all $T$ concepts) is:
\begin{equation}\label{eq:composed_output}
    \tilde{O}_T(\Xb_j^{(\ell)}) = \Wb_T^{(\ell)}\Xb_j^{(\ell)} = O_j(\Xb_j^{(\ell)}) + \sum_{k=j+1}^T \Ab_k^{(\ell)}(\Bb_k^{(\ell)})^\top \Xb_j^{(\ell)}.
\end{equation}
The total crosstalk at layer $\ell$ for concept $j$ is:
\begin{equation}\label{eq:crosstalk_def}
    \Deltab_j^{(\ell)} \triangleq \tilde{O}_T(\Xb_j^{(\ell)}) - O_j(\Xb_j^{(\ell)}) = \Cb_j^{(\ell)}\Xb_j^{(\ell)} \in \Rbb^{n_\ell \times p},
\end{equation}
where the crosstalk operator is:
\begin{equation}\label{eq:Cj_def}
    \Cb_j^{(\ell)} \triangleq \sum_{k=j+1}^T \Ab_k^{(\ell)}(\Bb_k^{(\ell)})^\top \;\in\; \Rbb^{n_\ell \times m_\ell}.
\end{equation}
This captures the interference of the parameters added after concept $j$ acting on concept $j$'s activations.
We decompose the activation into its component within concept $j$'s basis and its residual:
\begin{equation}\label{eq:decompose_X}
    \Xb_j^{(\ell)} = \Pc_{B_j}^{(\ell)}\Xb_j^{(\ell)} + \Pc_{B_j}^{(\ell),\perp} \Xb_j^{(\ell)}.
\end{equation}
where $\Pc_{B_j}^{(\ell)} = \Bb_j^{(\ell)}\left((\Bb_j^{(\ell)})^\top\Bb_j^{(\ell)}\right)^{-1}(\Bb_j^{(\ell)})^\top$ is the projector onto $\mathrm{col}(\Bb_j^{(\ell)})$ and its complement is $\Pc_{B_j}^{(\ell),\perp} = \Ib_{m^{(\ell)}}-\Pc_{B_j}^{(\ell)}$.
Since $\Bb_j^{(\ell)}$ is the learned LoRA basis for concept $j$, it captures the relevant features of $\Xb_j^{(\ell)}$. The first term carries most of the activation energy; the second term is the leaked energy outside the learned subspace.

\begin{proposition}\label{prop:crosstalk_residual}
    The crosstalk depends only on the residual component:
    \begin{equation}\label{eq:crosstalk_residual}
        \Deltab_j^{(\ell)} = \Cb_j^{(\ell)}\Xb_j^{(\ell)} = \Cb_j^{(\ell)}\Pc_{B_j}^{(\ell),\perp} \Xb_j^{(\ell)} = \Cb_j^{(\ell)} \Xb_j^{(\ell),\perp}.
    \end{equation}
    That is, the signal component $\Pc_{B_j}^{(\ell)}\Xb_j^{(\ell)}$ is invisible to the crosstalk.
\end{proposition}

\begin{proof}
    Substituting the decomposition~\eqref{eq:decompose_X}:
    \begin{equation}
        \Cb_j^{(\ell)}\Xb_j^{(\ell)} = \Cb_j^{(\ell)}\Pc_{B_j}^{(\ell)}\Xb_j^{(\ell)} + \Cb_j^{(\ell)}\Pc_{B_j}^{(\ell),\perp} \Xb_j^{(\ell)}.
    \end{equation}
    We show the first term vanishes. Expanding:
    \begin{equation}
        \Cb_j^{(\ell)}\Pc_{B_j}^{(\ell)} = \sum_{k=j+1}^T\Ab_k^{(\ell)}(\Bb_k^{(\ell)})^\top \Pc_{B_j}^{(\ell)}\,.
    \end{equation}
    Now $(\Bb_k^{(\ell)})^\top\Pc_{B_j}^{(\ell)} = (\Bb_k^{(\ell)})^\top \Bb_j^{(\ell)}(\Gb_j^{(\ell)})^{-1}(\Bb_j^{(\ell)})^\top$, where $\Gb_k^{(\ell)} = (\Bb_k^{(\ell)})^\top \Bb_k^{(\ell)}$. Since $(\Bb_k^{(\ell)})^\top \Bb_j^{(\ell)} = 0$ for $k \neq j$:
    \begin{equation}\label{eq:Cj_Pj_zero}
        \Cb_j^{(\ell)}\Pc_{B_j}^{(\ell)} = 0.
    \end{equation}
    Therefore $\Deltab_j^{(\ell)} = \Cb_j^{(\ell)}\Pc_{B_j}^{(\ell),\perp} \Xb_j^{(\ell)} = \Cb_j^{(\ell)} \Xb_j^{(\ell),\perp}$.
\end{proof}

This result implies that the entire activation energy within $\text{col}(\Bb_j^{(\ell)})$, regardless of its magnitude, does not contribute to the crosstalk, and the crosstalk sees only the residual

The residual activation covariance $\xb_j^{(\ell),\perp} = \text{vec}\left(\Xb_j^{(\ell),\perp}\right)$ is defined as follows
\begin{equation}\label{eq:Sigma_j}
    \text{Cov}(\xb_j^{(\ell),\perp}) \triangleq \EX\left\{\xb_j^{(\ell),\perp}(\xb_j^{(\ell),\perp})^\top\right\} = \Psib_j^{(\ell)}\otimes\Sigmab_j^{(\ell),\perp} \in \Rbb^{m_\ell p}.
\end{equation}
It is important to note that since $\Bb_j^{(\ell)}$ captures most of the activation energy of concept $j$, the residual covariance $\Sigmab_j^{(\ell),\perp}$  may have a small energy (small eigenvalues).

Substituting the matrix sub-Gaussian process into the crosstalk \eqref{eq:crosstalk_residual}, we obtain:
\begin{equation}
    \Deltab_j^{(\ell)} = \Cb_j^{(\ell)} \mathbf{X}_j^{(\ell),\perp} = \Cb_j^{(\ell)} (\Sigmab_j^{(\ell),\perp})^{1/2} \Zb^{(\ell)} (\Psib_j^{(\ell)})^{1/2} = \Qb_j^{(\ell)} \Zb^{(\ell)} (\mathbf{\Psi}_j)^{1/2},
\end{equation}
where $\Qb_j^{(\ell)} \triangleq \Cb_j^{(\ell)} (\Sigmab_j^{(\ell),\perp})^{1/2} \in \Rbb^{n_\ell \times m_\ell}$. The squared Frobenius norm of the crosstalk is a quadratic form:
\begin{align}
    z \triangleq \|\Deltab_j^{(\ell)}\|_F^2 &= \Tr\left( (\Qb_j^{(\ell)} \Zb^{(\ell)} (\Psib_j^{(\ell)})^{1/2})^\top (\Qb_j^{(\ell)} \Zb^{(\ell)} (\Psib_j^{(\ell)})^{1/2}) \right) \nonumber \\
    &= \Tr\left( ({\Zb^{(\ell)}})^\top \left((\Qb_j^{(\ell)})^\top \Qb_j^{(\ell)}\right) \Zb^{(\ell)} \Psib_j^{(\ell)} \right).
\end{align}
By using the standard matrix trace identity $\Tr(\Ab^\top \Bb \Ab \Cb) = \operatorname{vec}(\Ab)^\top (\Cb \otimes \Bb) \operatorname{vec}(\Ab)$, we vectorize the independent sub-Gaussian matrix $\Zb^{(\ell)}$ into $\zb^{(\ell)} \in \Rbb^{m_\ell p}$. The trace perfectly rewrites as a standard vector quadratic form:
\begin{equation}
    z = ({\zb^{(\ell)}})^\top \Hb \zb^{(\ell)},
\end{equation}
where $\Hb = \Psib_j^{(\ell)} \otimes \left((\Qb_j^{(\ell)})^\top \Qb_j^{(\ell)}\right)$.
Taking the expectation of the variable $z$ and using $\Tr(\Ab\otimes\Bb) = \Tr(\Ab)\Tr(\Bb)$ \citep{horn1991topics}
\begin{align}\label{eq:muzmain}
    \nonumber&\mu_z = \EX\{{(\zb^{(\ell)})}^\top \Hb \zb^{(\ell)}\} = \Tr(\EX\{\zb^{(\ell)}({\zb^{(\ell)}})^\top\}\Hb) = \Tr(\Hb) =\\& \Tr(\Psib_j^{(\ell)}) \Tr\left({\left((\Qb_j)^{(\ell)}\right)}^\top \Qb_j^{(\ell)}\right) = \Tr(\Psib_j^{(\ell)}) \nf{\Qb_j^{(\ell)}}^2,
\end{align}
where, based on the unit variance sub-Gaussian assumption, we have $\EX\{\zb^{(\ell)}({\zb^{(\ell)}})^\top\} = \Ib$.
By the Hanson-Wright inequality for sub-Gaussian random vectors \citep[Theorem 6.2.2]{Vershynin_2026}, for any $t > 0$:
\begin{equation} \label{eq:hw_ineq}
    \mathbb{P}(|z - \mu_z| > t) \le 2 \exp\left[ -c \min\left( \frac{t^2}{K^4 \nf{\Hb}^2}, \frac{t}{K^2 \nt{\Hb}} \right) \right],
\end{equation}
where $c > 0$ is a universal constant. We now bound the norms of $\Hb$.

For the $\ell_2$ norm and using $\nt{\Ab \otimes \Bb} = \nt{\Ab} \nt{\Bb}$ \citep{horn1991topics}, we have:
\begin{equation}\label{eq:l2norm}
    \nt{\Hb} = \nt{\Psib_j^{(\ell)} \otimes \left((\Qb_j^{(\ell)})^\top \Qb_j^{(\ell)}\right)} = \nt{\Psib_j^{(\ell)}} \nt{(\Qb_j^{(\ell)})^\top \Qb_j^{(\ell)}} = \nt{\Psib_j^{(\ell)}} \nt{\Qb_j^{(\ell)}}^2.
\end{equation}
To compute the squared Frobenius norm $\nf{\Hb}^2$, we use the property $\nf{\Ab\otimes\Bb} = \nf{\Ab}\nf{\Bb}$ \citep{horn1991topics}.
Consequently,
\begin{equation}\label{eq:nfnorm}
    \nf{\Hb} = \nf{\Psib_j^{(\ell)} \otimes \left((\Qb_j^{(\ell)})^\top \Qb_j^{(\ell)}\right)}= \nf{\Psib_j^{(\ell)}} \nf{(\Qb_j^{(\ell)})^\top \Qb_j^{(\ell)}}.
\end{equation}
To derive the exact minimal $t$ that guarantees a failure probability of at most $\xi$, we set the right-hand side equal to $\xi$:
\begin{equation}
    \min\left( \frac{t^2}{K^4 \nf{\Hb}^2}, \frac{t}{K^2 \nt{\Hb}} \right) = \frac{1}{c} \log \frac{2}{\xi}.
\end{equation}
For the minimum of two non-negative terms to equal a value $x$, both terms must be at least $x$. This yields two lower bounds for $t$:
\begin{align}
    \frac{t^2}{K^4 \nf{\Hb}^2} \ge \frac{1}{c} \log \frac{2}{\xi} \quad &\implies \quad t \ge \frac{K^2 \|\Hb\|_F}{\sqrt{c}} \sqrt{\log \frac{2}{\xi}}, \\
    \frac{t}{K^2 \nt{\Hb}} \ge \frac{1}{c} \log \frac{2}{\xi} \quad &\implies \quad t \ge \frac{K^2 \nt{\Hb}}{c} \log \frac{2}{\xi}.
\end{align}
To tightly satisfy both conditions, $t$ must be exactly the maximum of the two bounds. Using \eqref{eq:l2norm} and \eqref{eq:nfnorm} in the above inequalities, we have
\begin{equation}
    t^{(\ell)} = \max \left( \sqrt{C_1} K^2 \nf{\Psib_j^{(\ell)}}\nf{(\Qb_j^{(\ell)})^\top\Qb_j^{(\ell)}} \sqrt{\log \frac{2}{\xi}}, \,\, C_1 K^2 \nt{\Psib_j^{(\ell)}}\nt{\Qb_j^{(\ell)}}^2 \log \frac{2}{\xi} \right),
\end{equation}
where $C_1 = 1/c$. Using \eqref{eq:muzmain} and the above equation in the Hanson-Wright inequality results in the following crosstalk upper bound with probability $1-\xi$
\begin{equation}
\label{eq:finalbound}
    \nf{\Deltab_j^{(\ell)}} \le \sqrt{\Tr(\Psib_j^{(\ell)})~\nf{\Qb_j^{(\ell)}}^2 + t^{(\ell)}}.
\end{equation}
This concludes the proof of Lemma \ref{lem:hanson_wright}.
\end{proof}
\begin{lemma}[End-to-End Catastrophic Forgetting]
\label{lem:e2e_forgetting}
Under Assumptions~\ref{assum:multilayer_subgaussian} item \ref{assum_item:lipschitz}, let $\Gamma_\ell \triangleq \prod_{\ell'=\ell+1}^{L} \gamma_{\ell'} \nt{\Wb_T^{(\ell')}}$ be the downstream amplification factor at layer $\ell$. For any $\xi \in (0,1)$, the end-to-end catastrophic forgetting satisfies, with probability at least $1-\xi$:
\begin{equation}\label{eq:e2e_bound}
    \Lc_j(\theta_T) - \Lc_j(\theta_j) \le L_o \sum_{\ell=1}^{L} \Gamma_\ell \nf{\Deltab_j^{(\ell)}},
\end{equation}
where $\theta_T=\{\Wb_T^{(\ell)}\}_{\ell=1}^L$ and $\theta_j=\{\Wb_j^{(\ell)}\}_{\ell=1}^L$ are the full model after learning concept $j$ and $T$, respectively.
\end{lemma}

\begin{proof}
Let $O_j(\Xb_j^{(\ell)}) = \Wb_j^{(\ell)}\Xb_j^{(\ell)}$ and $\tilde{O}_T(\tilde{\Xb}_j^{(\ell)}) = \Wb_T^{(\ell)}\tilde{\Xb}_j^{(\ell)}$ denote the layer-$\ell$ inputs under model's parameters $\Wb_j^{(\ell)}$ and the composed model's parameters $T$ concepts ($\Wb_T^{(\ell)}$), respectively.
Defining the layer-wise perturbation as $e_\ell \triangleq \nf{O_j(\Xb_j^{(\ell)})  - \tilde{O}_T(\tilde{\Xb}_j^{(\ell)})}$, we have $e_0 = 0$ at the network input.

For $\ell \ge 1$, the inputs to the LoRA layer are $\Xb_j^{(\ell+1)} = \phi_\ell(O_j(\Xb_j^{(\ell)}))$ and $\hat{\Xb}_j^{(\ell+1)} = \phi_\ell(\tilde{O}_T(\tilde{\Xb}_j^{(\ell)}))$. The output difference expands as:
\begin{align}
    \tilde{O}_T(\tilde{\Xb}_j^{(\ell)}) - O_j(\Xb_j^{(\ell)}) = \Wb_T^{(\ell)}\tilde{\Xb}_j^{(\ell)} - \Wb_j^{(\ell)}\Xb_j^{(\ell)} = \Wb_T^{(\ell)}\big(\tilde{\Xb}_j^{(\ell)} - \Xb_j^{(\ell)}\big) + \Cb_j^{(\ell)}{\Xb}_j^{(\ell)}. \label{eq:output_perturb}
\end{align}
Taking the Frobenius norm, applying the triangle inequality, and utilizing the $\gamma_\ell$-Lipschitz property of the inter-layer mapping $\phi_\ell$ yields the linear recursion:
\begin{equation}
    e_\ell \le \nt{\Wb_T^{(\ell)}} \nf{\tilde{\Xb}_j^{(\ell)} - \Xb_j^{(\ell)}} + \nf{\Cb_j^{(\ell)}{\Xb}_j^{(\ell)}} \le \alpha_\ell e_{\ell-1} + \nf{{\Deltab}_j^{(\ell)}}, \label{eq:e_recursion_triangle}
\end{equation}
where we define $\alpha_\ell \triangleq \gamma_\ell \nt{\Wb_T^{(\ell)}}$ as the effective per-layer amplification and $\Deltab_j^{(\ell)} \triangleq \Cb_j^{(\ell)}{\Xb}_j^{(\ell)}$ as the injected crosstalk.

Unrolling the linear recursion \eqref{eq:e_recursion_triangle} entirely from $\ell = 1$ to $L$, and recognizing that the cumulative amplification from layer $\ell$ to $L$ is exactly $\Gamma_\ell = \prod_{\ell'=\ell+1}^{L} \alpha_{\ell'}$, we have the following bound
\begin{equation} \label{eq:l1_unroll}
    e_L \le \sum_{\ell=1}^L \Gamma_\ell \nf{\Deltab_j^{(\ell)}}.
\end{equation}
Finally, by the $L_o$-Lipschitz continuity of the downstream mapping $f_j$, the end-to-end catastrophic forgetting is bounded directly by the final-layer perturbation:
\begin{equation}
    \Lc_j(\theta_T) - \Lc_j(\theta_j) = f_j(\tilde{O}_T(\hat{\Xb}_j^{(L)})) - f_j(O_j(\hat{\Xb}_j^{(L)})) \le L_o \nf{\tilde{O}_T(\hat{\Xb}_j^{(L)}) - O_j(\hat{\Xb}_j^{(L)})} = L_o e_L.
\end{equation}
Using \eqref{eq:l1_unroll} in the above inequality results in
\begin{equation}
    \Lc_j(\theta_T) - \Lc_j(\theta_j) \leq L_o \sum_{\ell=1}^L \Gamma_\ell \nf{\Deltab_j^{(\ell)}}.
\end{equation}
\end{proof}
We now synthesize the results of Lemma \ref{lem:hanson_wright} and Lemma \ref{lem:e2e_forgetting} to complete the proof of the bound \eqref{eq:e2e_boundfin} in Theorem \ref{thm:hanson_wright_catas}. To ensure a global probabilistic guarantee across the entire network, we apply the statistical bound from Lemma \ref{lem:hanson_wright} with an adjusted layer-wise failure probability of $\xi/L$. Substituting this high-probability upper bound \eqref{eq:crosstalkupp} for the crosstalk norm $\nf{\Deltab_j^{(\ell)}}$ into the end-to-end deterministic bound \eqref{eq:e2e_bound}, we apply Boole's inequality (the union bound) over all $\ell \in [L]$. This yields to the inequality \eqref{eq:e2e_boundfin}. This concludes the proof of catastrophic forgetting bound \eqref{eq:e2e_boundfin} in Theorem \ref{thm:hanson_wright_catas}.

In Theorem \ref{thm:hanson_wright_catas}, due to the residual model $\Xb_j^{(\ell),\perp} = (\Sigmab_j^{(\ell),\perp})^{1/2} \Zb^{(\ell)} (\Psib_j^{(\ell)})^{1/2}$ in Assumption \ref{assum_item:subgaussian}, the trace of the covariance matrix of residual is given by
\begin{align*}
\Tr\Big(\EX\{ (\mathbf{X}_j^{(\ell),\perp})^\top \mathbf{X}_j^{(\ell),\perp} \}\Big) &= \Tr\Big(\EX\{ (\boldsymbol{\Psi}_j^{(\ell)})^{1/2} (\mathbf{Z}^{(\ell)})^\top \boldsymbol{\Sigma}_j^{(\ell),\perp} \mathbf{Z}^{(\ell)} (\boldsymbol{\Psi}_j^{(\ell)})^{1/2}  \}\Big) \\
&= \Tr\left( \boldsymbol{\Psi}_j^{(\ell)} \EX\{ (\mathbf{Z}^{(\ell)})^\top \boldsymbol{\Sigma}_j^{(\ell),\perp} \mathbf{Z}^{(\ell)} \} \right) \\
&= \Tr\left( \boldsymbol{\Psi}_j^{(\ell)} \cdot \Tr(\boldsymbol{\Sigma}_j^{(\ell),\perp}) \Ib_p \right) \\
&= \Tr(\boldsymbol{\Sigma}_j^{(\ell),\perp}) \Tr(\boldsymbol{\Psi}_j^{(\ell)}).
\end{align*}
Also, we know $\Xb_j^{(\ell),\perp} = \Pc_{B_j}^{(\ell),\perp}\Xb_j^{(\ell)}$, $\Sigmab_j^{(\ell)} = \EX\{\Xb_j^{(\ell)}(\Xb_j^{(\ell)})^\top\}$ and using the above equation results in
\begin{equation*}
    \Tr(\EX\{\Xb_j^{(\ell),\perp}(\Xb_j^{(\ell),\perp})^\top\}) = \Tr(\Pc_{B_j}^{(\ell),\perp}\Sigmab_j^{(\ell)}) = \Tr(\Sigmab_j^{(\ell),\perp}) \Tr(\Psib_j^{(\ell)}).
\end{equation*}
The term $\Tr(\Sigmab_j^{(\ell),\perp}) \Tr(\Psib_j^{(\ell)})$ governs the expected crosstalk energy within the end-to-end forgetting bound because
\begin{equation}\label{eq:e2e_boundfin1}
    \Lc_j(\theta_T) - \Lc_j(\theta_j) \le L_o \sum_{\ell=1}^{L} \Gamma_\ell \sqrt{\Tr(\Psib_j^{(\ell)})~\nf{\Qb_j^{(\ell)}}^2 + t^{(\ell)}},
\end{equation}
 identifying the residual covariance mass
 as the primary driver of catastrophic forgetting.  {To prove the rate \eqref{eq:e2e_boundfin}, we write $\epsilon_\ell = \Tr(\Psib_j^{(\ell)}) \Tr(\Sigmab_j^{(\ell),\perp})$, and $\kappa_\ell = \nt{\Cb_j^{(\ell)}}$. First,
\begin{equation}
    \nf{\Qb_j^{(\ell)}}^2 = \Tr\big(\Cb_j^{(\ell)}\Sigmab_j^{(\ell),\perp}(\Cb_j^{(\ell)})^\top\big) \le \nt{\Cb_j^{(\ell)}}^2\,\Tr(\Sigmab_j^{(\ell),\perp}) = \kappa_\ell^2\,\Tr(\Sigmab_j^{(\ell),\perp}),
\end{equation}
using $\Tr(\Cb\Mb\Cb^\top)\le\nt{\Cb}^2\Tr(\Mb)$ for positive semidefinite $\Mb$. Hence the expected crosstalk energy satisfies $\Tr(\Psib_j^{(\ell)})\nf{\Qb_j^{(\ell)}}^2 \le \kappa_\ell^2\epsilon_\ell$. Second, for positive semidefinite matrices, the Frobenius and operator norms are bounded by the trace, so $\nf{\Psib_j^{(\ell)}}\le \Tr(\Psib_j^{(\ell)})$, $\nt{\Psib_j^{(\ell)}}\le\Tr(\Psib_j^{(\ell)})$, $\nf{(\Qb_j^{(\ell)})^\top\Qb_j^{(\ell)}}\le\Tr((\Qb_j^{(\ell)})^\top\Qb_j^{(\ell)}) = \nf{\Qb}^2\le\kappa_\ell^2\,\Tr(\Sigmab_j^{(\ell),\perp})$ and $\nt{\Qb_j^{(\ell)}}^2\le\nf{\Qb_j^{(\ell)}}^2\le\kappa_\ell^2\,\Tr(\Sigmab_j^{(\ell),\perp})$. Substituting into \eqref{eq:tell},
\begin{equation}
    t^{(\ell)} \le \kappa_\ell^2\epsilon_\ell\cdot\max\Big(\sqrt{C_1}K^2\sqrt{\log\tfrac{2L}{\xi}},\;C_1K^2\log\tfrac{2L}{\xi}\Big) = \kappa_\ell^2\epsilon_\ell\, c_{K,\xi}.
\end{equation}
Therefore $\sqrt{\Tr(\Psib)\nf{\Qb}^2 + t^{(\ell)}} \le \kappa_\ell\sqrt{\epsilon_\ell}\sqrt{1+c_{K,\xi}}$, and summing over layers gives \eqref{eq:e2e_boundfin}.}
 
 This leads to a fundamental optimization question: among all feasible orthogonal rank-$r$ subspaces, which one minimizes this residual energy?
In the next subsection, we address this question, which proves the second part of Theorem \ref{thm:hanson_wright_catas} regarding minimizing the energy of interference.

\subsection{Optimal basis selection inside the free complement}\label{sec:optbas}
Let
\begin{equation}
    \Bb^{(\ell)}_{\mathrm{int}}
    \triangleq
    [\Bb_1^{(\ell)},\dots,\Bb_{j-1}^{(\ell)}],
    \qquad
    \Pc_{<j}^{(\ell)}
    \triangleq
    \Bb^{(\ell)}_{\mathrm{int}}
    \left((\Bb^{(\ell)}_{\mathrm{int}})^\top \Bb^{(\ell)}_{\mathrm{int}}\right)^{-1}
    (\Bb^{(\ell)}_{\mathrm{int}})^\top,
\end{equation}
and define the currently available orthogonal complement
\begin{equation}
    \Pc_{\mathrm{free},j}^{(\ell)}
    \triangleq
    \Ib-\Pc_{<j}^{(\ell)}.
\end{equation}
A rank-\(r\) projector \(\Pc^{(\ell)}\) is \emph{feasible} for concept \(j\) if
\begin{equation}
    \Pc^{(\ell)}=(\Pc^{(\ell)})^\top=(\Pc^{(\ell)})^2,
    \qquad
    \operatorname{rank}(\Pc^{(\ell)})=r,
    \qquad
    \operatorname{range}(\Pc^{(\ell)})\subseteq \operatorname{range}(\Pc_{\mathrm{free},j}^{(\ell)}).
\end{equation}
Equivalently, \(\Pc^{(\ell)}\) is any rank-\(r\) subspace that obeys the orthogonality constraints with all previously learned concepts. For any projector \(\Pc^{(\ell)}\), write \(\Pc^{(\ell),\perp} \triangleq \Ib-\Pc^{(\ell)}\). Define the compressed covariance inside the currently free subspace by
\begin{equation}
    \widetilde{\Sigmab}_j^{(\ell)}
    \triangleq
    \Pc_{\mathrm{free},j}^{(\ell)}\Sigmab_j^{(\ell)} \Pc_{\mathrm{free},j}^{(\ell)},
\end{equation}
and let \(d_{\mathrm{free},j} \triangleq \operatorname{rank}(\Pc_{\mathrm{free},j}^{(\ell)})\). Denote the eigenvalues of \(\widetilde{\Sigmab}_j^{(\ell)}\) by
\begin{equation}
    \lambda_1(\widetilde{\Sigmab}_j^{(\ell)})\ge \lambda_2(\widetilde{\Sigmab}_j^{(\ell)})\ge \cdots \ge \lambda_{d_{\mathrm{free},j}}(\widetilde{\Sigmab}_j^{(\ell)})\ge 0.
\end{equation}

\begin{theorem}[Optimal feasible basis captures maximal task energy]
\label{thm:optimal_feasible_basis}
Assume \(r \le d_{\mathrm{free},j}\). Then for any feasible rank-\(r\) projector \(\Pc^{(\ell)}\),
\begin{equation}
    \Tr(\Pc^{(\ell),\perp} \Sigmab_j^{(\ell)})
    =
    \Tr\!\left((\Ib-\Pc_{\mathrm{free},j}^{(\ell)})\Sigmab_j^{(\ell)}\right)
    +
    \Tr\!\left((\Pc_{\mathrm{free},j}^{(\ell)}-\Pc^{(\ell)})\widetilde{\Sigmab}_j^{(\ell)}\right).
\end{equation}
Hence minimizing the residual energy \(\Tr(\Pc^{(\ell),\perp} \Sigmab_j^{(\ell)})\) over all feasible rank-\(r\) projectors is equivalent to maximizing the captured energy \(\Tr(\Pc^{(\ell)} \widetilde{\Sigmab}_j^{(\ell)})\). The minimizer is the projector \({\Pc_j^{{(\ell)}}}^\star\) onto the top-\(r\) eigenspace of \(\widetilde{\Sigmab}_j^{(\ell)}\), and the minimum residual energy is
\begin{equation}
    \Tr\!\left(({\Pc_j^{{(\ell)}}}^\star)^\perp \Sigmab_j^{(\ell)}\right)
    =
    \Tr\!\left((\Ib-\Pc_{\mathrm{free},j}^{(\ell)})\Sigmab_j^{(\ell)}\right)
    +
    \sum_{q=r+1}^{d_{\mathrm{free},j}} \lambda_q(\widetilde{\Sigmab}_j^{(\ell)}).
\end{equation}
\end{theorem}

\begin{proof}
Because \(\operatorname{range}(\Pc^{(\ell)})\subseteq \operatorname{range}(\Pc_{\mathrm{free},j}^{(\ell)})\), we have
\begin{equation}
    \Pc_{\mathrm{free},j}^{(\ell)}\Pc^{(\ell)} = \Pc^{(\ell)}\Pc_{\mathrm{free},j}^{(\ell)} = \Pc^{(\ell)}.
\end{equation}
Therefore
\begin{equation}
    \Ib-\Pc^{(\ell)}
    =
    (\Ib-\Pc_{\mathrm{free},j}^{(\ell)}) + (\Pc_{\mathrm{free},j}^{(\ell)}-\Pc^{(\ell)}),
\end{equation}
and thus
\begin{align}
    \Tr(\Pc^{(\ell),\perp} \Sigmab_j^{(\ell)})
    &=
    \Tr\!\left((\Ib-\Pc_{\mathrm{free},j}^{(\ell)})\Sigmab_j^{(\ell)}\right)
    +
    \Tr\!\left((\Pc_{\mathrm{free},j}^{(\ell)}-\Pc)\Sigmab_j^{(\ell)}\right).
\end{align}
Since \(\Pc_{\mathrm{free},j}^{(\ell)}-\Pc\) acts entirely inside the free complement,
\begin{equation}
    \Tr\!\left((\Pc_{\mathrm{free},j}^{(\ell)}-\Pc)\Sigmab_j^{(\ell)}\right)
    =
    \Tr\!\left((\Pc_{\mathrm{free},j}^{(\ell)}-\Pc^{(\ell)})\Pc_{\mathrm{free},j}^{(\ell)}\Sigmab_j^{(\ell)}\Pc_{\mathrm{free},j}^{(\ell)}\right)
    =
    \Tr\!\left((\Pc_{\mathrm{free},j}^{(\ell)}-\Pc)\widetilde{\Sigmab}_j^{(\ell)}\right).
\end{equation}
This proves the decomposition.

Next, observe that
\begin{equation}
    \Tr\!\left((\Pc_{\mathrm{free},j}^{(\ell)}-\Pc^{(\ell)})\widetilde{\Sigmab}_j^{(\ell)}\right)
    =
    \Tr(\widetilde{\Sigmab}_j^{(\ell)}) - \Tr(\Pc^{(\ell)} \widetilde{\Sigmab}_j^{(\ell)}).
\end{equation}
Therefore minimizing \(\Tr(\Pc^{(\ell),\perp} \Sigmab_j^{(\ell)})\) is equivalent to maximizing \(\Tr(\Pc^{(\ell)} \widetilde{\Sigmab}_j^{(\ell)})\) over all feasible rank-\(r\) projectors. By the variational characterization of eigenvalues, the maximum is
\begin{equation}
    \max_{\Pc^{(\ell)}}\Tr(\Pc^{(\ell)} \widetilde{\Sigmab}_j^{(\ell)})
    =
    \sum_{q=1}^{r}\lambda_q(\widetilde{\Sigmab}_j^{(\ell)}),
\end{equation}
attained when \(\Pc^{(\ell)}={\Pc_j^{{(\ell)}}}^\star\) is the projector onto the top-\(r\) eigenspace of \(\widetilde{\Sigmab}_j^{(\ell)}\). Substituting back yields
\begin{align}
    \Tr\!\left(({\Pc_j^{{(\ell)}}}^\star)^\perp \Sigmab_j^{(\ell)}\right)
    &=
    \Tr\!\left((\Ib-\Pc_{\mathrm{free},j}^{(\ell)})\Sigmab_j^{(\ell)}\right)
    +
    \Tr(\widetilde{\Sigmab}_j^{(\ell)})
    -
    \sum_{q=1}^{r}\lambda_q(\widetilde{\Sigmab}_j^{(\ell)}) \\
    &=
    \Tr\!\left((\Ib-\Pc_{\mathrm{free},j}^{(\ell)})\Sigmab_j^{(\ell)}\right)
    +
    \sum_{q=r+1}^{d_{\mathrm{free},j}}\lambda_q(\widetilde{\Sigmab}_j^{(\ell)}).
\end{align}
\end{proof}

Theorem~\ref{thm:optimal_feasible_basis} identifies the best possible basis for continual-learning: among all subspaces compatible with the orthogonality constraints, the optimal one is the dominant task covariance subspace inside the currently available complement. This immediately explains why learning \(\Bb_j^{(\ell)}\) is beneficial: a frozen basis cannot adapt to the spectral structure of \(\widetilde{\Sigmab}_j^{(\ell)}\).

The next result makes this comparison explicit for random frozen bases, which are a natural abstraction of methods that choose an orthogonal basis a priori and then keep it fixed.

\begin{proposition}[Random frozen bases are optimal only in isotropic free complements]
\label{prop:random_frozen_basis}
Let \(\Pc_{\mathrm{rand}}\) be a Haar-random rank-\(r\) projector inside \(\operatorname{range}(\Pc_{\mathrm{free},j}^{(\ell)})\). Then
\begin{equation}
    \EX[\Pc_{\mathrm{rand}}]
    =
    \frac{r}{d_{\mathrm{free},j}}\Pc_{\mathrm{free},j}^{(\ell)},
    \qquad
    \EX\!\{\Tr(\Pc_{\mathrm{rand}}\widetilde{\Sigmab}_j^{(\ell)})\}
    =
    \frac{r}{d_{\mathrm{free},j}}\Tr(\widetilde{\Sigmab}_j^{(\ell)}),
\end{equation}
and therefore
\begin{equation}
    \EX\!\{\Tr(\Pc_{\mathrm{rand}}^\perp \Sigmab_j^{(\ell)})\}
    =
    \Tr\!\left((\Ib-\Pc_{\mathrm{free},j}^{(\ell)})\Sigmab_j^{(\ell)}\right)
    +
    \left(1-\frac{r}{d_{\mathrm{free},j}}\right)\Tr(\widetilde{\Sigmab}_j^{(\ell)}).
\end{equation}
Moreover,
\begin{equation}
    \EX\!\{\Tr(\Pc_{\mathrm{rand}}^\perp \Sigmab_j^{(\ell)})\}
    -
    \Tr\!\left(({\Pc_j^{{(\ell)}}}^\star)^\perp \Sigmab_j^{(\ell)}\right)
    =
    \sum_{q=1}^{r}\lambda_q(\widetilde{\Sigmab}_j^{(\ell)})
    -
    \frac{r}{d_{\mathrm{free},j}}
    \sum_{q=1}^{d_{\mathrm{free},j}}\lambda_q(\widetilde{\Sigmab}_j^{(\ell)})
    \ge 0.
\end{equation}
Thus, except in degenerate isotropic or capacity-saturated cases, a learned feasible basis strictly improves upon a frozen random feasible basis in expectation.
\end{proposition}

\begin{proof}
By rotational invariance of the Haar distribution on the \(d_{\mathrm{free},j}\)-dimensional free subspace, the expectation \(\EX[\Pc_{\mathrm{rand}}]\) must commute with every orthogonal transformation acting on \(\operatorname{range}(\Pc_{\mathrm{free},j}^{(\ell)})\). Hence it must be a scalar multiple of \(\Pc_{\mathrm{free},j}^{(\ell)}\), say
\begin{equation}
    \EX[\Pc_{\mathrm{rand}}] = \alpha \Pc_{\mathrm{free},j}^{(\ell)}.
\end{equation}
Taking traces gives
\begin{equation}
    r = \Tr(\EX[\Pc_{\mathrm{rand}}]) = \alpha\, d_{\mathrm{free},j},
\end{equation}
so \(\alpha = r/d_{\mathrm{free},j}\), proving the first identity.

Using \(\widetilde{\Sigmab}_j^{(\ell)}=\Pc_{\mathrm{free},j}^{(\ell)}\widetilde{\Sigmab}_j^{(\ell)}\),
\begin{equation}
    \EX\!\{\Tr(\Pc_{\mathrm{rand}}\widetilde{\Sigmab}_j^{(\ell)})\}
    =
    \Tr\!\left(\EX[\Pc_{\mathrm{rand}}]\widetilde{\Sigmab}_j^{(\ell)}\right)
    =
    \frac{r}{d_{\mathrm{free},j}}\Tr(\widetilde{\Sigmab}_j^{(\ell)}).
\end{equation}
The formula for \(\EX[\Tr(\Pc_{\mathrm{rand}}^\perp \Sigmab_j^{(\ell)})]\) then follows directly from Theorem~\ref{thm:optimal_feasible_basis}.

Finally, since the average of the top \(r\) eigenvalues is at least the average over all \(d_{\mathrm{free},j}\) eigenvalues,
\begin{equation}
    \frac{1}{r}\sum_{q=1}^{r}\lambda_q(\widetilde{\Sigmab}_j^{(\ell)})
    \ge
    \frac{1}{d_{\mathrm{free},j}}
    \sum_{q=1}^{d_{\mathrm{free},j}}\lambda_q(\widetilde{\Sigmab}_j^{(\ell)}),
\end{equation}
which implies
\begin{equation}
    \sum_{q=1}^{r}\lambda_q(\widetilde{\Sigmab}_j^{(\ell)})
    -
    \frac{r}{d_{\mathrm{free},j}}
    \sum_{q=1}^{d_{\mathrm{free},j}}\lambda_q(\widetilde{\Sigmab}_j^{(\ell)})
    \ge 0.
\end{equation}
This proves the claim.
\end{proof}
Combining the results in Section~\ref{subsec:proof} and Section~\ref{sec:optbas} proves Theorem \ref{thm:hanson_wright_catas}.
\begin{remark}[Relation to SeqLoRA]\label{rem:relseqlora}
Theorem~\ref{thm:optimal_feasible_basis} should not be interpreted as saying that SeqLoRA explicitly computes the top eigenspace of \(\widetilde{\Sigmab}_j^{(\ell)}\). Rather, it identifies the population quantity that matters for continual-learning stability under the orthogonality constraint. In the full nonlinear diffusion model, SeqLoRA seeks a favorable subspace implicitly by optimizing the end-to-end denoising loss over both \(\Ab_j^{(\ell)}\) and \(\Bb_j^{(\ell)}\), instead of freezing \(\Bb_j^{(\ell)}\) a priori; {the loss-optimal basis does not necessarily coincide with the covariance-optimal one, and the theory makes no claim about the critical points the algorithm reaches.}
\end{remark}

Theorem~\ref{thm:optimal_feasible_basis} provides a simple learning-theoretic justification for SeqLoRA. Orthogonality ensures that future concepts can only interfere with concept \(j\) through the residual covariance mass outside \(\operatorname{col}(\Bb_j)\), while learning \(\Bb_j^{(\ell)}\) from data minimizes exactly this exposed residual energy within the feasible complement. Proposition~\ref{prop:random_frozen_basis} further shows that frozen random bases are statistically optimal only when the task covariance is essentially isotropic inside the free complement. In this precise sense, optimizing both LoRA factors yields a strictly better continual-learning bias than fixing the basis in advance.

\section{Decomposition of Catastrophic Forgetting: Mean Interference vs. Stochastic Deviation}\label{subsec:interpre}
The end-to-end catastrophic forgetting bound derived in Theorem~\ref{thm:hanson_wright_catas} reveals a multi-scale structure. At each of the $L$ adapted layers, the locally injected crosstalk $\|\Deltab_j^{(\ell)}\|_F^2$ decomposes into two distinct sources: the expected systematic interference $\mu_z^{(\ell)}$ and the stochastic tail deviation $t^{(\ell)}$. These local statistical fluctuations are then propagated and aggregated through the downstream network layers, controlled by the amplification factors $\Gamma_\ell$, to produce the final global forgetting. We now analyze the physical meaning of each component and its role in the total network-wide interference.

\paragraph{Systematic Interference ($\mu_z^{(\ell)}$).}
The mean crosstalk $\mu_z = \Tr(\Psib_j^{(\ell)}) \nf{\Qb_j^{(\ell)}}^2$ factorizes into two physical quantities. The first, $\nf{\Qb_j^{(\ell)}}^2 = \Tr(\Cb_j^{(\ell)} \Sigmab_j^{(\ell),\perp})$, is the {feature-level interference energy}. Crucially, it depends exclusively on the residual covariance $\Sigmab_j^{(\ell),\perp}$, not the full covariance $\Sigmab_j^{(\ell)}$. This is a direct consequence of SeqLoRA's orthogonality constraint, which perfectly annihilates the primary signal component ($\Cb_j^{(\ell)} \Pc_{B_j}^{(\ell)} = \mathbf{0}$). The better the LoRA basis $\Bb_j^{(\ell)}$ captures concept $j$'s features, the smaller the residual $\Sigmab_j^{(\ell),\perp}$ becomes, driving feature interference toward small values. The second term, $\Tr(\Psib_j^{(\ell)})$, represents the total token sequence energy. Together, orthogonality guarantees that the expected forgetting is confined entirely to the small, uncaptured residual subspace. Additionally, because the activation $\Xb_j^{(\ell)}$ is itself the output of a network optimized for concept $j$, the LoRA bases across all layers work in concert to capture the concept's energy; any features not absorbed by $\Bb_j^{(\ell-1)}$ can be captured by $\Bb_j^{(\ell)}$, ensuring the per-layer residual covariance $\Sigmab_j^{(\ell),\perp}$ remains uniformly small throughout the network depth. {This is a heuristic description of the mechanism rather than a proven property.}

\paragraph{Stochastic Deviation ($t^{(\ell)}$).}
The deviation $t^{(\ell)}$ captures the random fluctuation of the crosstalk around its mean due to the realization of the sub-Gaussian noise. It is bounded by $t^{(\ell)} = \max(T_1, T_2)$, where:
\begin{align*}
    T_1 &= K^2 \sqrt{C_1} \nf{\Psib_j^{(\ell)}} \nf{(\Qb_j^{(\ell)})^\top \Qb_j^{(\ell)}} \sqrt{\log \frac{2}{\xi}} & &\text{(Sub-Gaussian Term)} \\
    T_2 &= K^2 C_1 \nt{\Psib_j^{(\ell)}} \nt{(\Qb_j^{(\ell)})^\top \Qb_j^{(\ell)}} \log \frac{2}{\xi} & &\text{(Sub-Exponential Term)}.
\end{align*}
Defining the effective rank as $\mathrm{r}_{\mathrm{eff}}(\Ab) \triangleq \nf{\Ab} / \nt{\Ab}$, we determine the dominant regime by setting $T_1 \ge T_2$. Solving this inequality yields the threshold condition for the sub-Gaussian term to dominate:
\begin{equation} \label{eq:eff_rank_condition}
    \mathrm{r}_{\mathrm{eff}}(\Psib_j^{(\ell)}) \cdot \mathrm{r}_{\mathrm{eff}}\big((\Qb_j^{(\ell)})^\top \Qb_j^{(\ell)}\big) \ge \sqrt{C_1 \log \frac{2}{\xi}}.
\end{equation}
The heavy-tailed sub-exponential regime only triggers if this strict inequality is reversed.

\paragraph{Regime Analysis: Diffuse vs. Spiky Interference.}
The threshold condition in Equation~\eqref{eq:eff_rank_condition} reveals that catastrophic forgetting behaves differently depending on the magnitude of the joint effective rank. This transition is governed by whether the residual geometry is diffuse or spiky:

\begin{itemize}
    \item \textbf{The Sub-Gaussian Regime (Diffuse Interference):} When the joint effective rank is large ($T_1 \ge T_2$), the sub-Gaussian term dominates. Because the learned LoRA basis $\Bb_j^{(\ell)}$ successfully absorbs the dominant principal components of the activations, the residual covariance $\Sigmab_j^{(\ell),\perp}$ retains only a flat, unstructured spectral tail. The interference is \emph{diffuse}, spread evenly across many dimensions, acting as benign noise that scales gracefully as $\sqrt{\log(2/\xi)}$.

    \item \textbf{The Sub-Exponential Regime (Spiky Interference):} Conversely, the heavier-tailed sub-exponential regime ($T_2 > T_1$) only triggers if the joint effective rank collapses ($\approx 1$). This occurs if the residual covariance still contains a massive, uncaptured principal direction, making the interference highly \textit{spiky} and concentrated. In this state, forgetting is dictated by worst-case fluctuations along that single dominant axis, scaling poorly as $\log(2/\xi)$.
\end{itemize}

\textbf{The Dual Role of Orthogonality.}
SeqLoRA's inter-concept orthogonality ($\Bb_i^\top \Bb_j = \mathbf{0}$) serves a dual mathematical purpose to protect the network. First, it shrinks the systematic mean $\mu_z$ by restricting interference exclusively to the residual subspace. Second, by projecting out the dominant principal components captured by
$\Bb_j^{(\ell)}$, it exposes the naturally diffuse structure of the residual
spectrum, ensuring a large joint effective rank that keeps the crosstalk
in the well-concentrated sub-Gaussian regime.

\section{Population Risk Analysis in a Stylized Linear Setting}
\label{sec:population_theory}
Theorem~\ref{thm:smoothcon} establishes the optimization dynamics of SeqLoRA,
and Theorem~\ref{thm:hanson_wright_catas} provides a high-probability bound
on end-to-end catastrophic forgetting in deep, non-linear networks. To
complement this result, we now derive an \emph{exact}, closed-form
decomposition of catastrophic forgetting in a stylized linear setting with
a single adapted layer. This is not intended to model the full diffusion
network; rather, by removing non-linearities, it isolates the precise
mechanism through which orthogonality and learned-basis adaptation interact,
and makes explicit the residual-energy quantity that controls forgetting.
The two analyses agree in their conclusion: the high-probability bound in
Theorem~\ref{thm:hanson_wright_catas} and the exact decomposition below
both identify $\mathrm{Tr}(\Pc^\perp_{B_j} \Sigmab_j)$ as the driver of
forgetting, providing convergent evidence from two distinct analytical
viewpoints.

\paragraph{Stylized population model.}
Consider a single linear map with fixed input features \(\xb \in \Rbb^m\) and output \(\yb \in \Rbb^n\). For concept \(j\), let \(\mathcal{D}_j\) be a population distribution over \((\xb,\yb)\), and define the squared population risk
\begin{equation}
    \Lc_j(\Wb)
    \triangleq
    \EX_{(\xb,\yb)\sim \mathcal{D}_j}\!\{\|\yb-\Wb \xb\|_2^2\},
    \qquad
    \Sigmab_j
    \triangleq
    \EX_{(\xb,\yb)\sim \mathcal{D}_j}\!\{\xb\xb^\top\}.
\end{equation}
As in LoRA, after learning \(t\) concepts the adapted matrix is
\begin{equation}
    \Wb_t = \Wb_0 + \sum_{s=1}^{t} \Ab_s \Bb_s^\top.
\end{equation}
Fix a concept \(j\). After learning concept \(j\), the model is
\begin{equation}
    \Wb_j = \Wb_0 + \sum_{s=1}^{j} \Ab_s \Bb_s^\top,
\end{equation}
while after all later concepts have been learned,
\begin{equation}
    \Wb_T = \Wb_j + \Cb_j,
    \qquad
    \Cb_j \triangleq \sum_{s=j+1}^{T} \Ab_s \Bb_s^\top.
\end{equation}
Let
\begin{equation}
    \Pc_{B_j}
    \triangleq
    \Bb_j(\Bb_j^\top \Bb_j)^{-1} \Bb_j^\top,
    \qquad
    \Pc_{B_j}^{\perp}
    \triangleq
    \Ib-\Pc_{B_j},
\end{equation}
denote the orthogonal projector onto the learned right-subspace of concept \(j\) and its orthogonal complement. In this simplified view, the role of the right LoRA factor \(\Bb_j\) is summarized by \(\Pc_{B_j}\), while the cumulative effect of all future concepts is summarized by the crosstalk operator \(\Cb_j\).

The first result shows that, under orthogonality, forgetting is driven only by the part of the concept covariance that lies outside the learned basis.

\begin{theorem}[Residual-subspace decomposition of forgetting]
\label{thm:linear_residual_forgetting}
Assume that the later bases satisfy
\begin{equation}
    \Bb_s^\top \Bb_j = \mathbf{0},
    \qquad \forall s>j.
\end{equation}
Then the population forgetting on concept \(j\) admits the exact decomposition
\begin{equation}
    \Lc_j(\Wb_T)-\Lc_j(\Wb_j)
    =
    \Tr\!\left(
        \Cb_j \Pc_{B_j}^{\perp}\Sigmab_j \Pc_{B_j}^{\perp} \Cb_j^\top
    \right)
    +
    \left\langle \nabla \Lc_j(\Wb_j),\, \Cb_j \right\rangle.
\end{equation}
Consequently, if \(\|\nabla \Lc_j(\Wb_j)\|_F \le \varepsilon_j\), then
\begin{equation}
    \Lc_j(\Wb_T)-\Lc_j(\Wb_j)
    \le
    \nt{\Cb_j}^2 \Tr\!\left(\Pc_{B_j}^{\perp}\Sigmab_j\right)
    +
    \varepsilon_j \|\Cb_j\|_F.
\end{equation}
In particular, if \(\Wb_j\) is a stationary point of \(\Lc_j\), then
\begin{equation}
    \Lc_j(\Wb_T)-\Lc_j(\Wb_j)
    =
    \Tr\!\left(
        \Cb_j \Pc_{B_j}^{\perp}\Sigmab_j \Pc_{B_j}^{\perp} \Cb_j^\top
    \right)
    \le
    \nt{\Cb_j}^2 \Tr\!\left(\Pc_{B_j}^{\perp}\Sigmab_j\right).
\end{equation}
\end{theorem}

\begin{proof}
Because \(\Bb_s^\top \Bb_j=\mathbf{0}\) for every \(s>j\), we also have
\begin{equation}
    \Bb_s^\top \Pc_{B_j} = \mathbf{0},
    \qquad \forall s>j,
\end{equation}
since \(\operatorname{range}(\Pc_{B_j}) = \operatorname{col}(\Bb_j)\). Therefore
\begin{equation}
    \Cb_j \Pc_{B_j}
    =
    \sum_{s=j+1}^{T}\Ab_s \Bb_s^\top \Pc_{B_j}
    =
    \mathbf{0},
\end{equation}
which implies \(\Cb_j = \Cb_j \Pc_{B_j}^{\perp}\).

Now expand the population risk difference:
\begin{align}
    \Lc_j(\Wb_T)-\Lc_j(\Wb_j)
    &=
    \EX\!\{
        \|\yb-(\Wb_j+\Cb_j)\xb\|_2^2 - \|\yb-\Wb_j\xb\|_2^2
    \} \\
    &=
    \EX\!\{\|\Cb_j \xb\|_2^2\}
    +
    2\,\EX\!\{(\Wb_j\xb-\yb)^\top \Cb_j \xb\}.
\end{align}
Since
\begin{equation}
    \nabla \Lc_j(\Wb)
    =
    2\,\EX\!\{(\Wb\xb-\yb)\xb^\top\},
\end{equation}
the second term is exactly \(\langle \nabla \Lc_j(\Wb_j), \Cb_j\rangle\). For the first term, using \(\Cb_j=\Cb_j\Pc_{B_j}^{\perp}\),
\begin{align}
    \EX\!\{\|\Cb_j \xb\|_2^2\}
    &=
    \EX\!\{\xb^\top \Cb_j^\top \Cb_j \xb\} \\
    &=
    \Tr\!\left(\Cb_j \Sigmab_j \Cb_j^\top\right) \\
    &=
    \Tr\!\left(
        \Cb_j \Pc_{B_j}^{\perp}\Sigmab_j \Pc_{B_j}^{\perp} \Cb_j^\top
    \right).
\end{align}
This proves the identity.

For the upper bound, note that for any positive semidefinite matrix \(\Mb\),
\begin{equation}
    \Tr(\Cb_j \Mb \Cb_j^\top)
    \le
    \nt{\Cb_j}^2 \Tr(\Mb).
\end{equation}
Applying this with \(\Mb=\Pc_{B_j}^{\perp}\Sigmab_j\Pc_{B_j}^{\perp}\) gives
\begin{equation}
    \Tr\!\left(
        \Cb_j \Pc_{B_j}^{\perp}\Sigmab_j \Pc_{B_j}^{\perp} \Cb_j^\top
    \right)
    \le
    \nt{\Cb_j}^2 \Tr\!\left(\Pc_{B_j}^{\perp}\Sigmab_j\right).
\end{equation}
Finally, by Cauchy--Schwarz,
\begin{equation}
    \left|
    \left\langle \nabla \Lc_j(\Wb_j),\, \Cb_j \right\rangle
    \right|
    \le
    \|\nabla \Lc_j(\Wb_j)\|_F \|\Cb_j\|_F
    \le
    \varepsilon_j \|\Cb_j\|_F.
\end{equation}
The stationary case follows by setting \(\varepsilon_j=0\).
\end{proof}
This result is perfectly consistent with Theorem~\ref{thm:hanson_wright_catas}, which derives the catastrophic forgetting by  $\Tr(\Pc_{B_j}^{(\ell),\perp}\Sigmab_j^{(\ell)}) = \Tr(\Sigmab_j^{(\ell),\perp}) \Tr(\Psib_j^{(\ell)})$. The alignment between our deep-network concentration bounds and this population-risk analysis underscores that minimizing this residual energy is the key to sequential stability.

\newpage
\section{Numerical Illustration of Theorem~\ref{thm:hanson_wright_catas} in a Small Nonlinear Network}\label{sec:theory_numerics}
  
Evaluating Theorem~2 directly on a diffusion model is challenging because the quantities it involves, such as residual covariances, crosstalk matrices, and products of layerwise Lipschitz constants, are not readily available in that setting.
 Its structure can be checked, however, and the theorem makes three predictions that are concrete enough to test. First, under the orthogonality constraint the crosstalk injected at a layer acts only on the residual part of the activations, so its energy should equal $\Tr(\Cb_j^{(\ell)}\Sigmab_j^{(\ell),\perp}(\Cb_j^{(\ell)})^\top)$ (Proposition~\ref{prop:crosstalk_residual} and Lemma~\ref{lem:hanson_wright}). Second, the per-concept quantity $\sum_\ell \Gamma_\ell\kappa_\ell\sqrt{\epsilon_\ell}$ that appears in the bound should order the concepts by how much they are forgotten. Third, a basis learned inside the feasible complement should leave less residual energy, and therefore forget less, than a random basis frozen in the same complement (Theorem~\ref{thm:optimal_feasible_basis} and Proposition~\ref{prop:random_frozen_basis}). We test these three predictions in a small nonlinear network.

\paragraph{Setup.} The network has $L=3$ LoRA-adapted linear layers of sizes $48{\times}48$, $48{\times}48$ and $16{\times}48$, separated by $\tanh$ ($\gamma_\ell=1$), and every layer carries a rank-$r=4$ residual $\Ab_j^{(\ell)}(\Bb_j^{(\ell)})^\top$. A stream of $T=8$ concepts arrives sequentially. Concept $j$ has its own input distribution $\xb=\Sigmab_j^{1/2}\zb$ with $\zb\sim\Nc(\mathbf 0,\Ib_{48})$; the covariance $\Sigmab_j$ has six dominant eigen-directions, with eigenvalues decreasing from $3.5$ to $1.3$, and an isotropic floor of $0.0225$ on the remaining $42$ directions. Each concept also has its own teacher, obtained by adding a concept-specific rank-$4$ perturbation to every layer of the pretrained network. As in the theory, a batch is a matrix $\Xb_j^{(\ell)}\in\Rbb^{m\times p}$ whose $p$ columns are tokens; here $p=256$ and the tokens are i.i.d., so $\Psib_j^{(\ell)}=\Ib_p$ and the factor $\Tr(\Psib_j^{(\ell)})=p$ is common to all quantities and is dropped. The LoRA factors of concept $j$ are trained for $400$ steps of gradient descent (step size $0.05$) on the squared loss to the teacher, with $\Ab_j^{(\ell)}$ initialised at zero and $\Bb_j^{(\ell)}$ at random. Three ways of handling the basis are compared:
\begin{itemize}[leftmargin=*,itemsep=1pt,topsep=2pt]
    \item \emph{learned, orthogonal}: both factors are trained, and $\Bb_j^{(\ell)}$ is projected onto the complement of the previous bases after every step, as in \eqref{eq:update-B};
    \item \emph{frozen random, orthogonal}: $\Bb_j^{(\ell)}$ is drawn at random inside the feasible complement and kept fixed, and only $\Ab_j^{(\ell)}$ is trained; this is the abstraction of frozen-basis methods analysed in Proposition~\ref{prop:random_frozen_basis};
    \item \emph{learned, unconstrained}: both factors are trained and no orthogonality is imposed.
\end{itemize}
Once the stream is complete we take $4{,}096$ fresh samples of each concept $j<T$ (the last concept has no successors and hence no crosstalk) and compute, at every layer, the residual energy $\epsilon_\ell=\Tr(\Sigmab_j^{(\ell),\perp})$ of the layer inputs under $\theta_j$, the crosstalk magnitude $\kappa_\ell=\nt{\Cb_j^{(\ell)}}$, the amplification $\Gamma_\ell=\prod_{\ell'>\ell}\nt{\Wb_T^{(\ell')}}$, the injected crosstalk energy $\|\Cb_j^{(\ell)}\Xb_j^{(\ell)}\|_F^2/p$ together with its residual-only counterpart $\Tr(\Cb_j^{(\ell)}\Sigmab_j^{(\ell),\perp}(\Cb_j^{(\ell)})^\top)$, and finally the forgetting $\Lc_j(\theta_T)-\Lc_j(\theta_j)$, i.e.\ the measured increase of the loss between the model right after concept $j$ and the final model. With five seeds this gives $35$ (concept, seed) pairs per strategy. Figure~\ref{fig:theory_numerics} shows all of them and Table~\ref{tab:theory_numerics} reports the means.

\begin{figure}[h]
\centering
\includegraphics[width=\textwidth]{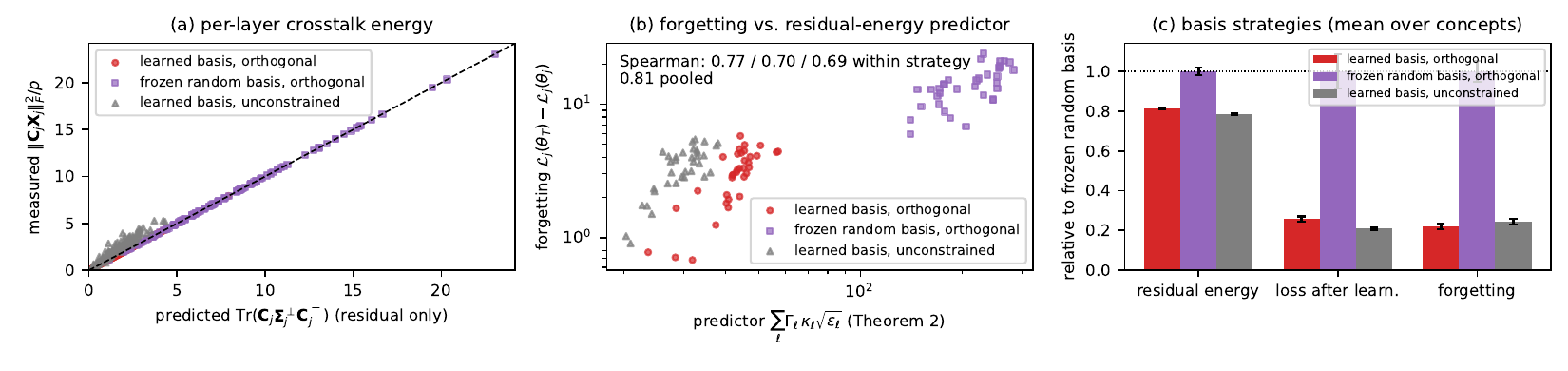}
\caption{Synthetic study of Theorem~\ref{thm:hanson_wright_catas}. (a)~Measured per-layer crosstalk energy $\|\Cb_j^{(\ell)}\Xb_j^{(\ell)}\|_F^2/p$ against the residual-only value $\Tr(\Cb_j^{(\ell)}\Sigmab_j^{(\ell),\perp}(\Cb_j^{(\ell)})^\top)$ of Lemma~\ref{lem:hanson_wright}, one point per (concept, layer, seed). Under orthogonality the two coincide (Proposition~\ref{prop:crosstalk_residual}); without the constraint the measured value is larger because the crosstalk also acts on the captured subspace. (b)~Measured forgetting of each concept against the predictor $\sum_\ell\Gamma_\ell\kappa_\ell\sqrt{\epsilon_\ell}$ of Theorem~\ref{thm:hanson_wright_catas} (log--log axes, one point per concept and seed). (c)~Mean residual energy, loss right after learning, and forgetting of the three strategies, relative to the frozen random basis (error bars: standard error over the 35 pairs).}
\label{fig:theory_numerics}
\end{figure}

\begin{table}[h]
\centering
\caption{Summary of the synthetic study (mean $\pm$ standard error over 35 (concept, seed) pairs; energies summed over the three layers). ``Forgetting'' is the measured loss increase $\Lc_j(\theta_T)-\Lc_j(\theta_j)$, not the bound. ``Residual-only'' is $\sum_\ell\Tr(\Cb_j^{(\ell)}\Sigmab_j^{(\ell),\perp}(\Cb_j^{(\ell)})^\top)$ and ``measured'' is $\sum_\ell\|\Cb_j^{(\ell)}\Xb_j^{(\ell)}\|_F^2/p$; their difference is the crosstalk that passes through the captured subspace, which orthogonality removes. The last column is the Spearman correlation, within the strategy, between the predictor of Theorem~\ref{thm:hanson_wright_catas} and the measured forgetting.}
\label{tab:theory_numerics}
\resizebox{\textwidth}{!}{%
\begin{tabular}{@{}lcccccc@{}}
\toprule
Basis strategy & Residual energy $\sum_\ell\epsilon_\ell$ & Loss after learning & Forgetting & \multicolumn{2}{c}{Crosstalk energy} & Spearman \\
\cmidrule(lr){5-6}
 & & & & residual-only & measured & \\
\midrule
Frozen random, orthogonal & $34.3 \pm 0.7$ & $1.32 \pm 0.11$ & $14.2 \pm 0.8$ & $23.8$ & $23.8$ & $0.70$ \\
Learned, orthogonal       & $27.9 \pm 0.2$ & $0.34 \pm 0.02$ & $3.1 \pm 0.2$  & $6.67$ & $6.67$ & $0.77$ \\
Learned, unconstrained    & $26.9 \pm 0.2$ & $0.28 \pm 0.01$ & $3.5 \pm 0.2$  & $4.95$ & $7.02$ & $0.69$ \\
\bottomrule
\end{tabular}}
\end{table}

\paragraph{The residual covariance as the only interference channel.} For the two orthogonal strategies the measured crosstalk energy coincides with the residual-only value at every layer, seed and concept: all their points in panel~(a) lie on the diagonal, and the summed values in Table~\ref{tab:theory_numerics} agree to all digits. This is not surprising in itself, since under exact orthogonality the two quantities are algebraically equal ($\Cb_j^{(\ell)}\Pc_{B_j}^{(\ell)}=\mathbf 0$), so for these strategies the panel mainly confirms that the projection is exact at every layer of the trained network. The informative case is the unconstrained one. There the measured crosstalk is larger than the residual-only value, by $30\%$ on average ($7.02$ against $4.95$ summed over layers). The excess is interference that enters through the captured subspace $\mathrm{col}(\Bb_j^{(\ell)})$ once later bases overlap with it, and it is the channel that the orthogonality constraint removes.

\paragraph{Ordering of concepts by the predictor.} Within each strategy, the Spearman correlation between $\sum_\ell\Gamma_\ell\kappa_\ell\sqrt{\epsilon_\ell}$ and the measured forgetting is $0.77$ for the learned orthogonal basis, $0.70$ for the frozen orthogonal basis and $0.69$ for the unconstrained basis; pooled over the three strategies it is $0.81$, with a Pearson correlation of $0.93$. The relation is monotone but not linear. Panel~(b) also shows that the predictor separates the three strategies by close to an order of magnitude, in the same direction as the measured forgetting. Its absolute value overestimates the forgetting by a roughly constant factor, which is what one expects from the conservative prefactors in the bound; the predictor should therefore be used to rank concepts and strategies, not as a numerical estimate of the loss increase.

\paragraph{Learned versus frozen basis, and what the constraint costs.} Compared with a random basis frozen in the feasible complement, learning the basis lowers the summed residual energy from $34.3$ to $27.9$, the loss right after learning from $1.32$ to $0.34$, and the forgetting from $14.2$ to $3.1$ (panel~(c)). The reduction in forgetting is considerably larger than the reduction in residual energy. The reason is that a frozen basis affects every factor of the bound at once: the concept is fitted less well, the factor $\Ab_j^{(\ell)}$ grows to compensate and thereby inflates the crosstalk magnitudes ($\kappa_\ell$ averages $3.1$ against $1.6$), and the larger weights also increase the downstream amplification ($\Gamma_1=15.0$ against $4.8$). This is the mechanism behind Proposition~\ref{prop:random_frozen_basis} and the second part of Theorem~\ref{thm:hanson_wright_catas}.

The unconstrained learned basis reaches a slightly lower residual energy than the orthogonal one ($26.9$ against $27.9$) and a slightly lower loss after learning ($0.28$ against $0.34$). This is the expressiveness cost of the constraint that Theorem~\ref{thm:optimal_feasible_basis} makes explicit: a feasible basis has to lie in the free complement, so whatever part of concept $j$'s activation energy falls inside the subspaces of earlier concepts, $\Tr((\Ib-\Pc_{\mathrm{free},j})\Sigmab_j)$, necessarily stays in the residual. Consistently with this, the residual-energy gap between the two strategies is $0.0$ for the first concept and grows along the stream ($0.0$, $-0.1$, $0.1$, $0.9$, $1.7$, $2.3$, $1.6$ for $j=1,\dots,7$). In spite of its smaller residual, the unconstrained basis forgets slightly more ($3.5$ against $3.1$): the $2.1$ units of crosstalk that pass through its captured subspace outweigh the $1.0$ unit of residual energy it saves. The constraint thus exchanges a somewhat larger residual for a closed leakage channel. With $m=48$ and eight rank-$4$ concepts, bases learned independently overlap little, so the net effect on forgetting is small in this setting; the concept-similarity experiment in Appendix~\ref{sec:supp_similarity} (Table~\ref{tab:similarity}) suggests that it becomes larger as the concepts become more similar.
  
\newpage
\section{Evaluation Metrics}\label{app:metrics}
We use the following metrics to evaluate different aspects of the generated images:
\begin{itemize}[leftmargin=*]
    \item \textbf{CLIP-I / CLIP-T} \citep{radford2021learning, ruiz2023dreambooth}: image-to-image and text-to-image alignment in CLIP space, measuring whether generated images match the reference concept and the prompt.
    \item \textbf{DINO / DINOv2 / DINOv3}~\citep{caron2021emerging, oquab2023dinov2, simeoni2025dinov3}: cosine similarity of self-supervised features to the reference images, which is sensitive to identity rather than category.
    \item \textbf{DreamSim}~\citep{fu2023dreamsim}: holistic visual similarity (layout, pose, and semantic content) to the reference images.
    \item \textbf{HPSv2 / HPSv3}~\citep{wu2023human, ma2025hpsv3}: human preference scores for overall visual quality.
\end{itemize}

\section{Evaluation Protocols}
\label{sec:appendix_eval_methodology}
This section details the three protocols of Section~\ref{sec:experiments} and the aggregation rules. In all protocols each per-concept score is first averaged over $16$ generated images obtained with different random seeds.

\subsection{Single-concept retention (Table~\ref{tab:overall_comparison_32}, Figure~\ref{fig:overall_comparison}, Table~\ref{tab:supp_overall})}
For a stream of $T$ concepts we evaluate the model only after it has been trained on all $T$ concepts. The final composed model generates 16 images for every one of the $T$ concepts using the CustomConcept101 prompts, yielding $16 \times T$ images; each metric is computed for every image against the concept's reference images. Figure~\ref{fig:overall_comparison} plots the mean of the $16\times T$ scores (at $T=101$ this is the mean of 1,616 generations). Tables report mean $\pm$ standard error of the mean (SEM) over the $T$ per-concept scores, $\mathrm{SEM} = \sigma/\sqrt{T}$ with $\sigma$ the unbiased standard deviation across concepts; SEMs therefore capture concept-level variability and decrease naturally with $T$.

\subsection{Forgetting (Figure~\ref{fig:forgetting}, Table~\ref{tab:summary_forgetting_new}, Table~\ref{tab:supp_forgetting})}
Figure~\ref{fig:forgetting} tracks the 32-concept sequential run: a point at step $k$ is the state of the model right after learning the $k$-th concept, evaluated on all concepts $1,\dots,k$ seen so far ($16 \times k$ images), and the plotted value is the mean over these images. {Because the evaluated pool grows with $k$ and later concepts may be intrinsically harder, a downward slope of this curve is not, by itself, forgetting; the curve measures the aggregate health of the model along the stream and is best read as a comparison between methods trained on the same stream.} Table~\ref{tab:summary_forgetting_new} isolates forgetting: for each concept $j$ the \emph{Initial} score is measured right after concept $j$ is learned and the \emph{Final} score in the final $T$-concept model, and \emph{Forgetting} is their difference, averaged over concepts.

\subsection{Multi-concept composition (Table~\ref{tab:composition}, Figure~\ref{fig:regional_comparison})}
We follow the regional generation protocol of \citet{gu2023mix}: three or four concepts learned in the same 32-concept stream are rendered in one image, each in a region specified by a bounding box, with sketch and keypose conditioning shared across methods and the same random seed for every method. For each composition and seed we crop each concept's region and compute DINOv2 cosine similarities: \emph{identity} is the similarity of the crop to the concept's own reference images, and \emph{leakage} is its similarity to the reference images of the other concepts in the same composition (averaged over those concepts). CLIP-T is computed on the full image against the composition prompt. Scores are averaged over five seeds per composition and over compositions; the significance test in Table~\ref{tab:composition} is a paired Wilcoxon signed-rank test over compositions between SeqLoRA and \CA{}.

\section{Additional Experiments}\label{sec:supp_experiments}
  
\subsection{Sensitivity to concept order}\label{sec:supp_order}
Because SeqLoRA projects each new basis onto the complement of the bases learned before it, later concepts are learned in a smaller subspace, and one may ask whether the order of the stream matters. Table~\ref{tab:order} reports composition identity and leakage (protocol of Appendix~\ref{sec:appendix_eval_methodology}) for four random orderings of a 16-concept stream, five seeds each. The variation across orderings is small (standard deviation $\le 0.007$, coefficient of variation $\le 3\%$), indicating limited sensitivity to order at this stream length; with $m = 768$ and $r = 4$, sixteen concepts occupy 64 of 768 dimensions per layer, so the complement available to the last concept remains large.
\begin{table}[h]
\centering
\caption{Sensitivity to concept order: SeqLoRA on four random orderings of a 16-concept stream (5 seeds each, DINOv2 features, composition protocol).}
\label{tab:order}
\small
\begin{tabular}{@{}lcc@{}}
\toprule
Ordering & Identity$\uparrow$ & Leakage$\downarrow$ \\
\midrule
Order 1 & 0.316 & 0.206 \\
Order 2 & 0.311 & 0.201 \\
Order 3 & 0.324 & 0.216 \\
Order 4 & 0.324 & 0.207 \\
\midrule
Mean $\pm$ std & $0.319 \pm 0.007$ & $0.208 \pm 0.006$ \\
\bottomrule
\end{tabular}
\end{table}

\subsection{Sensitivity to LoRA rank}\label{sec:supp_rank}
The orthogonality constraint admits at most $\lfloor m/r\rfloor$ concepts per layer, so the rank trades capacity against per-concept expressiveness. Table~\ref{tab:rank} varies the rank on a 16-concept stream. Identity is not monotone in the rank, and a larger rank increases leakage without improving the identity--leakage margin; we therefore use $r=4$ throughout.
\begin{table}[h]
\centering
\caption{LoRA-rank sensitivity of SeqLoRA (16 concepts, DINOv2 features, composition protocol). Margin is identity minus leakage.}
\label{tab:rank}
\small
\begin{tabular}{@{}lccc@{}}
\toprule
LoRA rank & Identity$\uparrow$ & Leakage$\downarrow$ & Margin$\uparrow$ \\
\midrule
$r=4$  & 0.322 & 0.200 & 0.122 \\
$r=8$  & 0.316 & 0.221 & 0.095 \\
$r=16$ & 0.338 & 0.219 & 0.119 \\
\bottomrule
\end{tabular}
\end{table}

\subsection{Effect of concept similarity}\label{sec:supp_similarity}
Interference is expected to be largest between visually similar concepts, whose activations share directions. To test where the orthogonal projection matters, we compare a group of visually similar concepts with a group of diverse concepts under the composition protocol, with and without the projection step of \eqref{eq:update-B}. Table~\ref{tab:similarity} shows that the projection reduces leakage substantially for the similar group and is essentially a no-op for the diverse group, whose leakage is already low. The mechanism therefore helps precisely in the regime where a dimension-counting capacity argument is least informative.
\begin{table}[h]
\centering
\caption{Cross-concept leakage ($\downarrow$, DINOv2 features) for a group of visually similar concepts and a group of diverse concepts, with the orthogonal projection of \eqref{eq:update-B} disabled and enabled.}
\label{tab:similarity}
\small
\begin{tabular}{@{}lcc@{}}
\toprule
Concept group & Without projection & With projection \\
\midrule
Visually similar & 0.434 & 0.386 \\
Diverse          & 0.070 & 0.078 \\
\bottomrule
\end{tabular}
\end{table}

\subsection{Compute and memory}\label{sec:supp_compute}
Table~\ref{tab:efficiency} in the main text reports end-to-end cost at 32 concepts. Table~\ref{tab:memory} reports the peak memory of the composition step as a function of the number of concepts. Optimization-based fusion must instantiate all per-concept adapters (and the optimizer state) simultaneously, so its peak memory grows with the number of concepts; LoRACLR's grows most steeply and reaches 44.7\,GB at 32 concepts, which is consistent with the out-of-memory failures of Mix-of-Show and LoRACLR beyond 48 and 64 concepts in Table~\ref{tab:supp_overall}. Training-step peak memory is independent of the number of concepts for all methods, since each concept's adapter is trained with the previously learned bases held fixed (7.3\,GB for SeqLoRA vs.\ 6.9\,GB for the parallel baselines, per step).
\begin{table}[h]
\centering
\caption{Peak memory (GB) of the composition step versus number of concepts. For SeqLoRA the composition step is the summation of adapters into the base weights.}
\label{tab:memory}
\small
\begin{tabular}{@{}lccc@{}}
\toprule
\# Concepts & Mix-of-Show & LoRACLR & SeqLoRA \\
\midrule
4  & 8.9  & 12.9 & 9.2 \\
16 & 12.8 & 23.1 & 12.9 \\
32 & 18.4 & 44.7 & 17.9 \\
\bottomrule
\end{tabular}
\end{table}

\subsection{Hyperparameter ablation}\label{sec:supp_ablation}
Table~\ref{tab:ablation_32} varies the number of local steps $S_A=S_B$, the regularization $\epsilon$ of the projection, and the number of bilevel iterations $K$ for the 32-concept stream. Reducing the local steps or the bilevel iterations to one leaves the identity metrics (DINO, DINOv2, DINOv3, DreamSim) essentially unchanged but lowers HPSv3; five local steps raise HPSv3 with a marginal change in identity; larger $\epsilon$ ($10^{-4}$) slightly reduces identity scores. The default configuration keeps near-optimal identity preservation with a competitive preference score. {Since a single bilevel iteration matches the default identity metrics, the per-concept training cost of Section~\ref{sec:efficiency} can be reduced roughly threefold when preference scores are not the priority.}
\begin{table}[h]
\centering
\caption{Hyperparameter ablation for SeqLoRA on the 32-concept stream (default: $S_A=S_B=2$ local steps, $\epsilon=10^{-8}$, $K=3$ bilevel iterations).}
\label{tab:ablation_32}
\resizebox{\textwidth}{!}{%
\begin{tabular}{@{}llcccccccc@{}}
\toprule
Factor & Value & DINO$\uparrow$ & DINOv2$\uparrow$ & DINOv3$\uparrow$ & DreamSim$\uparrow$ & CLIP-I$\uparrow$ & CLIP-T$\uparrow$ & HPSv2$\uparrow$ & HPSv3$\uparrow$ \\
\midrule
Default & - & 0.468 $\pm$ 0.021 & 0.418 $\pm$ 0.019 & 0.418 $\pm$ 0.016 & 0.477 $\pm$ 0.016 & 0.693 $\pm$ 0.009 & 0.286 $\pm$ 0.003 & 0.270 $\pm$ 0.001 & 9.043 $\pm$ 0.399 \\
\midrule
\multirow[c]{2}{*}{Local Steps} & 1 Step & 0.466 $\pm$ 0.021 & 0.418 $\pm$ 0.017 & 0.414 $\pm$ 0.015 & 0.479 $\pm$ 0.016 & 0.694 $\pm$ 0.009 & 0.287 $\pm$ 0.003 & 0.270 $\pm$ 0.001 & 8.814 $\pm$ 0.422 \\
 & 5 Steps & 0.469 $\pm$ 0.021 & 0.417 $\pm$ 0.020 & 0.415 $\pm$ 0.017 & 0.475 $\pm$ 0.016 & 0.690 $\pm$ 0.009 & 0.286 $\pm$ 0.003 & 0.271 $\pm$ 0.001 & 9.378 $\pm$ 0.346 \\
\midrule
\multirow[c]{2}{*}{Epsilon} & $10^{-4}$ & 0.459 $\pm$ 0.021 & 0.406 $\pm$ 0.018 & 0.405 $\pm$ 0.016 & 0.470 $\pm$ 0.017 & 0.686 $\pm$ 0.009 & 0.288 $\pm$ 0.003 & 0.271 $\pm$ 0.001 & 8.838 $\pm$ 0.402 \\
 & $10^{-5}$ & 0.466 $\pm$ 0.021 & 0.422 $\pm$ 0.019 & 0.417 $\pm$ 0.017 & 0.476 $\pm$ 0.016 & 0.692 $\pm$ 0.009 & 0.287 $\pm$ 0.004 & 0.271 $\pm$ 0.001 & 8.914 $\pm$ 0.402 \\
\midrule
\multirow[c]{2}{*}{Bilevel Iters} & 1 Iter & 0.466 $\pm$ 0.021 & 0.418 $\pm$ 0.019 & 0.417 $\pm$ 0.016 & 0.477 $\pm$ 0.017 & 0.695 $\pm$ 0.009 & 0.288 $\pm$ 0.003 & 0.270 $\pm$ 0.001 & 8.784 $\pm$ 0.377 \\
 & 4 Iters & 0.464 $\pm$ 0.020 & 0.408 $\pm$ 0.018 & 0.410 $\pm$ 0.016 & 0.474 $\pm$ 0.016 & 0.691 $\pm$ 0.009 & 0.288 $\pm$ 0.003 & 0.271 $\pm$ 0.001 & 8.855 $\pm$ 0.390 \\
\bottomrule
\end{tabular}%
}
\end{table}

\subsection{Preliminary transfer to SDXL}\label{sec:supp_sdxl}
The projection \eqref{eq:update-B} is a closed-form operation on the basis factors that depends only on the LoRA rank and the layer width, with no assumption specific to Stable Diffusion v1.5 or to a single text encoder. To test this, we ported SeqLoRA to SDXL-base~\citep{podell2024sdxl}, whose architecture differs substantially (two text encoders, 2048-dimensional cross-attention, 70 cross-attention layers, and pooled/time micro-conditioning); only the architecture-specific wiring (dual-encoder conditioning and the per-layer embedding count) was adapted. Table~\ref{tab:sdxl} reports a four-concept stream with rank 4 and single-concept evaluation: DINO (ViT-S/16) similarity of generated images to the concept's own references (identity) and to the other concepts' references (cross-concept, a genericness baseline). Identity is in the usual personalization range with a large margin over cross-concept similarity, indicating faithful, concept-specific transfer.
\begin{table}[h]
\centering
\caption{SeqLoRA on SDXL-base (4 concepts, rank 4): per-concept DINO identity versus cross-concept similarity.}
\label{tab:sdxl}
\small
\begin{tabular}{@{}lccc@{}}
\toprule
Concept & DINO identity$\uparrow$ & Cross-concept & Margin$\uparrow$ \\
\midrule
\texttt{things\_cup2}       & 0.747 & 0.193 & $+0.554$ \\
\texttt{plushie\_cow}       & 0.620 & 0.207 & $+0.413$ \\
\texttt{wearable\_glasses}  & 0.591 & 0.186 & $+0.405$ \\
\texttt{actionfigure\_1}    & 0.711 & 0.218 & $+0.493$ \\
\midrule
Mean & 0.667 & 0.201 & $+0.466$ \\
\bottomrule
\end{tabular}
\end{table}

\section{Qualitative Results}\label{sec:supp_qual}
We present qualitative comparisons for all 32 concepts of the main experiment, ordered by their position in the training sequence. Figures~\ref{fig:supp_qual1}, \ref{fig:supp_qual2}, and \ref{fig:supp_qual3} show concepts 1--12, 13--24, and 25--32. As in Figure~\ref{fig:qualitative}, each row shows the reference images and samples from SeqLoRA, \CA{}, LoRACLR, Mix-of-Show, and Orthogonal Adaptation, all generated with the same seeds.

\begin{figure}[H]
\centering
\includegraphics[width=\textwidth,height=0.95\textheight,keepaspectratio]{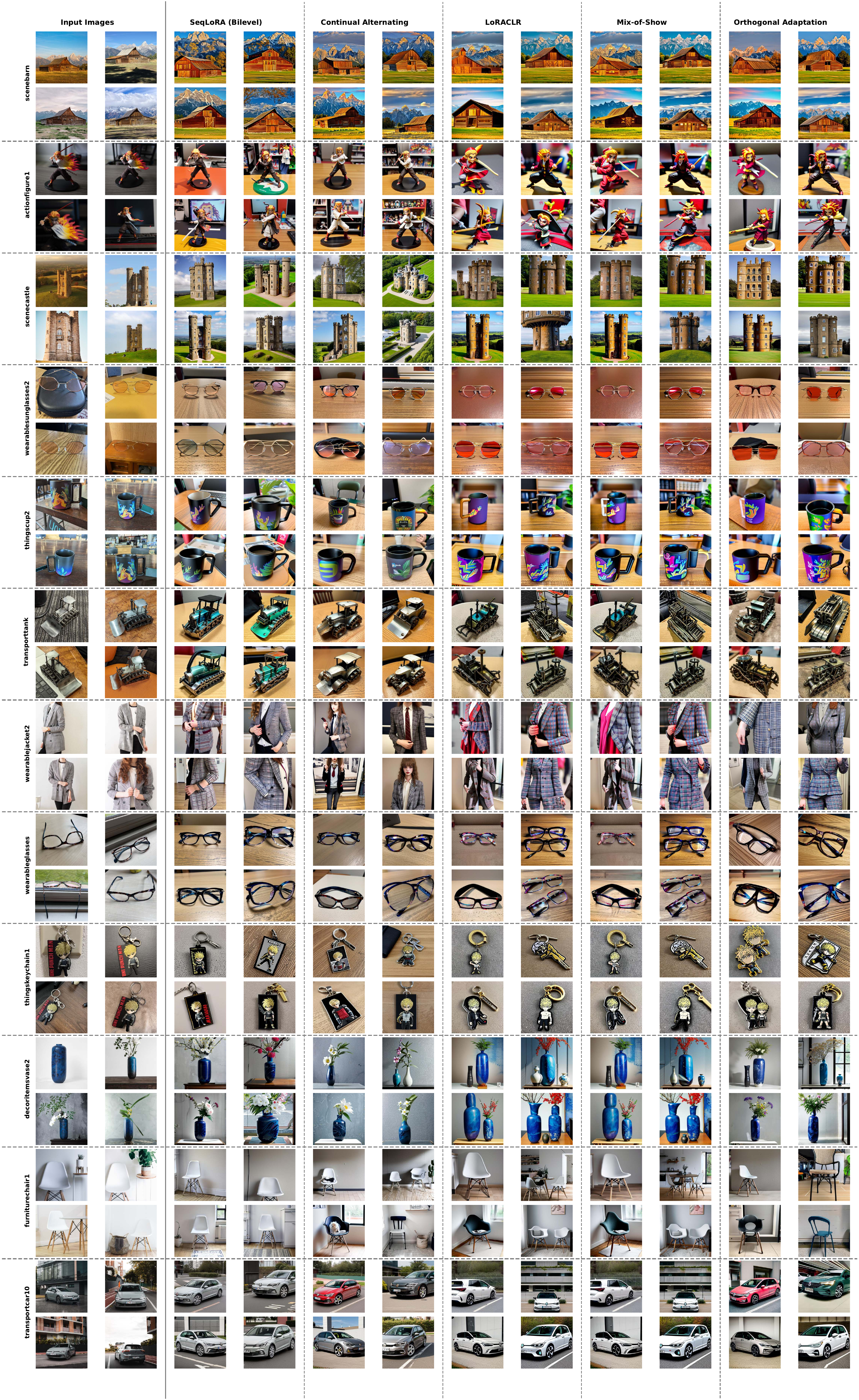}
\caption{Qualitative comparison for concepts 1--12 (ordered by training sequence); same seeds for all methods.}
\label{fig:supp_qual1}
\end{figure}
\begin{figure}[H]
\centering
\includegraphics[width=\textwidth,height=0.95\textheight,keepaspectratio]{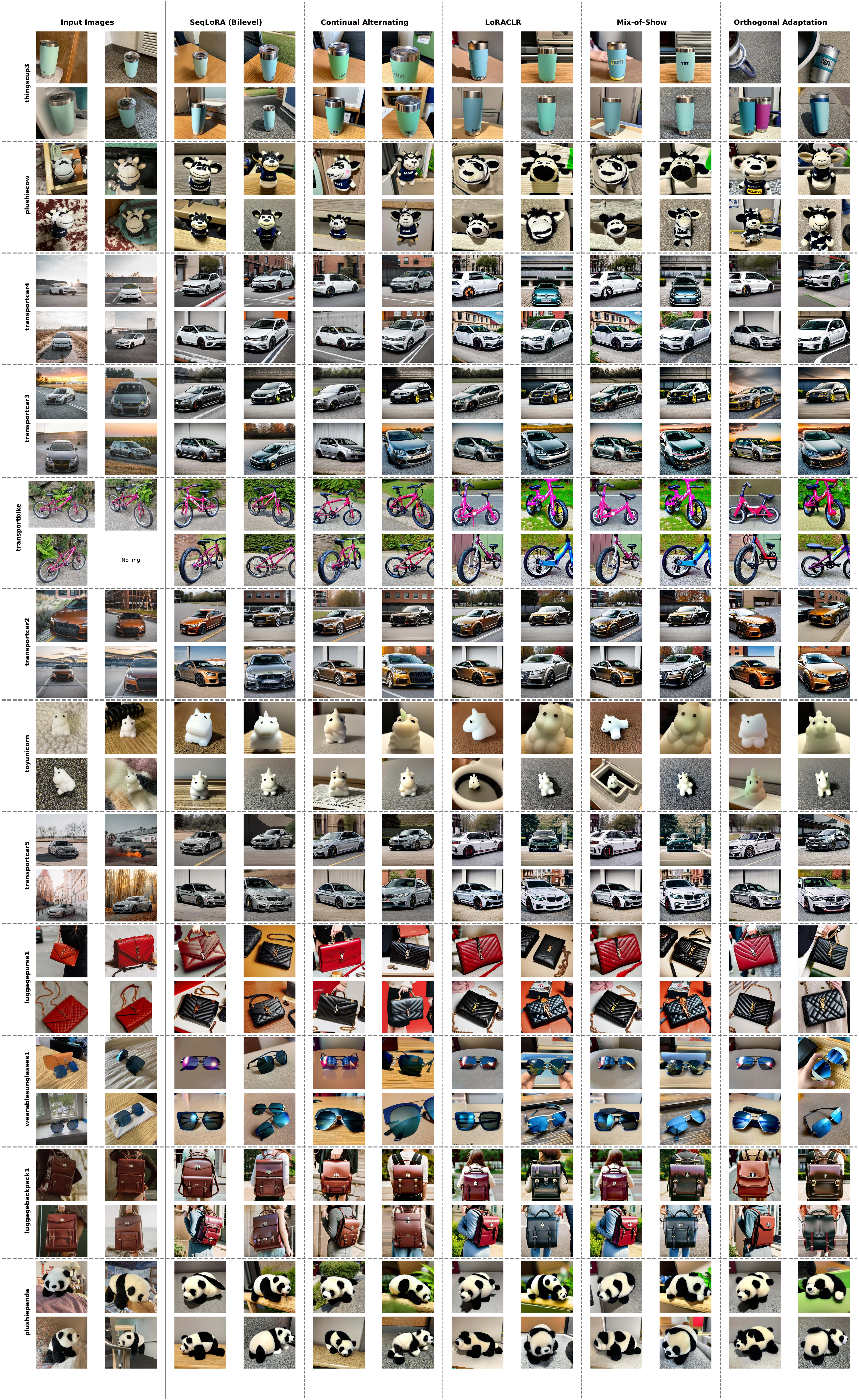}
\caption{Qualitative comparison for concepts 13--24 (ordered by training sequence); same seeds for all methods.}
\label{fig:supp_qual2}
\end{figure}
\begin{figure}[H]
\centering
\includegraphics[width=\textwidth,height=0.95\textheight,keepaspectratio]{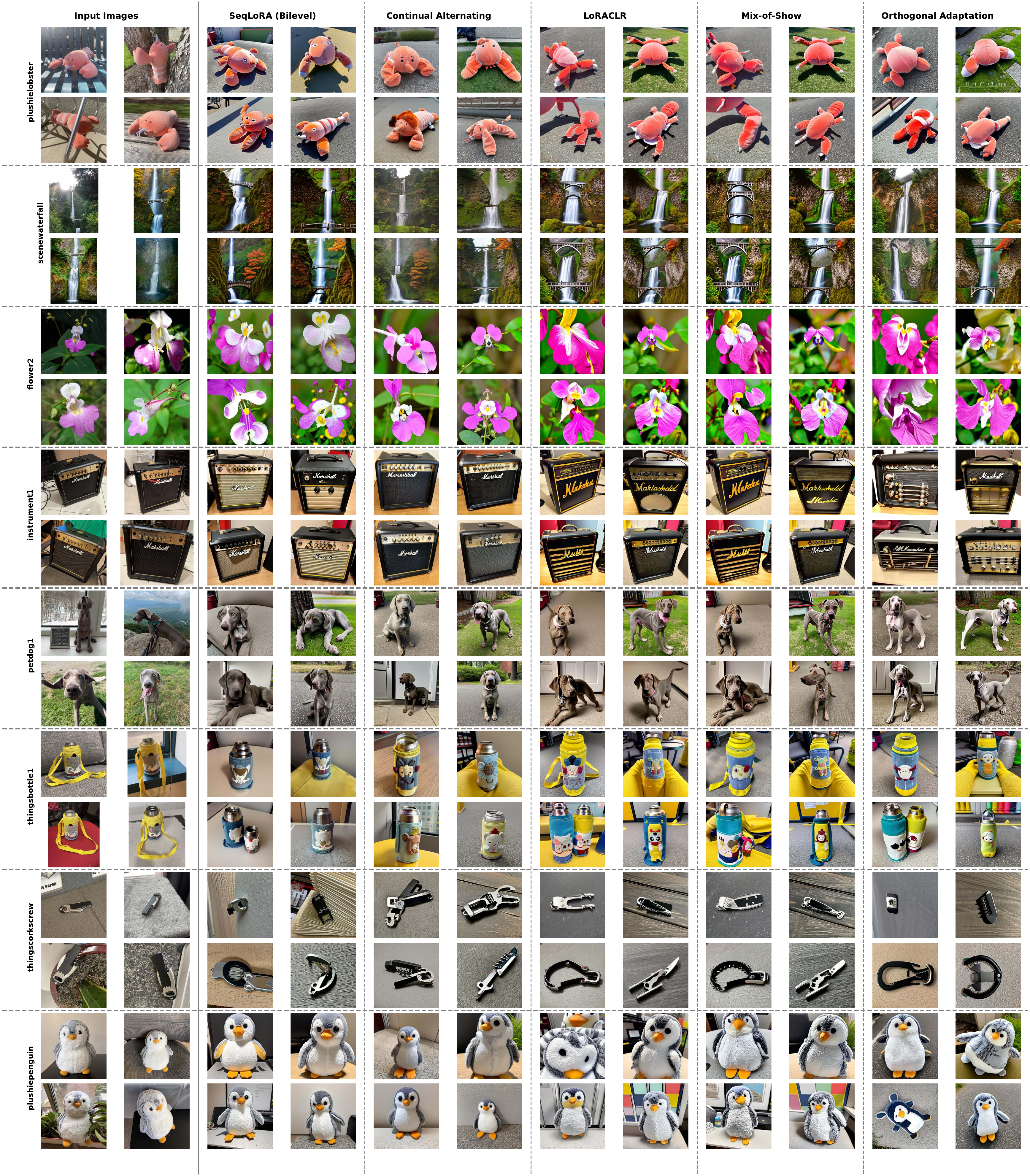}
\caption{Qualitative comparison for concepts 25--32 (ordered by training sequence); same seeds for all methods.}
\label{fig:supp_qual3}
\end{figure}

\section{Multi-Concept Regional Generation}\label{sec:supp_regional}
Figure~\ref{fig:regional_comparison} shows the compositions that underlie Table~\ref{tab:composition}: a three-concept scene combining two chairs, a table and a vase, and a four-concept scene that adds a dog. Every method uses the same random seed and the same regional control inputs (sketch, keypose, and per-region bounding boxes), so visual differences are attributable to the multi-concept mechanism rather than to randomness or spatial conditioning. SeqLoRA preserves per-region identity (chair upholstery, vase silhouette, dog appearance), whereas the fusion-based baselines and \CA{} exhibit visible attribute leakage across regions.

\begin{figure}[H]
\centering
\setlength{\tabcolsep}{2pt}
\renewcommand{\arraystretch}{1.0}

\begin{subfigure}{\textwidth}
\centering
\includegraphics{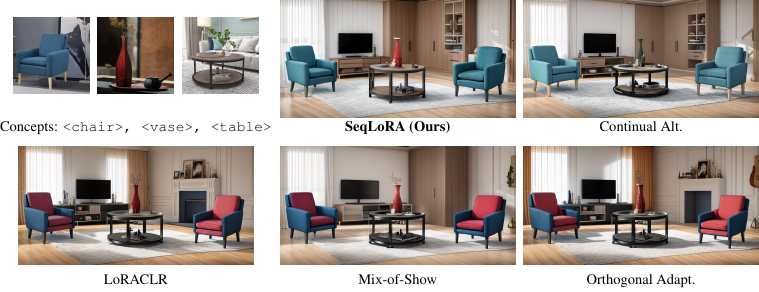}
\caption{\textbf{Three concepts.} Prompt: \textit{``two chairs, a table, and a vase, in a living room''}.}
\label{fig:regional_comp_a}
\end{subfigure}

\vspace{8pt}

\begin{subfigure}{\textwidth}
\centering
\includegraphics{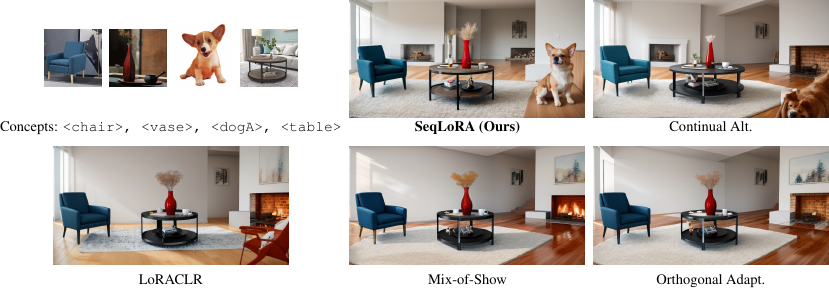}
\caption{\textbf{Four concepts.} Prompt: \textit{``a chair, a vase, a dog and a table, in a living room''}.}
\label{fig:regional_comp_b}
\end{subfigure}

\caption{\textbf{Multi-concept regional generation: qualitative comparison.}
Each subfigure shows, in the top-left cell, the input concept reference images,
followed by five generated outputs (one per method) arranged in a $3\times 2$ grid.
All methods use the \emph{same} random seed and the same regional sketch/keypose
conditioning, so visual differences reflect the underlying multi-concept fusion
mechanism rather than randomness or spatial conditioning. SeqLoRA preserves
per-region concept identity (e.g., the chair upholstery, vase silhouette, and
dog appearance) more faithfully than the baselines, which exhibit attribute
leakage and identity collapse across regions.}
\label{fig:regional_comparison}
\end{figure}

\section{Additional Forgetting Results}\label{sec:supp_forgetting}
Table~\ref{tab:summary_forgetting_full} expands Table~\ref{tab:summary_forgetting_new} with the initial and final scores from which the forgetting values are computed.
\begin{table}[h]
    \centering
    \caption{\textbf{Forgetting for SeqLoRA across stream lengths.} \emph{Initial}: score of each concept measured right after it is learned; \emph{Final}: score of the same concept in the final $T$-concept model; \emph{Forgetting}: Initial $-$ Final (lower is better; negative means improvement). Mean $\pm$ standard error over the $T$ concepts.}
    \vspace{-0.1cm}
    \resizebox{\textwidth}{!}{%
    \begin{tabular}{@{}llcccccccc@{}}
    \toprule
    $T$ & Measure & CLIP-I & CLIP-T & DINO & DINOv2 & DINOv3 & DreamSim & HPSv2 & HPSv3 \\
    \midrule
    \multirow[c]{3}{*}{8}  & Initial  & 0.709 $\pm$ 0.018 & 0.286 $\pm$ 0.005 & 0.485 $\pm$ 0.047 & 0.491 $\pm$ 0.049 & 0.468 $\pm$ 0.039 & 0.495 $\pm$ 0.040 & 0.270 $\pm$ 0.003 & 9.513 $\pm$ 0.579 \\
      & Final  & 0.695 $\pm$ 0.014 & 0.289 $\pm$ 0.005 & 0.462 $\pm$ 0.041 & 0.465 $\pm$ 0.041 & 0.442 $\pm$ 0.030 & 0.478 $\pm$ 0.035 & 0.271 $\pm$ 0.003 & 9.539 $\pm$ 0.588 \\
     & Forgetting & 0.014 $\pm$ 0.006 & $-$0.002 $\pm$ 0.001 & 0.022 $\pm$ 0.009 & 0.025 $\pm$ 0.009 & 0.026 $\pm$ 0.010 & 0.017 $\pm$ 0.008 & $-$0.001 $\pm$ 0.000 & $-$0.025 $\pm$ 0.020 \\
    \midrule
    \multirow[c]{3}{*}{16}  & Initial  & 0.703 $\pm$ 0.013 & 0.281 $\pm$ 0.004 & 0.492 $\pm$ 0.030 & 0.436 $\pm$ 0.032 & 0.437 $\pm$ 0.027 & 0.493 $\pm$ 0.026 & 0.268 $\pm$ 0.002 & 9.494 $\pm$ 0.416 \\
      & Final  & 0.694 $\pm$ 0.012 & 0.284 $\pm$ 0.004 & 0.474 $\pm$ 0.029 & 0.416 $\pm$ 0.028 & 0.414 $\pm$ 0.022 & 0.478 $\pm$ 0.024 & 0.269 $\pm$ 0.002 & 9.476 $\pm$ 0.424 \\
     & Forgetting & 0.009 $\pm$ 0.004 & $-$0.002 $\pm$ 0.001 & 0.018 $\pm$ 0.007 & 0.020 $\pm$ 0.007 & 0.022 $\pm$ 0.008 & 0.014 $\pm$ 0.005 & $-$0.000 $\pm$ 0.000 & 0.019 $\pm$ 0.029 \\
    \midrule
    \multirow[c]{3}{*}{32}  & Initial  & 0.696 $\pm$ 0.009 & 0.288 $\pm$ 0.003 & 0.465 $\pm$ 0.022 & 0.413 $\pm$ 0.021 & 0.420 $\pm$ 0.018 & 0.481 $\pm$ 0.017 & 0.270 $\pm$ 0.001 & 9.044 $\pm$ 0.396 \\
      & Final  & 0.687 $\pm$ 0.009 & 0.290 $\pm$ 0.003 & 0.452 $\pm$ 0.022 & 0.398 $\pm$ 0.019 & 0.402 $\pm$ 0.017 & 0.469 $\pm$ 0.016 & 0.270 $\pm$ 0.001 & 9.043 $\pm$ 0.399 \\
     & Forgetting & 0.008 $\pm$ 0.003 & $-$0.001 $\pm$ 0.000 & 0.013 $\pm$ 0.004 & 0.016 $\pm$ 0.005 & 0.017 $\pm$ 0.005 & 0.012 $\pm$ 0.004 & $-$0.000 $\pm$ 0.000 & 0.001 $\pm$ 0.025 \\
    \bottomrule
    \end{tabular}}
    \label{tab:summary_forgetting_full}
\end{table}

Table~\ref{tab:supp_forgetting} lists the running-mean CLIP-I, CLIP-T, DINO, and DreamSim scores over all concepts learned so far (the quantity plotted in Figure~\ref{fig:forgetting}) at steps 1, 8, 16, 24, and 32 of the 32-concept stream, for SeqLoRA and \CA{}, with standard errors over the concepts evaluated at each step. Figures~\ref{fig:forgetting_evolution_main}, \ref{fig:forgetting_evolution_supp1} and \ref{fig:forgetting_evolution_supp2} show, for every concept, the generated image at steps 1, 4, 8, 12, 16, 20, 24, 28, and 32, using the same random seed and the basic prompt template (``photo of a $<$concept$>$'') at every step so that changes reflect the model rather than the sampling.
\begin{figure}[H]
    \centering
    \includegraphics[width=\textwidth]{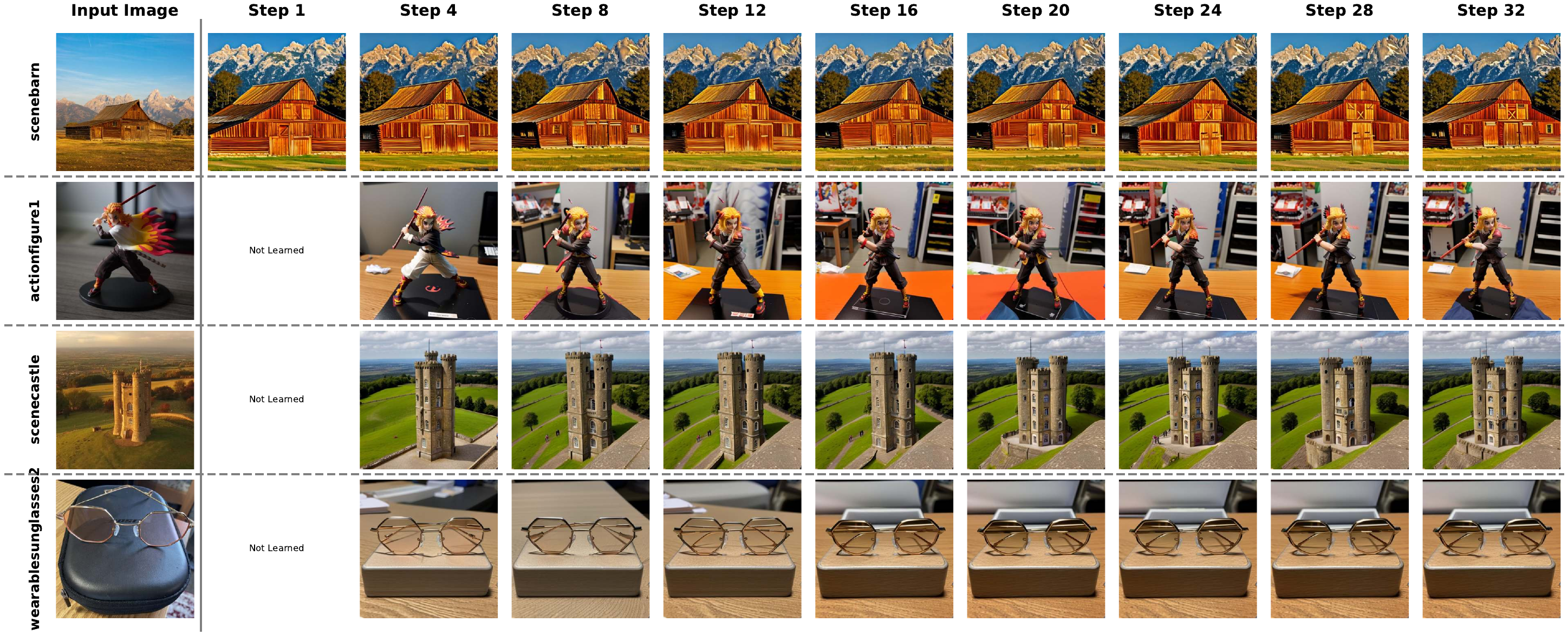}
    \caption{Evolution of generated images for selected concepts across training steps of the 32-concept SeqLoRA run (input image followed by samples at steps 1, 4, 8, 12, 16, 20, 24, 28, and 32).}
    \label{fig:forgetting_evolution_main}
\end{figure}
{\footnotesize
\begin{figure}[H]
\centering
\includegraphics[width=0.95\textwidth]{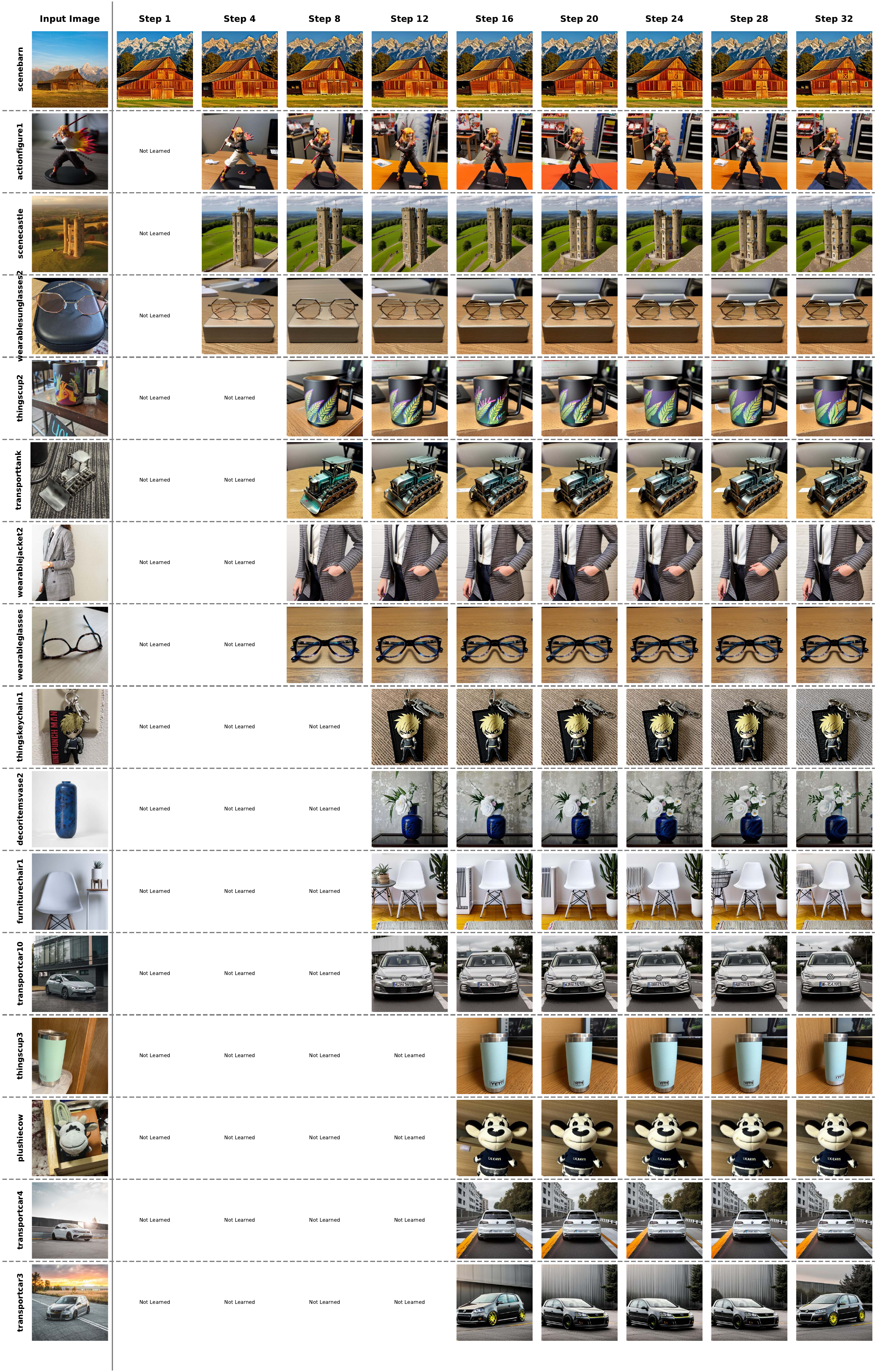}
\caption{Evolution of generated images for concepts 1--16 across training steps (same seed and prompt template at every step).}
\label{fig:forgetting_evolution_supp1}
\end{figure}
\begin{figure}[H]
\centering
\includegraphics[width=0.95\textwidth]{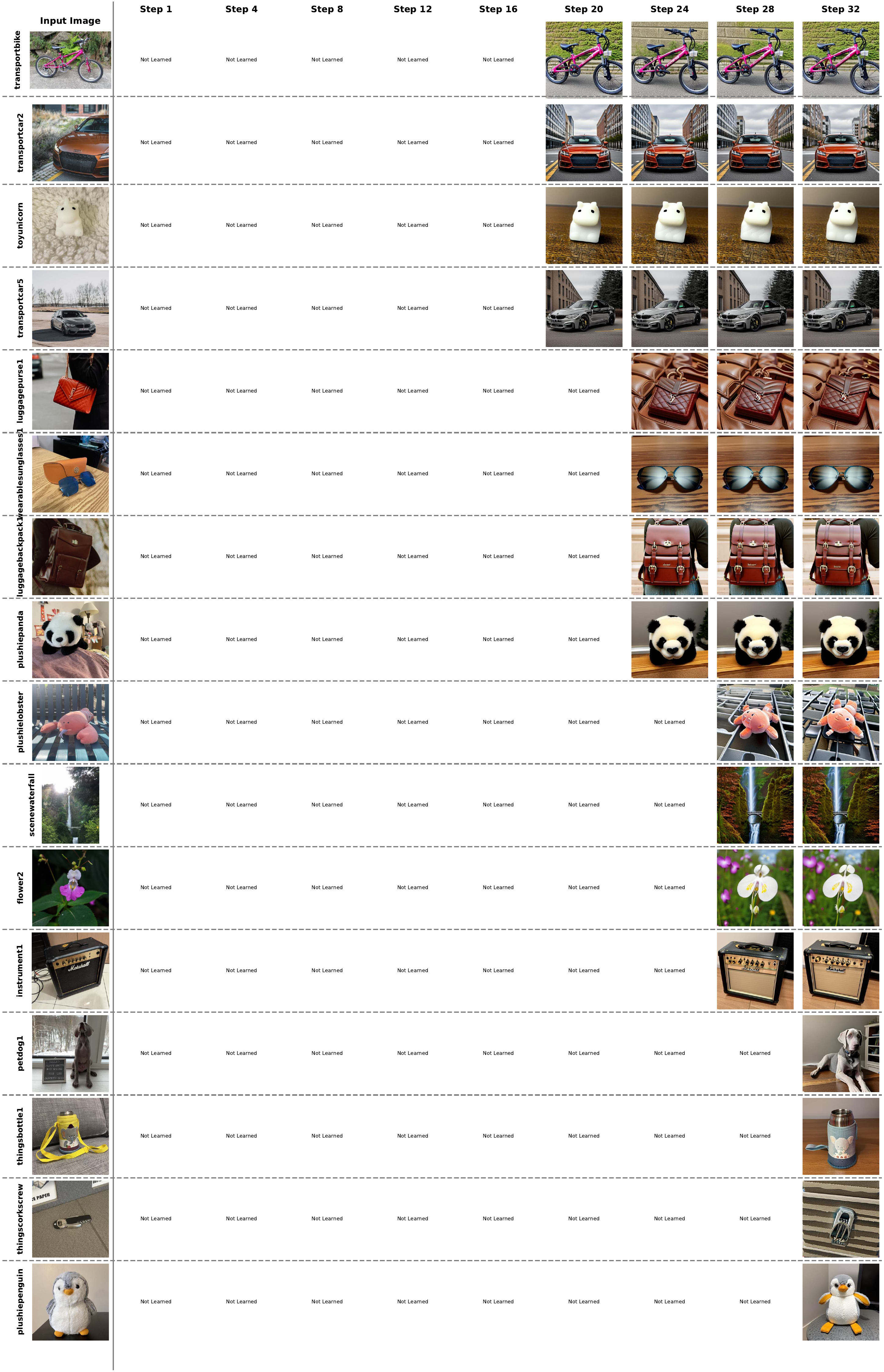}
\caption{Evolution of generated images for concepts 17--32 across training steps (same seed and prompt template at every step).}
\label{fig:forgetting_evolution_supp2}
\end{figure}
\setlength{\tabcolsep}{3pt}
\begin{table}[h]
\centering
\caption[Summary of forgetting results for 32 concepts]{Summary of forgetting results for 32 concepts across selected steps. We show the average performance across all concepts for SeqLoRA (SeqL) and Continual Alternating (Alt); mean $\pm$ standard error over the concepts learned so far.}\label{tab:supp_forgetting}
\resizebox{\textwidth}{!}{%
\begin{tabular}{@{}lcccccccccc@{}}
\toprule
\multirow{2}{*}{Metric} & \multicolumn{2}{c}{Step 1} & \multicolumn{2}{c}{Step 8} & \multicolumn{2}{c}{Step 16} & \multicolumn{2}{c}{Step 24} & \multicolumn{2}{c}{Step 32} \\
\cmidrule(lr){2-3} \cmidrule(lr){4-5} \cmidrule(lr){6-7} \cmidrule(lr){8-9} \cmidrule(lr){10-11} \\
 & SeqL & Alt & SeqL & Alt & SeqL & Alt & SeqL & Alt & SeqL & Alt \\
\midrule
\textbf{CLIP-I} & 0.739 & 0.754 & 0.694 $\pm$ 0.014 & 0.693 $\pm$ 0.012 & 0.692 $\pm$ 0.011 & 0.682 $\pm$ 0.012 & 0.691 $\pm$ 0.009 & 0.688 $\pm$ 0.009 & 0.687 $\pm$ 0.009 & 0.683 $\pm$ 0.008 \\
\textbf{CLIP-T} & 0.286 & 0.280 & 0.289 $\pm$ 0.005 & 0.289 $\pm$ 0.006 & 0.284 $\pm$ 0.004 & 0.285 $\pm$ 0.004 & 0.288 $\pm$ 0.004 & 0.288 $\pm$ 0.004 & 0.290 $\pm$ 0.003 & 0.290 $\pm$ 0.004 \\
\textbf{DINO} & 0.627 & 0.656 & 0.462 $\pm$ 0.043 & 0.469 $\pm$ 0.041 & 0.473 $\pm$ 0.026 & 0.461 $\pm$ 0.028 & 0.465 $\pm$ 0.025 & 0.463 $\pm$ 0.025 & 0.452 $\pm$ 0.022 & 0.450 $\pm$ 0.021 \\
\textbf{DreamSim} & 0.600 & 0.624 & 0.475 $\pm$ 0.035 & 0.473 $\pm$ 0.034 & 0.475 $\pm$ 0.020 & 0.461 $\pm$ 0.021 & 0.470 $\pm$ 0.019 & 0.465 $\pm$ 0.018 & 0.469 $\pm$ 0.016 & 0.466 $\pm$ 0.015 \\
\bottomrule
\end{tabular}%
}
\end{table}

}

\section{Full Scalability and Stability Results}\label{sec:supp_overall}
Figure~\ref{fig:overall_comparison_full} shows the scalability curves of Figure~\ref{fig:overall_comparison} for all eight metrics, and Figure~\ref{fig:forgetting_full} the corresponding running-mean stability curves along the 32-concept stream. Table~\ref{tab:supp_overall} gives the numbers for all metrics and all stream lengths from 8 to 101 concepts; entries marked ``--'' correspond to out-of-memory failures of the fusion step.
\begin{figure}[h]
    \centering
    \includegraphics[width=\textwidth]{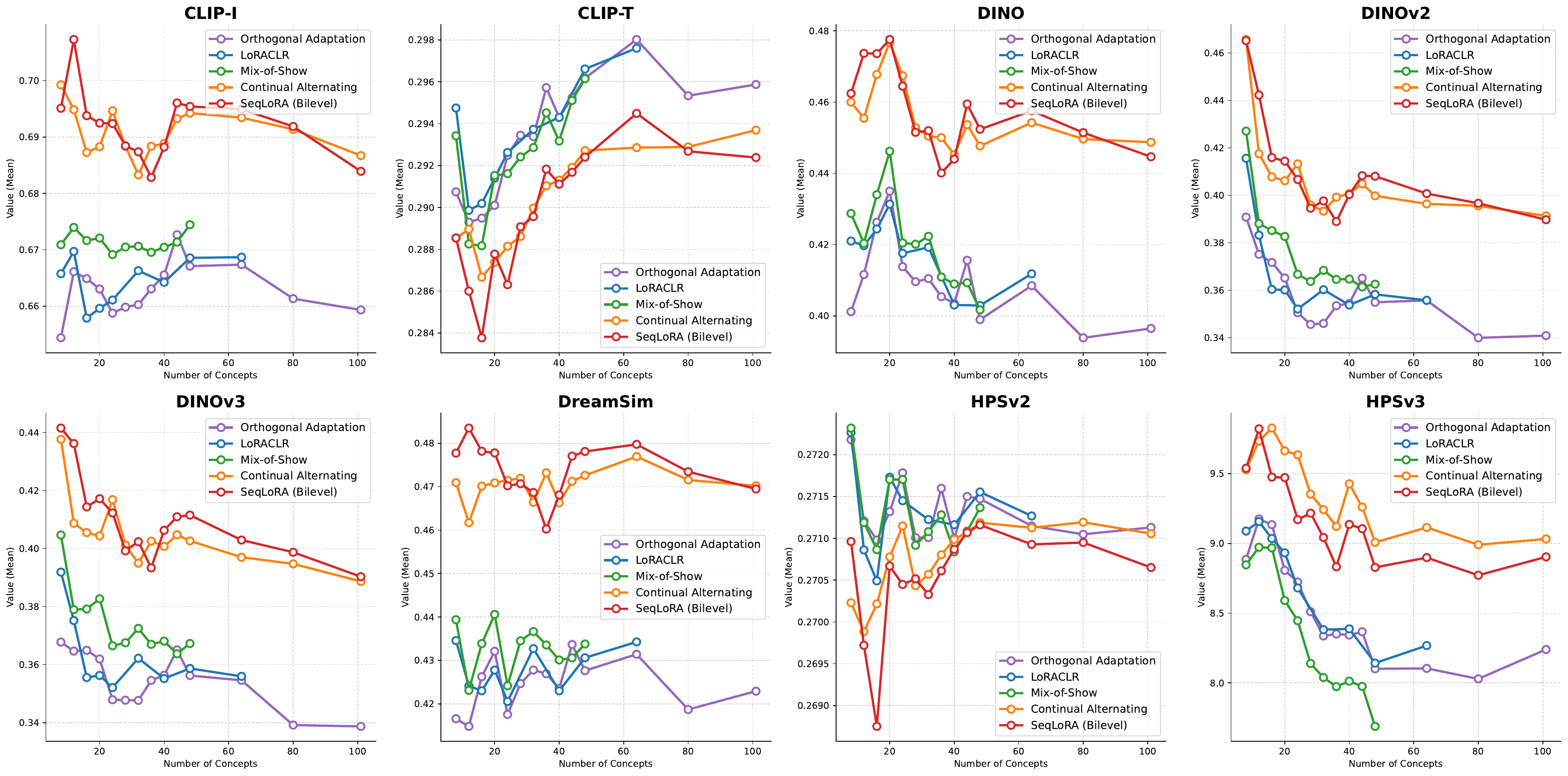}
    \caption{Final-model performance versus number of concepts for all eight metrics (averaged over all $T$ concepts).}
    \label{fig:overall_comparison_full}
\end{figure}
\begin{figure}[h]
    \centering
    \includegraphics[width=\textwidth]{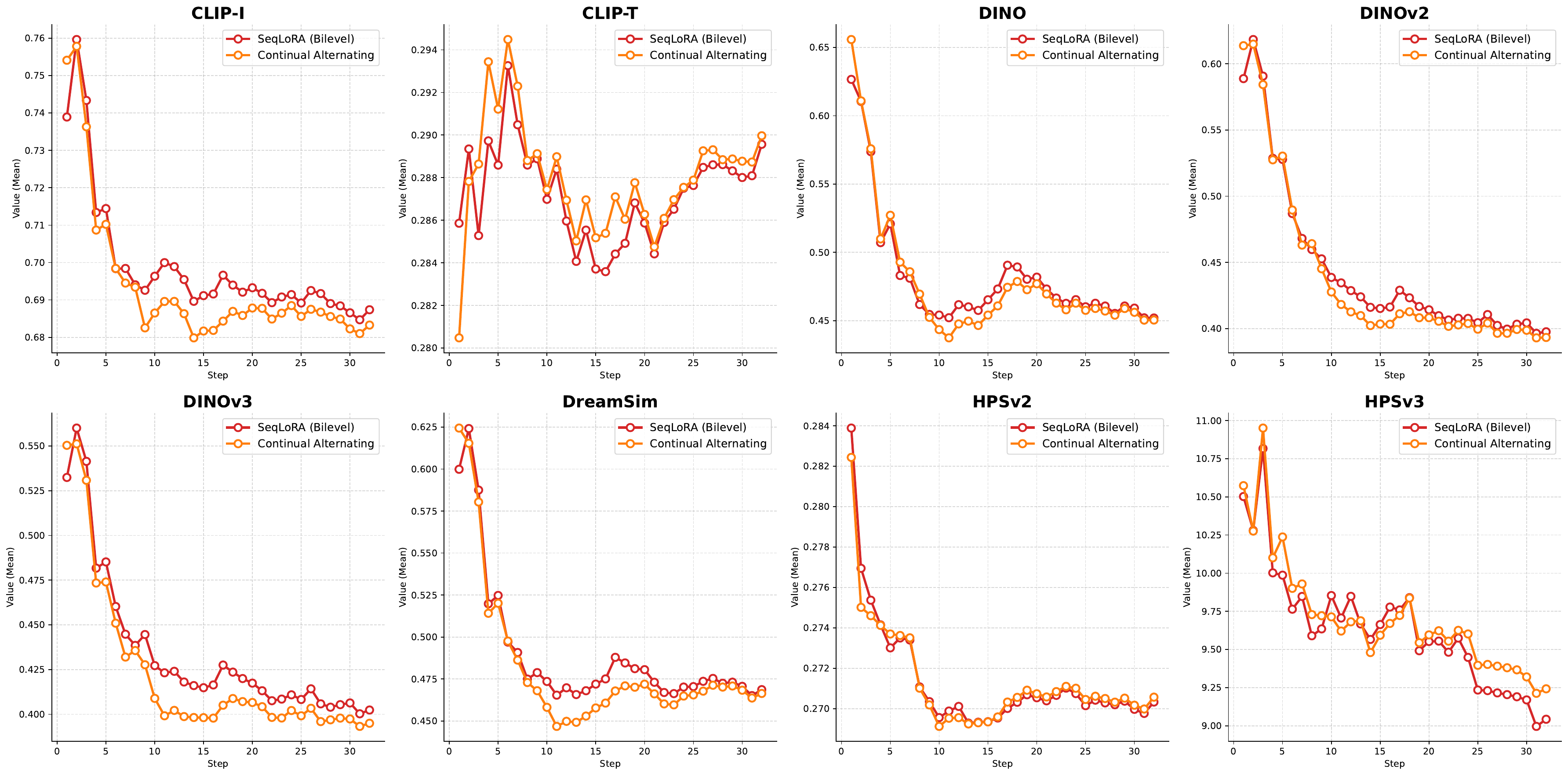}
    \caption{Running mean over all concepts learned so far after each step of the 32-concept stream, for all eight metrics.}
    \label{fig:forgetting_full}
\end{figure}
{\footnotesize
\setlength{\tabcolsep}{3pt}

\begin{table}[htbp]
\caption{Single-concept retention for all metrics and all stream lengths (8 to 101 concepts); mean $\pm$ standard error over concepts. ``--'' marks out-of-memory failures of the fusion step.}
\label{tab:supp_overall}
\centering
\resizebox{\textwidth}{!}{%
\begin{tabular}{@{}llccccccccc@{}}
\toprule
Metric & Method & 8 & 16 & 24 & 32 & 40 & 48 & 64 & 80 & 101 \\
\midrule
\multirow{5}{*}{\rotatebox{90}{CLIP-I}} & SeqLoRA (ours) & 0.710 $\pm$ 0.015 & 0.704 $\pm$ 0.012 & 0.695 $\pm$ 0.010 & 0.693 $\pm$ 0.009 & 0.707 $\pm$ 0.007 & 0.714 $\pm$ 0.007 & 0.715 $\pm$ 0.006 & 0.711 $\pm$ 0.006 & 0.705 $\pm$ 0.005 \\
  & Continual Alternating & 0.717 $\pm$ 0.015 & 0.696 $\pm$ 0.012 & 0.698 $\pm$ 0.009 & 0.688 $\pm$ 0.008 & 0.705 $\pm$ 0.007 & 0.711 $\pm$ 0.007 & 0.711 $\pm$ 0.006 & 0.708 $\pm$ 0.006 & 0.706 $\pm$ 0.005 \\
  & Mix-of-Show & 0.677 $\pm$ 0.018 & 0.677 $\pm$ 0.011 & 0.677 $\pm$ 0.009 & 0.677 $\pm$ 0.010 & 0.678 $\pm$ 0.008 & 0.681 $\pm$ 0.007 & - & - & - \\
  & Orthogonal Adaptation & 0.664 $\pm$ 0.016 & 0.673 $\pm$ 0.011 & 0.668 $\pm$ 0.009 & 0.671 $\pm$ 0.009 & 0.675 $\pm$ 0.008 & 0.677 $\pm$ 0.008 & 0.677 $\pm$ 0.007 & 0.670 $\pm$ 0.007 & 0.669 $\pm$ 0.006 \\
  & LoRACLR & 0.674 $\pm$ 0.017 & 0.664 $\pm$ 0.012 & 0.669 $\pm$ 0.009 & 0.674 $\pm$ 0.010 & 0.674 $\pm$ 0.008 & 0.678 $\pm$ 0.008 & 0.677 $\pm$ 0.007 & - & - \\
\midrule
\multirow{5}{*}{\rotatebox{90}{CLIP-T}} & SeqLoRA (ours) & 0.287 $\pm$ 0.005 & 0.282 $\pm$ 0.004 & 0.284 $\pm$ 0.004 & 0.286 $\pm$ 0.003 & 0.278 $\pm$ 0.003 & 0.279 $\pm$ 0.003 & 0.280 $\pm$ 0.002 & 0.279 $\pm$ 0.002 & 0.277 $\pm$ 0.002 \\
  & Continual Alternating & 0.287 $\pm$ 0.006 & 0.288 $\pm$ 0.004 & 0.286 $\pm$ 0.004 & 0.287 $\pm$ 0.004 & 0.277 $\pm$ 0.003 & 0.279 $\pm$ 0.003 & 0.278 $\pm$ 0.002 & 0.279 $\pm$ 0.002 & 0.278 $\pm$ 0.002 \\
  & Mix-of-Show & 0.278 $\pm$ 0.004 & 0.273 $\pm$ 0.003 & 0.274 $\pm$ 0.003 & 0.276 $\pm$ 0.003 & 0.277 $\pm$ 0.003 & 0.279 $\pm$ 0.002 & - & - & - \\
  & Orthogonal Adaptation & 0.274 $\pm$ 0.004 & 0.274 $\pm$ 0.003 & 0.275 $\pm$ 0.003 & 0.277 $\pm$ 0.003 & 0.278 $\pm$ 0.003 & 0.279 $\pm$ 0.002 & 0.280 $\pm$ 0.002 & 0.277 $\pm$ 0.002 & 0.278 $\pm$ 0.002 \\
  & LoRACLR & 0.278 $\pm$ 0.004 & 0.273 $\pm$ 0.003 & 0.275 $\pm$ 0.003 & 0.277 $\pm$ 0.003 & 0.278 $\pm$ 0.003 & 0.280 $\pm$ 0.002 & 0.280 $\pm$ 0.002 & - & - \\
\midrule
\multirow{5}{*}{\rotatebox{90}{DINO}} & SeqLoRA (ours) & 0.502 $\pm$ 0.041 & 0.484 $\pm$ 0.029 & 0.474 $\pm$ 0.026 & 0.468 $\pm$ 0.021 & 0.478 $\pm$ 0.017 & 0.488 $\pm$ 0.017 & 0.494 $\pm$ 0.014 & 0.487 $\pm$ 0.013 & 0.484 $\pm$ 0.011 \\
  & Continual Alternating & 0.504 $\pm$ 0.040 & 0.472 $\pm$ 0.026 & 0.477 $\pm$ 0.025 & 0.463 $\pm$ 0.020 & 0.475 $\pm$ 0.017 & 0.482 $\pm$ 0.016 & 0.486 $\pm$ 0.015 & 0.482 $\pm$ 0.013 & 0.485 $\pm$ 0.011 \\
  & Mix-of-Show & 0.440 $\pm$ 0.053 & 0.446 $\pm$ 0.028 & 0.435 $\pm$ 0.026 & 0.436 $\pm$ 0.023 & 0.424 $\pm$ 0.018 & 0.416 $\pm$ 0.017 & - & - & - \\
  & Orthogonal Adaptation & 0.416 $\pm$ 0.048 & 0.440 $\pm$ 0.029 & 0.429 $\pm$ 0.026 & 0.428 $\pm$ 0.023 & 0.420 $\pm$ 0.018 & 0.417 $\pm$ 0.016 & 0.427 $\pm$ 0.014 & 0.410 $\pm$ 0.013 & 0.416 $\pm$ 0.011 \\
  & LoRACLR & 0.435 $\pm$ 0.052 & 0.434 $\pm$ 0.032 & 0.431 $\pm$ 0.028 & 0.434 $\pm$ 0.023 & 0.419 $\pm$ 0.018 & 0.420 $\pm$ 0.017 & 0.429 $\pm$ 0.015 & - & - \\
\midrule
\multirow{5}{*}{\rotatebox{90}{DINOv2}} & SeqLoRA (ours) & 0.509 $\pm$ 0.041 & 0.452 $\pm$ 0.028 & 0.427 $\pm$ 0.021 & 0.418 $\pm$ 0.019 & 0.437 $\pm$ 0.015 & 0.445 $\pm$ 0.014 & 0.439 $\pm$ 0.013 & 0.435 $\pm$ 0.012 & 0.431 $\pm$ 0.012 \\
  & Continual Alternating & 0.515 $\pm$ 0.045 & 0.439 $\pm$ 0.022 & 0.433 $\pm$ 0.020 & 0.412 $\pm$ 0.017 & 0.432 $\pm$ 0.014 & 0.434 $\pm$ 0.014 & 0.429 $\pm$ 0.013 & 0.429 $\pm$ 0.012 & 0.429 $\pm$ 0.011 \\
  & Mix-of-Show & 0.439 $\pm$ 0.049 & 0.394 $\pm$ 0.026 & 0.375 $\pm$ 0.019 & 0.377 $\pm$ 0.020 & 0.375 $\pm$ 0.016 & 0.372 $\pm$ 0.015 & - & - & - \\
  & Orthogonal Adaptation & 0.405 $\pm$ 0.040 & 0.384 $\pm$ 0.025 & 0.361 $\pm$ 0.020 & 0.361 $\pm$ 0.020 & 0.368 $\pm$ 0.016 & 0.369 $\pm$ 0.015 & 0.371 $\pm$ 0.013 & 0.354 $\pm$ 0.013 & 0.358 $\pm$ 0.012 \\
  & LoRACLR & 0.431 $\pm$ 0.046 & 0.367 $\pm$ 0.028 & 0.361 $\pm$ 0.022 & 0.372 $\pm$ 0.020 & 0.367 $\pm$ 0.016 & 0.372 $\pm$ 0.015 & 0.370 $\pm$ 0.014 & - & - \\
\midrule
\multirow{5}{*}{\rotatebox{90}{DINOv3}} & SeqLoRA (ours) & 0.475 $\pm$ 0.031 & 0.439 $\pm$ 0.023 & 0.425 $\pm$ 0.018 & 0.418 $\pm$ 0.016 & 0.440 $\pm$ 0.013 & 0.446 $\pm$ 0.013 & 0.439 $\pm$ 0.011 & 0.434 $\pm$ 0.010 & 0.428 $\pm$ 0.011 \\
  & Continual Alternating & 0.475 $\pm$ 0.030 & 0.427 $\pm$ 0.016 & 0.429 $\pm$ 0.018 & 0.406 $\pm$ 0.015 & 0.431 $\pm$ 0.012 & 0.435 $\pm$ 0.013 & 0.428 $\pm$ 0.012 & 0.426 $\pm$ 0.010 & 0.424 $\pm$ 0.010 \\
  & Mix-of-Show & 0.411 $\pm$ 0.035 & 0.385 $\pm$ 0.020 & 0.371 $\pm$ 0.016 & 0.378 $\pm$ 0.017 & 0.374 $\pm$ 0.013 & 0.372 $\pm$ 0.013 & - & - & - \\
  & Orthogonal Adaptation & 0.379 $\pm$ 0.027 & 0.375 $\pm$ 0.020 & 0.356 $\pm$ 0.017 & 0.359 $\pm$ 0.017 & 0.365 $\pm$ 0.013 & 0.367 $\pm$ 0.014 & 0.366 $\pm$ 0.011 & 0.350 $\pm$ 0.011 & 0.351 $\pm$ 0.010 \\
  & LoRACLR & 0.403 $\pm$ 0.032 & 0.361 $\pm$ 0.021 & 0.359 $\pm$ 0.019 & 0.371 $\pm$ 0.017 & 0.364 $\pm$ 0.013 & 0.369 $\pm$ 0.014 & 0.367 $\pm$ 0.012 & - & - \\
\midrule
\multirow{5}{*}{\rotatebox{90}{DreamSim}} & SeqLoRA (ours) & 0.513 $\pm$ 0.035 & 0.494 $\pm$ 0.024 & 0.478 $\pm$ 0.020 & 0.477 $\pm$ 0.016 & 0.501 $\pm$ 0.014 & 0.512 $\pm$ 0.012 & 0.515 $\pm$ 0.010 & 0.508 $\pm$ 0.010 & 0.507 $\pm$ 0.008 \\
  & Continual Alternating & 0.510 $\pm$ 0.033 & 0.481 $\pm$ 0.020 & 0.480 $\pm$ 0.020 & 0.470 $\pm$ 0.015 & 0.495 $\pm$ 0.013 & 0.504 $\pm$ 0.012 & 0.508 $\pm$ 0.011 & 0.503 $\pm$ 0.009 & 0.504 $\pm$ 0.008 \\
  & Mix-of-Show & 0.447 $\pm$ 0.041 & 0.443 $\pm$ 0.022 & 0.436 $\pm$ 0.020 & 0.448 $\pm$ 0.017 & 0.442 $\pm$ 0.014 & 0.445 $\pm$ 0.013 & - & - & - \\
  & Orthogonal Adaptation & 0.429 $\pm$ 0.037 & 0.438 $\pm$ 0.021 & 0.430 $\pm$ 0.019 & 0.442 $\pm$ 0.018 & 0.437 $\pm$ 0.014 & 0.442 $\pm$ 0.013 & 0.446 $\pm$ 0.010 & 0.432 $\pm$ 0.010 & 0.438 $\pm$ 0.008 \\
  & LoRACLR & 0.445 $\pm$ 0.041 & 0.431 $\pm$ 0.024 & 0.431 $\pm$ 0.021 & 0.445 $\pm$ 0.018 & 0.437 $\pm$ 0.014 & 0.444 $\pm$ 0.013 & 0.447 $\pm$ 0.011 & - & - \\
\midrule
\multirow{5}{*}{\rotatebox{90}{HPSv2}} & SeqLoRA (ours) & 0.271 $\pm$ 0.003 & 0.269 $\pm$ 0.002 & 0.270 $\pm$ 0.001 & 0.270 $\pm$ 0.001 & 0.271 $\pm$ 0.001 & 0.271 $\pm$ 0.001 & 0.271 $\pm$ 0.001 & 0.271 $\pm$ 0.001 & 0.271 $\pm$ 0.001 \\
  & Continual Alternating & 0.270 $\pm$ 0.003 & 0.270 $\pm$ 0.002 & 0.271 $\pm$ 0.001 & 0.271 $\pm$ 0.001 & 0.271 $\pm$ 0.001 & 0.271 $\pm$ 0.001 & 0.271 $\pm$ 0.001 & 0.271 $\pm$ 0.001 & 0.271 $\pm$ 0.001 \\
  & Mix-of-Show & 0.272 $\pm$ 0.002 & 0.271 $\pm$ 0.002 & 0.272 $\pm$ 0.001 & 0.271 $\pm$ 0.001 & 0.271 $\pm$ 0.001 & 0.271 $\pm$ 0.001 & - & - & - \\
  & Orthogonal Adaptation & 0.272 $\pm$ 0.003 & 0.271 $\pm$ 0.001 & 0.272 $\pm$ 0.001 & 0.271 $\pm$ 0.001 & 0.271 $\pm$ 0.001 & 0.271 $\pm$ 0.001 & 0.271 $\pm$ 0.001 & 0.271 $\pm$ 0.001 & 0.271 $\pm$ 0.001 \\
  & LoRACLR & 0.272 $\pm$ 0.002 & 0.270 $\pm$ 0.002 & 0.271 $\pm$ 0.001 & 0.271 $\pm$ 0.001 & 0.271 $\pm$ 0.001 & 0.272 $\pm$ 0.001 & 0.271 $\pm$ 0.001 & - & - \\
\midrule
\multirow{5}{*}{\rotatebox{90}{HPSv3}} & SeqLoRA (ours) & 9.539 $\pm$ 0.588 & 9.476 $\pm$ 0.424 & 9.171 $\pm$ 0.509 & 9.043 $\pm$ 0.399 & 9.136 $\pm$ 0.283 & 8.828 $\pm$ 0.326 & 8.898 $\pm$ 0.252 & 8.772 $\pm$ 0.231 & 8.904 $\pm$ 0.195 \\
  & Continual Alternating & 9.529 $\pm$ 0.552 & 9.827 $\pm$ 0.328 & 9.636 $\pm$ 0.407 & 9.242 $\pm$ 0.338 & 9.427 $\pm$ 0.303 & 9.008 $\pm$ 0.347 & 9.115 $\pm$ 0.261 & 8.991 $\pm$ 0.228 & 9.032 $\pm$ 0.200 \\
  & Mix-of-Show & 8.847 $\pm$ 0.735 & 8.968 $\pm$ 0.444 & 8.446 $\pm$ 0.503 & 8.039 $\pm$ 0.457 & 8.013 $\pm$ 0.337 & 7.688 $\pm$ 0.314 & - & - & - \\
  & Orthogonal Adaptation & 8.886 $\pm$ 0.788 & 9.135 $\pm$ 0.431 & 8.723 $\pm$ 0.413 & 8.336 $\pm$ 0.432 & 8.343 $\pm$ 0.312 & 8.101 $\pm$ 0.314 & 8.104 $\pm$ 0.290 & 8.029 $\pm$ 0.244 & 8.240 $\pm$ 0.205 \\
  & LoRACLR & 9.089 $\pm$ 0.722 & 9.035 $\pm$ 0.400 & 8.681 $\pm$ 0.432 & 8.382 $\pm$ 0.431 & 8.388 $\pm$ 0.304 & 8.142 $\pm$ 0.308 & 8.268 $\pm$ 0.270 & - & - \\
\bottomrule
\end{tabular}%
}
\end{table}

}

\end{document}